\documentclass{article}

\usepackage[final, nonatbib]{neurips_2024}
\usepackage[numbers]{natbib}
\usepackage[utf8]{inputenc} % allow utf-8 input
\usepackage[T1]{fontenc}    % use 8-bit T1 fonts
\usepackage[hidelinks]{hyperref}
\usepackage{url}            % simple URL typesetting
\usepackage{booktabs}       % professional-quality tables
\usepackage{amsfonts}       % blackboard math symbols
\usepackage{nicefrac}       % compact symbols for 1/2, etc.
\usepackage{microtype}      % microtypography
\usepackage{xcolor}         % colors
\usepackage{times}

\usepackage{amsmath,amsfonts,bm,amssymb}

\def\eqref#1{equation~\ref{#1}}
\def\1{\bm{1}}

\DeclareMathAlphabet{\mathsfit}{\encodingdefault}{\sfdefault}{m}{sl}
\SetMathAlphabet{\mathsfit}{bold}{\encodingdefault}{\sfdefault}{bx}{n}

\newcommand{\E}{\mathbb{E}}

\newcommand{\R}{\mathbb{R}}

\usepackage{hyperref}
\usepackage{amsmath}
\usepackage{amsthm}
\usepackage[capitalise]{cleveref}
\newtheorem{theorem}{Theorem}[section]
\newtheorem{proposition}[theorem]{Proposition}
\newtheorem{lemma}[theorem]{Lemma}

\newtheorem{definition}[theorem]{Definition}
\newtheorem{assumption}[theorem]{Assumption}
\newtheorem{remark}[theorem]{Remark}
\usepackage{bbold}
\usepackage{url}
\usepackage{mathtools}
\usepackage{booktabs} 
\usepackage{graphicx}
\usepackage{wrapfig}
\usepackage{algorithm}
\usepackage{algorithmic}
\usepackage{caption}
\usepackage{subcaption}
\usepackage{xspace}
\usepackage{nicefrac}

\newcommand{\XX}{\mathcal{X}}

\renewcommand{\SS}{\mathcal{S}}
\renewcommand{\AA}{\mathcal{A}}
\newcommand{\CC}{\mathcal{C}}
\newcommand{\TT}{\mathcal{T}}

\newcommand{\NN}{\mathcal{N}}
\newcommand{\MM}{\mathcal{M}}
\newcommand{\FF}{\mathcal{F}}
\newcommand{\GG}{\mathcal{G}}
\newcommand{\one}{\mathbb{1}}
\renewcommand{\E}{\mathbb{E}}
\newcommand{\dyn}{\text{dyn}}
\newcommand{\goal}{\text{goal}}

\newcommand{\algo}{CODA\xspace}
\newcommand{\cgo}{CGO\xspace}
\usepackage{soul}

\title{How to Solve Contextual Goal-Oriented Problems with Offline Datasets?}

\author{Ying Fan$^{1}$, Jingling Li$^{2}$, Adith Swaminathan$^{3}$, Aditya Modi$^{3}$, Ching-An Cheng$^{3}$\\
$^{1}$University of Wisconsin-Madison  \hspace{4pt}
  $^{2}$ByteDance Research
\hspace{4pt} $^{3} \vspace{.1em}  $Microsoft Research \hspace{4pt} \\
}
\vspace{-4mm}

\begin{document}

\maketitle

\begin{abstract}

We present a novel method, Contextual goal-Oriented Data Augmentation (CODA), which uses commonly available unlabeled trajectories and context-goal pairs to solve Contextual Goal-Oriented (CGO) problems. By carefully constructing an action-augmented MDP that is equivalent to the original MDP, CODA creates a fully labeled transition dataset under training contexts without additional approximation error. We conduct a novel theoretical analysis to demonstrate CODA's capability to solve CGO problems in the offline data setup. Empirical results also showcase the effectiveness of CODA, which outperforms other baseline methods across various context-goal relationships of CGO problem. This approach offers a promising direction to solving CGO problems using offline datasets.  

\vspace{-4mm}
% Many important sequential decision problems -- from robotics, games to logistics -- are multi-tasked and goal-oriented. In this work, we frame them as Contextual Goal Oriented (CGO) problems, a goal-reaching special case of the contextual Markov decision process. CGO is a framework for designing multi-task agents that can follow instructions (represented by contexts) to solve goal-oriented tasks. We show that CGO problem can be systematically tackled using datasets that are commonly obtainable: an unsupervised interaction dataset of transitions and a supervised dataset of context-goal pairs. Leveraging the goal-oriented structure of CGO, we propose a simple data sharing technique that can provably solve CGO problems offline under natural assumptions on the datasets' quality. While an offline CGO problem is a special case of offline reinforcement learning (RL) with unlabelled data, running a generic offline RL algorithm here can be overly conservative since the goal-oriented structure of CGO is ignored. In contrast, our approach carefully constructs an augmented Markov Decision Process (MDP) to avoid introducing unnecessary pessimistic bias. In the experiments, we demonstrate our algorithm can learn near-optimal context-conditioned policies in simulated CGO problems, outperforming offline RL baselines.

\end{abstract}

\section{Introduction}
\label{sec:introduction}

Goal-oriented problems~\citep{kaelbling1993learning} are an important class of sequential decision-making problems with widespread applications, ranging from robotics~\citep{yu2023using}, game-playing~\citep{hessel2019multi}, to logistics~\citep{mirowski2018learning}. 
In particular, many real-world goal oriented problems are \emph{contextual}, where the objective of the agent is to reach a goal set communicated by a context. 
For example, consider instructing a truck operator with the context ``Deliver goods to a warehouse in the Bay area''. Given such a context and an initial state, it is acceptable to reach any feasible goal (a reachable warehouse location) in the goal set (warehouse locations including non-reachable ones). We call such problems \emph{Contextual Goal-Oriented} (\cgo) problems, which form an important special case of contextual Markov Decision Process (MDP)~\citep{hallak2015contextual}.

% For instance, instruction-following robots~\citep{yu2023using} learn a repertoire of skills that are framed as reaching different goals, %(e.g., instructions such as ``navigate to kitchen'') 
% and real-world logistics providers require goal-oriented navigation policies to operate their transportation fleets efficiently.  
% Game-playing agents in Montezuma's Revenge~\citep{ecoffet2021first} tackle the sparse reward challenge inherent in goal-oriented problems, where a positive binary reward is received only when the agent accomplishes a goal. 
% And real-world logistics providers require goal-oriented navigation policies to operate their transportation fleets efficiently. 

% \chingan{It is missing a connection from goal-oriented problems to contextual ones.}

% \chingan{Somehow this paragraph put an emphasis on goal ``set'' (as opposed to a point). I'm not sure why this is the emphasis. I thought the important aspect of CGO is the context part. Many standard goal oriented problems are about set reaching already.}
% \yingc{add visualizations of such problems}
% , it often suffices to reach a context-specific goal set, rather than a single goal state only. 
% In a \cgo problem, each task is a reaching problem with a goal set that is communicated indirectly to the agent via a context.  

\cgo is a practical setup that includes goal-conditioned reinforcement learning (GCRL) as a special case (the context in GCRL is just the target goal), but in general contexts in \cgo problem can be more abstract (like high-level task instructions in the above example) and the relationship between contexts and goals are not known beforehand. \cgo problems are challenging because 1) the rewards are sparse as in GCRL and 2) the contexts can be difficult to map into {feasible} goals. Nevertheless, \cgo problem has an important structure that the transition dynamics (e.g., navigating a city road network) are independent of the contexts that specify tasks. Therefore, efficient multitask learning can be achieved by sharing dynamics data across tasks.

In this paper, we study solving for \cgo problems in an offline setup. 
% \yingc{Our practical offline data setup... including...}
% The ideal datasets could be expert trajectories with diverse distributions of contexts and initial states, yet these datasets are often expensive to acquire. Instead, \emph{successful goal examples} alone given the context (without the trajectories) are much easier to obtain. 
% \chingan{"Our practical dataset setup" sounds weird. The following sentences are difficult to parse as well.}
We suppose access to two datasets --- an (unlabeled) \emph{dynamics} dataset of trajectories, and a (labeled) \emph{context-goal} dataset containing pairs of contexts and goal examples. 
Such datasets are commonly available in practice.
The typical contextual datasets for imitation learning (IL) (which has pairs of contexts and expert trajectories) is one example, since we can convert the contextual IL data into dynamics data and context-goal pairs. 
Generally, this setup also covers scenarios where expert trajectories are \emph{not} accessible (e.g., because of diverse contexts and initial states), since it does not assume goal examples to appear in the trajectories or the contexts are readily paired with transitions in expert trajectories.
Instead, it allows the dynamics datasets and the context-goal datasets to be independently collected.
% We identify two different kinds of datasets that are commonly available in \cgo applications -- an (unsupervised) \emph{dynamics} dataset of agent trajectories, and a (supervised) \emph{context-goal} dataset of pairs of contexts and goals.
For example, in robotics, task-agnostic play data can be obtained at scale~\citep{lynch2020learning,walke2023bridgedata} in an unsupervised manner whereas instruction datasets  (e.g.,~\citep{misra2016tell}) can provide context-goal pairs. 
In navigation, self-driving car trajectories (e.g.,~\citep{wilson2021argoverse,Sun_2020_CVPR}) also allow us to learn dynamics whereas landmarks datasets (e.g.~\citep{mirowski2018learning,NEURIPS2021_e02a35b1}) provide context-goal pairs. 

% Motivated by the above observations, we study the following research question:

% \emph{How can we utilize offline dynamics and context-goal  datasets above to solve contextual goal-oriented problems?}
% \ying{\textbf{Notice that naive context labeling can \emph{not} be a general solution to the problem, since the contexts distribution could be defined in a continuous space, and the states in the goal set might not necessarily overlap with states in the dynamics dataset}.}
% Introduce the two intuitive baselines and drawbacks
While offline CGO problems as described above are common in practical scenarios, to our knowledge, no algorithms have been specifically designed to solve such problems and CGO has not been formally studied yet. Some baseline methods could be easily conceptualized from the literature, but their drawbacks are equally apparent. One intuitive approach is to extend the goal prediction methods in GCRL~\citep{nair2018visual,nair2020contextual}:  given a test context, we can predict a goal and navigate to it using a goal-conditioned policy, where the goal prediction model can be learned from the context-goal dataset and the goal-conditioned policy can be learned from the trajectory dataset. However, the predicted goal might not always be feasible given the initial state since our context-goal dataset is not necessarily paired with transitions. Alternatively, the offline problem could be formulated as a special case of missing label problems~\citep{yu2021conservative} %That is,  %when we pair contexts with the unsupervised transitions, 
and we can learn a context-conditioned reward model to label the unsupervised transitions when paired with contexts as in~\cite{hu2023provable}. However, this approach ignores the goal-oriented nature of the problem and the fact that here only positive data (i.e. goal examples) are available for reward learning, which poses extra significant challenges. 
\cgo can be framed as an offline reinforcement learning (RL) problem with missing labels; However, existing algorithms~\citep{yu2022leverage,hu2023provable,li2023mahalo} in family assume access to both positive data (contexts-goal pairs) and negative data (contexts and non-goal examples), whereas only positive data are available here.
% \vspace{-3mm}
% \paragraph{Our method.}
% \yingc{Introduce the spirit of it; why it could work and better than the baselines. Then theory things.}
% \begin{wrapfigure}{r}{0.9\linewidth}
% \end{wrapfigure}

In this work, we present the first precise formalization of the \cgo setting, and propose a novel Contextual goal-Oriented Data Augmentation (CODA) technique that can provably solve \cgo problems subject to natural assumptions on the datasets' quality. The core idea is to \textbf{convert the context-goal dataset and the unsupervised dynamics dataset to a fully labeled transition dataset of an equivalent action-augmented MDP}, which circumvents the drawbacks in other baseline methods by fully making use of the CGO structure of the problem. We give a high-level illustration of this idea in Figure~\ref{fig:method}. In Figure~\ref{fig:method}, given a randomly sampled context-goal pair from the context-goal dataset, we create fictitious transitions from the corresponding goal example to a fictitious terminal state with a fictitious action and reward 1, and pair with the corresponding context. Also, we label all unsupervised transitions with reward 0 and non-terminal, and pair with the contexts randomly. Combining the two, we then have a fully labeled dataset (of an action-augmented contextual MDP, which this data augmentation and relabeling process effectively creates), making it possible to propagate supervision signals from the context-goal dataset to unsupervised transitions via the Bellman equation. We can then apply any offline RL algorithm based on Bellman updates like CQL~\citep{kumar2020conservative}, IQL~\citep{kostrikov2021offline}, PSPI~\citep{xie2021bellman}, ATAC~\citep{cheng2022adversarially} etc. %\footnote{Due to the nature of our method, we do not consider imitation learning or other supervised offline algorithms as the base algorithm.}. 
% We also show that our offline data setup is enough to learn CGO problems with certain assumptions, i.e., 
In comparison with the baseline methods discussed earlier, our method naturally circumvents their intrinsic challenges: 1) \algo directly learns context-conditioned policy and avoids the need to predict goals; 2) \algo effectively uses a fully labeled dataset, avoiding the need to learn a reward model and extra costs from inaccurate reward modeling.
\begin{figure}
    \centering
    \includegraphics[width=0.8\linewidth]{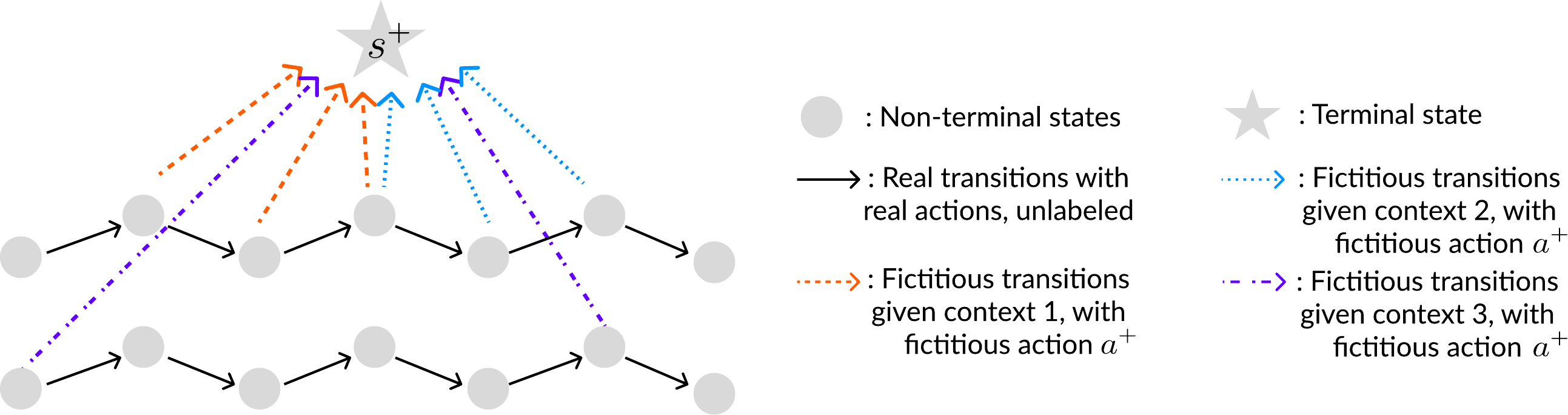}
    \caption{Illustration of \algo: We create fictitious transitions from goal examples to terminal states under the given context in the action-augmented MDP with reward 1, which enables the supervised signal to propagate back to unsupervised transitions via Bellman equation.\protect\footnotemark  
    % In practice, the generalization ability of the learned value function could handle the mismatch.
}
    \label{fig:method}
\vspace{-4mm}
\end{figure}

\section{Related Work}
\label{sec:related}
% \yingc{I added two most related paragraphs back to the main text.}
\paragraph{Offline RL.} %The ability to use any generic offline RL algorithm in our CGO setting enables the use of the shared dynamics dataset across tasks without any additional data collection \citep{levine2020offline}. 
Offline RL methods have proven to be effective in goal-oriented problems as it also allows learning a common set of sub-goals/skills \citep{chebotar2021actionable,ma2022offline,yang2023essential}. A variety of approaches are used to mitigate the distribution shift between the collected datasets and the trajectories likely to be generated by learned policies: 1) constrain target policies to be close to the dataset distribution \citep{fujimoto2019off,wu2019behavior,fujimoto2021minimalist}, 2) incorporate value pessimism for low-coverage or Out-Of-Distribution states and actions \citep{kumar2020conservative,yu2020mopo,jin2021pessimism} and 3) adversarial training via a two-player game \citep{xie2021bellman,cheng2022adversarially}.

% Our \algo allows the use of generic offline RL algorithms to solve CGO problem offline. We demonstrate its applicability with PSPI~\citep{xie2021bellman} and IQL~\citep{kostrikov2021offline} as our base offline RL algorithm in analyses (\cref{sec:analysis}) and experiments (\cref{sec:experiments}), respectively.

% \footnotetext{Contexts could be continuous. There is no need for fictitious transitions to start from the goal states that exactly appear in the unsupervised dataset.}
% \vspace{-3mm}
\textbf{Offline RL with unlabeled data.}
Our CGO setting is a special case of offline RL with unlabeled data, or more broadly the offline policy learning from observations paradigm~\citep{li2023mahalo}: There is only a subset of the offline data labeled with rewards (in our setting, that is the contexts dataset, as we don't know which samples in the dynamics dataset are goals.). However, the MAHALO scheme in~\citep{li2023mahalo} is much more general than necessary for CGO problems, and we show instead that our \algo scheme has better theoretical guarantees than MAHALO in Section~\ref{sec:analysis}. %instead of relying on the much more generic scheme MAHALO, we show that any generic offline RL algorithm can be used in CGO problems using a simple data augmentation approach. 
In our experiments, we compare CGO with several offline RL algorithms designed for unlabeled data: UDS \citep{yu2022leverage} where unlabeled data is assigned zero rewards and PDS \citep{hu2023provable} where a pessimistic reward function is learned from a labeled dataset.
% \vspace{-5mm}

\textbf{Goal-conditioned RL (GCRL).}
GCRL is a special case of our CGO setting, which has been extensively studied since ~\cite{kaelbling1993learning}. %Similar to our approach,
There are two critical aspects of GCRL: 1) data relabeling to make better use of available data and 2) learning reusable skills to solve long-horizon problems by chaining sub-goals or skills. On the one hand, hindsight relabeling methods \citep{andrychowicz2017hindsight,li2020generalized} are effective by reusing visited states in the trajectories as successful goal examples. For 2), hierarchical methods for determining sub-goals, and training goal reaching policies have been effective in long-horizon problems \citep{nair2019hierarchical,singh2020cog,chebotar2021actionable}. 
%Hindsight relabeling has emerged as a widely adopted method for reusing trajectories across various goals while also addressing the challenge of sparse reward in GO problems. Hierarchical methods have also been proposed where low-level controllers are trained to achieve sub-goals and a high-level policy proposes the sequence of sub-goals. Hindsight relabeling can be used for modeling richer low-level rewards as well, which aids in learning a common set of reusable skills.  
Another key objective of GCRL is goal generalization. Popular strategies include universal value function approximators \citep{schaul2015universal}, unsupervised representation learning \citep{nair2018visual,nair2019hierarchical,han2021learning}, and pessimism-induced generalization in offline GCRL formulations \citep{yang2023essential}. Our CGO framing enables both data reuse and goal generalization, by using contextual representations and a reduction to offline RL to combine dynamics and context-goal datasets.

\paragraph{Data-sharing in RL} Sharing information across multiple tasks is a promising approach to accelerate learning and to identify transferable features across tasks. In RL, both multi-task and transfer learning settings have been studied under varying assumption on the shared properties and structures of different tasks \citep{zhu2023transfer,teh2017distral,barreto2017successor,D'Eramo2020Sharing}. For data sharing in CGO, we adopt the contextual MDP formulation \citep{hallak2015contextual,sodhani2021multi}, which enables knowledge transfer via high-level contextual cues. Prior work on offline RL has also shown the utility of sharing data across tasks: hindsight relabeling and manual skill grouping \citep{kalashnikov2021mt}, inverse RL \citep{li2020generalized}, sharing Q-value estimates \citep{yu2021conservative,singh2020cog} and reward labeling \citep{yu2022leverage,hu2023provable}. 

\vspace{-1mm}
\section{Preliminaries}
In this section, we introduce the setup of CGO problems, infinite-horizon formulation for CGO, and the offline learning setup with basic assumptions for our offline dataset.
\label{sec:background}
\vspace{-1mm}

\paragraph{CGO Setup}\label{sec:bg_cgo}
A Contextual Goal-Oriented (CGO) problem describes a multi-task goal-oriented setting with a \emph{shared} transition kernel.
We consider a Markovian CGO problem, defined by the tuple $\MM = (\SS,\AA, P, R, \gamma, \CC, d_0)$, where $\SS$ is the state space, $\AA$ is the action space, $P:\SS\times\AA\to\Delta(\SS)$ is the transition kernel, $R:\SS\times\CC\to\{0,1\}$ is the sparse reward function,  $\gamma \in [0, 1)$ is the discount factor, $\CC$ is the context space,
% \footnote{We do not assume any particular topology on $\SS,\AA$ and $\CC$ and they can be continuous.}
 and $\Delta$ denotes the space of distributions.

Each context $c\in\CC$ specifies a goal-reaching task with a goal set $G_c \subset \SS$, and reaching any goal in the goal set $G_c$ is regarded as successful, inducing the
reward function  $R(s,c) = \one(s \in G_c) $. %(the indicator of whether a state $s$ is in the goal set $G_c$).
An episode of a CGO problem starts from an initial state $s_0$ and a context $c$ sampled from $d_0(s_0,c)$, and terminates when the agent reaches the goal set $G_c$. 
$c$ does not change during the transition; only $s_t$ changes according to $P(s'|s,a)$ and the transition kernel is context-independent. 
%In other words, CGO problem captures multiple goal-reaching problems specified by different contexts (we can think of contexts as instructions given to an agent) that share the same dynamics $P(s'|s,a)$.

%
% The classical GO problem~\citep{kaelbling1993learning} is a special case of CGO, where %single-goal problem has one context, while 
% a multi-goal problem can be viewed as multiple contexts with each context describing a goal.

\paragraph{Infinite-horizon Formulation for CGO setup}
\label{sec:formulation}

A fictitious zero-reward absorbing state $s^+ $$\notin \mathcal{S}$ can translate termination after reaching the goal to an infinite horizon formulation: \emph{whenever the agent enters $G_c$ it transits to $s^+$ in the next step (for all actions) and stays at $s^+$ forever}. This is a standard technique to convert a goal-reaching problem (with a random problem horizon) to an infinite horizon problem. This translation does \emph{not} change the problem, but allows cleaner analyses. We adopt this formulation in the following. 

We give details of this infinite-horizon conversion in the following.
First, we extend the reward and the dynamics: 
Let $\bar{\SS} = \SS \bigcup \{ s^+ \}$, $\XX \coloneq \SS \times \mathcal{C}$, and $\bar\XX \coloneq \bar{\SS} \times \mathcal{C}$. Define $\XX^+ \coloneq \{x : x=(s,c), s=s^+, c\in\CC\}$.
With abuse of notation, we define the reward and transition on $\bar\XX$ as
%\begin{align}
    $R(x) = \one(s \in G_c)$ where $x = (s,c)$. The transition kernel %\quad \text{and} \quad 
    $P(x'|x,a) \coloneq P(s'|s,c,a) \one(c'=c)$,
%\end{align}
where 
\begin{align*}
P(s'|s,c,a) =
    \begin{cases}
        \mathbb{1}(s'=s^+) &\text{\text{if} $s\in G_c$ \text{or} $s = s^+$,}\\
        P(s'|s,a) &\text{otherwise.}
    \end{cases}
\end{align*}
% \yingc{this definition of transition could be troublesome: when reaching $s\in G_c$, by definition, the env terminates. There will be no further action required and how would it transit to $s^+$? Do we want it to transit to $s^+$ without taking any action? But it would be conflicting since here we need the action input.} 
% For all value functions, we define their value at $s^+$ as zero. 
Given a policy $\pi: \XX \to\Delta(\AA)$,
% \footnote{Taking any action after reaching $s^+$ would not take any effect, so we do not model policy on $\XX^+$.}
the state-action value function (i.e., Q function) is 
$Q^\pi(x,a) \coloneqq \E_{\pi,P} \left[ \sum_{t=0}^\infty \gamma^t R(x) | x_0 = x, a_0 = a  \right].$    
%
% The Q function of policy $\pi$ satisfies the Bellman equation $Q^\pi  = \TT^\pi Q^\pi$, where we define the Bellman backup operator as
% \begin{align}
%     \TT^\pi f(x,a) \coloneq R(x) + \gamma \E_{x'\sim P(\cdot|x,a)} [ Q^\pi(x',\pi)] \label{eq:orig_bellman}
% \end{align}
% for a function $f:\XX\times\AA\to\R$.
%
$V^\pi(x) \coloneqq Q^\pi(x,\pi)$ is the value function given $\pi$, where  $Q(x,\pi) \coloneqq \E_{a\sim \pi}[Q(x,a)] \in [0,1]$.
The return $J(\pi) = V^\pi(d_0) = Q^\pi(d_0,\pi)$.
$\pi^*$ is the optimal policy that maximized  $J(\pi)$ and $Q^* \coloneqq Q^{\pi^*}$, $V^* \coloneqq V^{\pi^*}$. Let $G$ represent the goal set on $\XX$, that is, $G \coloneq \{ x \in \XX : x =(s,c), s \in G_c \}$.
%
% For all functions over $\bar{\XX}$, we define their values on $\XX^+$ as zero.

% finally, we next show that a simple data-sharing technique is sufficient to return good context-conditioned policies even for the most complex cases as in 

% \paragraph{Objective}
% Since the context carries rich information, % about the target task, 
% a CGO policy in general is context-conditioned, i.e.,  $\pi:\SS\times\CC\to\Delta(\AA)$. %In the infinite horizon discounted setting here, we measure 
% The performance of a policy $\pi$ is measured by its return, $J(\pi) \coloneqq \mathbb{E}_{\pi,P,d_0}\left[\sum_{=0}^T \gamma^t {R}(s_t, c)\right]$, where $T$ is the time the agent first enters $G_c$ (a random variable dependent on $\pi$, ${P}$ and $d_0$), and $\mathbb{E}_{\pi,P, d_0}$ denotes the expectation over trajectories generated by running $\pi$ with $P$ starting from $s_0,c$ sampled from $d_0$.
% %
% We can view the return as the average success rate of reaching \emph{any} goal in the goal set $G_c$ when the problem horizon is exponentially distributed (according to the discount $\gamma$).
% %
% A CGO algorithm takes a policy class $\Pi = \{ \pi: \SS\times\CC \to\Delta(\AA)\}$ as input and returns a near-optimal policy $\pi^\dagger$ such that $J(\pi^\dagger) \approx \max_{\pi\in\Pi} J(\pi)$.

% %$\mathbb{E}_{s,c\sim d_0} \left[ V^{\pi^\dagger}(s,c) \right] \approx \argmax_{\pi \in \Pi} \mathbb{E}_{s,c\sim d_0} \left[ V^\pi(s,c) \right]$. 

\vspace{-2mm}
\paragraph{Offline Learning for CGO}
\label{sec:offline_setup}
\vspace{-1mm}

We aim to solve CGO problems using offline datasets without additional online environment interactions, namely, by offline RL.
We identify two types of data that are commonly available: %in CGO applications:
$D_\dyn \coloneq \{ (s,a,s') \} $ is an \emph{unsupervised} dynamics dataset of agent trajectories collected from $P(s'|s,a)$, %from the environment (according to $P(s'|s,a)$, 
and $D_\goal \coloneq \{ (c,s) : s\in G_c \}$ is a \emph{supervised} dataset of context-goal pairs, which can be easier to collect than expert trajectories. %with a diverse contextual and initial states distribution.
We suppose that there are two distributions %of transitions 
$\mu_{\dyn}(s,a,s') $ and %a distribution of context-goal pairs 
$\mu_{\goal} (s,c)$, where %$\mu_{\dyn}$ is consistent with the transition kernel of $\MM$ (i.e., 
$\mu_\dyn(s'|s,a) = P(s'|s,a)$ and $\mu_{\goal}$ has support within $G_c$, i.e., $\mu_{\goal}(s|c)>0 \Rightarrow  s \in G_c$. % has support within $G_c$.
We assume that $D_{\dyn}$ and $D_{\goal}$ are i.i.d. samples drawn from the distributions $\mu_\dyn$ and $\mu_\goal$, i.e., 
\begin{align*}
    D_{\dyn} = \{  (s_i,a_i,s_i') \sim \mu_{\dyn} \}, 
    D_{\goal} = \{  (s_j,c_j) \sim \mu_{\goal} \}. 
\end{align*}
Notice that we do not assume the goal states in $D_{\goal}$ to be in $D_{\dyn}$, thus we cannot always naively pair transitions in $D_{\dyn}$ with contexts in $D_{\goal}$ and assign them with reward $1$.
To our knowledge, no existing algorithm can provably learn near-optimal $\pi$ using only the positive $D_{\goal}$ data (i.e., without non-goal examples) when combined with $D_{\dyn}$ data.

\vspace{-1mm}
\section{Contextual Goal-Oriented Data Augmentation (\algo)}
\label{sec:approach}
\vspace{-0.5mm}
The key idea of \algo is the construction of an \emph{action}-augmented MDP with which the dynamics and context-goal datasets can be combined into a fully labeled offline RL dataset. %, and convert the policy to the original MDP which is equivalent to the augmented one. %without missing rewards.
In the following, we first describe this action-augmented MDP
(Section~\ref{sec:augmented_mdp}) and show that it preserves the optimal policies of the original MDP (Appendix~\ref{app:augmented_mdp}). Then we outline a practical algorithm to convert the two datasets of an offline CGO problem into a dataset for this augmented MDP (Section~\ref{sec:method_des}) such that any generic offline RL algorithm based on Bellman equation can be used as a solver.

\subsection{Action-Augmented MDP}
\label{sec:augmented_mdp}
% \vspace{-1mm}

% % % why this?
% One reason why offline RL cannot directly combine $D_\dyn$ and $D_\goal$ to solve a CGO problem is that each goal-reaching problem has its own context-specific termination criterion.
% % Notice that although the dynamics datasets $D_\dyn$ is consistent with the original MDP transition kernel (i.e. $P(s'|s,a)$), it is however not consistent with the transition kernel $P(x'|x,a)$ (which also includes the effect of context-specific termination) of the context-augmented MDP in Section~\ref{sec:notation}. This is easiest to see 
% If some $s \in G_c$ in the $D_{\text{goal}}$ dataset is also observed in the dynamics dataset. $D_{\text{dyn}}$ will imply from $(s,a,s')$ that action $a$ can transition to $s'$, however $D_{\text{goal}}$ implies that all actions at $s$ will transition to $s^+$. This conflict means that combining the two datasets naively leads to an inconsistent algorithm.

We propose an action-augmented MDP (shown in Figure~\ref{fig:method}), 
which augments the action space of the contextual MDP in Section~\ref{sec:formulation} with \emph{a fictitious action $a^+\notin \mathcal{A}$}.

Let $\bar{\AA} = \AA \bigcup \{ a^+ \}$. We define the reward of this action-augmented MDP to be \emph{action-dependent}: for $x=(s,c)\in \XX$, $\bar{R}(x,a) \coloneq \one( s\in G_c) \one( a= a^+),$
which means the reward is 1 only if $a^+$ is taken in the goal set, otherwise 0.

We also extend the transition upon taking action $a^+$: $\bar{P}(x'|x,a^+) \coloneq \mathbb{1}(s'=s^+)$, and maintain the transition with real actions:
$\bar{P}(x'|x,a) \coloneq P(s'|s,a) \one(c'=c),$
which means whenever taking $a^+$, the agent would always transit to $s^+$, and the transition remains the same as in the original MDP given real actions. Further, we implement $s^+$ as $\texttt{terminal}=\text{True}$.

We define this augmented MDP as $\overline{\MM} \coloneq (\bar\XX, \bar\AA, \bar{R}, \bar{P}, \gamma)$. 

\paragraph{Policy conversion.} For a policy $\pi: \XX \to \Delta(\AA)$ in the original MDP, define its extension on $\overline{\MM}$:
\begin{align} %\label{eq:policy extension}
    \bar{\pi}(a|x) = 
    \begin{cases}
        \pi(a|x), & \text{$x\notin G$,}\\
        a^+, & \text{otherwise}.
    \end{cases}
\end{align}

\vspace{-2mm}
\paragraph{Regret equivalence.} An observation that comes with the construction is that 
if a policy is optimal in the original MDP, we can easily use the extension above to create an optimal policy in the augmented one. If a policy is optimal in the augmented MDP, it must take $a^+$ only when $x\in G$ (otherwise the return is lower, due to entering $s^+$ too early), thus we can revert this optimal policy of the augmented MDP  to find an optimal policy in the original MDP without changing its behavior and performance. We stated this property below; details can be found as Lemma~\ref{lm:Q equivalence} in Appendix~\ref{app:augmented_mdp}.

\begin{theorem}[Informal] \label{th:regret}
    The regret of a policy extended to the augmented MDP is equal to the regret of the policy in the original MDP, and any policy defined in the augmented MDP can be converted into that in the original MDP without increasing the regret. Thus, solving the augmented MDP can yield correspondingly optimal policies for the original problem. 
\end{theorem}
% Formally, we prove that the regret of a policy extended to the augmented MDP is equal to the regret of the policy in the original MDP, and any policy defined in the augmented MDP can be converted into that in the original MDP without increasing the regret. Thus, solving the augmented MDP can yield correspondingly optimal policies for the original problem. 

\begin{remark}
The benefit of using the equivalent $\overline{\MM}$ is to avoid missing labels: given contexts in $D_{\goal}$, the rewards in $\overline{\MM}$ are known from our dataset setup in Section~\ref{sec:offline_setup}, whereas the rewards of the original MDP $\MM$ are missing.
\end{remark}

\vspace{-0.5em}
\subsection{Method} 
\label{sec:method_des}
\vspace{-0.25em}

\algo is designed based on the observation on regret relationship in \cref{th:regret}: 
As described in Figure~\ref{fig:method}, given a context-goal pair $(s,c)$ from the dataset $D_{\text{goal}}$, we  create  a fictitious transition from $s$ to $s^+$ with action $a^+$, reward $1$ under context $c$. We also label all unsupervised transitions in the dataset $D_{\text{dyn}}$ with the original action and reward $0$ under $c$. In this way, we can have a fully labeled transition dataset in the augmented MDP given any $c$ from the context-goal dataset and then run offline algorithms (based on the Bellman equation) on this dataset.
This \algo algorithm is formally stated in Algorithm~\ref{alg:sds}. It takes two datasets $D_{\text{dyn}}$ and $D_{\text{goal}}$ as input, and produces a labeled transition dataset $\bar{D}_\dyn \bigcup \bar{D}_\goal$ that is suitable for use by any offline RL algorithm 
based on Bellman equation
like CQL~\citep{kumar2020conservative}, IQL~\citep{kostrikov2021offline}, PSPI~\citep{xie2021bellman}, ATAC~\citep{cheng2022adversarially}, etc.

\vspace{-3mm}
\paragraph{Interpretation.}Why would our action augmentation make sense? We consider dynamic programming on the created dataset. Imagine we have a fictious transition from $s$ to $s^+$ with $a^+$ under context $c$. When we calculate $V^*(x)$ via Bellman equation where $x = (s,c)$, it will choose the action with the highest $Q^*$ value in the augmented action space. The fictitious action would be the optimal action since it induces the highest $Q^*$ value\footnote{For all $a \ne a^+$ $Q^*$, $Q^*(x,a) < Q^*(x,a^+)$ when $\gamma<1$. If $\gamma=1$, the agent might also learn to travel to other goal states starting from $x$ with some probability, which is also acceptable in CGO.}, meaning $s$ is already in $G_c$, and no further action is needed. Then the value of $V^*(x)$ would \emph{naturally propagate to some state $x_{\text{prev}}=(s_{\text{prev}},c)$ via Bellman equation if $x$ is reachable starting from $x_{\text{prev}}$ } as shown in Figure~\ref{fig:method}, so $x_{\text{prev}}$ would still have meaningful values even with the intermediate reward 0. For $x$ to be reachable starting from $x_{\text{prev}}$, we do not require the exact $s$ to appear in the trajectory dataset due to the generalization ability of the value function (details in Section~\ref{sec:analysis}). 
% 2) A real action induces the highest $Q^*$, meaning that the agent needs to transfer to other states to reach the optimal goal in $G_c$ given $s$. Such behavior would be learned by the policy, and $V^*(x)$ would also propagate to previous transitions. 
For non-goal states, such fictitious action never appears in the dataset, thus it would not be the optimal action in Bellman equation in pessimistic offline RL. For example, the fictitious action never appears as the candidate in argmax in algorithms like IQL, and would be punished as OOD actions in algorithms like CQL.
% \footnote{\yingc{is the explanation ok?}}
We will prove this insight formally below in \cref{sec:analysis}.

\vspace{-2mm}
% \subsection{Algorithm: \algo}
% \label{sec:sds}
% \vspace{-2mm}

\begin{algorithm}[ht]
  \caption{\algo for CGO}
\textbf{Input}: Dynamics  dataset $D_{\text{dyn}}$, context-goal dataset $D_{\text{goal}}$
\begin{algorithmic}
% \STATE Initialize policy $\pi$ with an extra dimension in the action space for $a^+$, and extra dimensions for context input $c$
% \WHILE{not converged}
\FOR{each sample $(s,c) \sim D_{\text{goal}}$}
% \STATE Obtain $(1-\alpha)m$ random samples of $(c,s)$ pairs from $\mathcal{D}'$
\STATE Create transition\footnotemark  $(x, a^+, 1, x^+)$, where $x=(s,c)$ and $x^+=(s^+,c)$, add it to $\bar{D}_\goal$
%with context label $c$, $r=1$ %and $\text{terminal}=\text{True}$ \ying{do we need explicit s+ here?}
% \STATE Train $\pi$ using $\text{BaseAlg}$ and the combined minibatch transition data with batch size $m$
% \ENDWHILE
\ENDFOR
\FOR{each $(s,a,s') \sim D_{\text{dyn}}$}
\FOR{each $(\cdot,c)\sim \mathcal{D}_{\text{goal}}$}
\STATE Create transition $(x, a^+, 0, x')$, where $x=(s,c)$ and $x'=(s',c)$, add it to $\bar{D}_\dyn$
%Label the transition with context $c$ from $\mathcal{D}'$, $r=0$ and $\text{terminal}=\text{False}$
\ENDFOR
\ENDFOR
% \STATE Filter out the extra dimension in the output action space of the trained policy $\pi$ to get $\pi'$
\end{algorithmic}
\label{alg:sds}
\textbf{Output}: $\bar{D}_\dyn$ and  $\bar{D}_\goal$
% Labeled contextual real transition dataset $\mathcal{D}_{\text{real}}$ and fake transition dataset $\mathcal{D}_{\text{fake}}$
\end{algorithm}

\begin{remark}
    We do not need to learn to perform $a^+$ for the policy in practice since it is only for fictitious transitions which is already inside the goal set in the original MDP. (From the proof of Lemma~\ref{lm:Q equivalence}, we know taking $a^+$ is always strictly worse than taking actions in the original action space $\AA$.)
Therefore, we simply use the original action space for policy modeling and only use the fictitious transitions in value learning.
We note that in practice \cref{alg:sds} can be implemented as a pre-processing step in the minibatch sampling of a deep offline RL algorithm (as opposed to computing the full $\bar{D}_\dyn$ and $\bar{D}_\goal$ once before learning).
\end{remark}

\vspace{-2mm}

% Below we analyze \algo theoretically by applying \algo to PSPI~\citep{xie2021bellman}; later in \cref{sec:experiments}, we apply \algo to IQL~\citep{kostrikov2021offline} in simulation experiments. % \ying{mentioning how IQL is related with theory?}

% \yingc{Add reminder for equivalence of the augmented MDP, and why this dataset is fully labeled for the augmented MDP. Intuitively explain that, and then to analyze the data assumption needed -- using pspi which is theoretically justified but not test in antmaze?}

% \ying{mention we use the minibatch version in practice which is more efficient when first constructing the dataset with all possible combinations}

% \ying{actually, we have two sets as the output of algo 1. and we sample equally from them to run the minibatch version of offline RL algorithm with the intuition below:}

% \ying{already added some text in Appendix~\ref{app:hyper}: Intuitively, we should sample from both real and fake datasets with the same probability to create an overall balanced distribution. Empirically, we also found that the balance distribution generates the best result.}

\vspace{-0.25em}
\section{CGO is Learnable with Positive Data Only}
\label{sec:analysis}
\vspace{-0.25em}
In Section~\ref{sec:approach}, we show that a fully labeled dataset can be created in the augmented MDP without inducing extra approximation errors. But we still have no access to negative data, i.e., context and non-goal pairs. A natural question arises: \emph{Can we learn to solve CGO problems with positive data only? What conditions are needed for CGO to be learnable with offline datasets? }
% \yingc{We create a fully labeled dataset. Still, we only have positive data. Not answered: is it enough? under what conditions are we guaranteed to learn the optimal policy? }
% theoretically sound offline RL algorithm + sds -> new theoretical results
% \vspace{-1mm}
% theoretically any problem blablabla. the core trick blablabla.

We show in theory that we do \emph{not} need negative data to solve CGO problems by conducting a formal analysis for our method, instantiated with PSPI \citep{xie2021bellman} as an example of the base algorithm. We present the detailed algorithm CODA+PSPI in Appendix~\ref{app:sds+pspi}. This algorithm uses function classes $\FF:\SS\times\AA\to\R$ and $\GG:\SS\to\R$ to model value functions and optimizes the policy given a policy class $\Pi$ based on absolute pessimism defined on initial states. 
% For experimental results shown later in this paper, we use implicit Q-learning (IQL) \citep{kostrikov2021offline} which also uses value/policy function classes for offline learning while incorporating pessimism in value function estimation.

We present our assumptions and the main theoretical result as follows. 
\begin{assumption}[Realizability] \label{as:realizability}
We assume for any $ \pi \in \Pi$, $Q^\pi\in\FF$ and $R \in\GG$, where $\FF, \GG$ are the function classes for action-value and reward respectively.
\end{assumption}
\begin{assumption}[Completeness] \label{as:completeness}
We assume: For any $f\in\FF$, $g\in\GG$ and $\pi \in \Pi$, $\max(g(x),f(x,\pi)) \in \FF$; And for any $f\in\FF$, $\pi \in \Pi$, $\TT^\pi f (x,a) \in \FF$, where $\TT^\pi$ is a zero-reward Bellman backup operator with respect to $P(s'|s,a)$:
$\TT^\pi f (x,a) \coloneq \gamma \E_{x'\sim  P(s'|s,a) \one(c'=c)} [ f(x',\pi) ]$.
\end{assumption}
These two assumptions mean that the function classes $\FF$ and $\GG$ are expressive enough, which are standard assumptions in offline RL based on Bellman equation~\citep{xie2021bellman}. For deriving our main result, we define the coverage assumption needed below.
\begin{definition}\label{def:concentrability}
\label{as:concentrability}
We define the generalized concentrability coefficients:
\begin{small}
\begin{align}
\label{as:coverage}
    \mathfrak{C}_{\dyn}(\pi) \coloneq \max_{f,f'\in\FF} \frac{\| f - \TT^\pi f' \|_{\rho^\pi_{\notin G}}^2}{ \| f - \TT^\pi f' \|_{\mu_\dyn}^2} \qquad \text{and} \qquad \mathfrak{C}_{\goal}(\pi) \coloneq \max_{g\in\GG} \frac{\| g-R \|_{\rho^\pi_{\in G}}^2}{ \| g-R \|_{\mu_\goal}^2}
\end{align}
\end{small}%
% \yingc{What is r? should be $\bar{R}$ with $a^+$?}
where $\|h \|_\mu^2 \coloneq \E_{x\sim\mu}[ h(x)^2]$, $\rho_{\notin G}^\pi(x,a) = \E_{\pi,P} \left[ \sum_{t=0}^{T-1} \gamma^t \one(x_t=x, a_t=a)  \right]$, $\rho_{\in G}^\pi(x)=\E_{\pi,P} \left[ \gamma^T \one(x_T=x)  \right]$, and $T$ is the first time the agent enters the goal set. %\ying{We assume that $\mathfrak{C}_{\dyn}(\pi),\mathfrak{C}_{\goal}(\pi)$ are finite given $\mu_\goal, \mu_\dyn, \FF, \GG, \pi$.}
\end{definition}

Concentrability coefficients is a generalization notion of density ratio: It describes how much the (unnormalized) distribution in the numerator is ``covered'' by that in the denominator in terms of the generalization ability of function approximators~\citep{xie2021bellman}.
If $\mathfrak{C}_{\dyn}(\pi),\mathfrak{C}_{\goal}(\pi)$ are finite given $\mu_\goal, \mu_\dyn, \FF, \GG$ and $\pi$, then we say $\pi$ is covered by the data distributions, and conceptually offline RL can learn a policy to be no worse than $\pi$. 

We now state our theoretical result, which is proven by a careful reformulation of the Bellman equation of the action-augmented MDP, and construct augmented value function and policy classes in the analysis using the CGO structures (see Appendix~\ref{app:theory}).

\begin{theorem}\label{th:main theorm pspi (informal)}
Let $\pi^\dagger$ denote the learned policy of \algo + PSPI with datasets $D_\dyn$ and $D_\goal$, using value function classes $\FF=\{\XX\times\AA\to[0,1]\}$ and $\GG=\{\XX\to[0,1]\}$. 
Under Assumption~\ref{as:realizability},~\ref{as:completeness} and~\ref{def:concentrability}, with probability $1-\delta$, it holds, for any $\pi\in\Pi$,
% \begin{align*}
%     J(\pi) - J(\pi^\dagger)
%     &\lesssim  \mathfrak{C}_\dyn(\pi) \left( \sqrt{\frac{\log (\square/\delta)} {|D_\dyn|}} +  \sqrt{\frac{\log(\square/\delta)} {|D_\goal|}}\right) +  \mathfrak{C}_\goal(\pi) \sqrt{\frac{\log \frac{ \NN_\infty(\GG, 1/|D_\goal|)}{\delta} }{|D_\goal|}}
% \end{align*}
\begin{align*}
    J(\pi) - J(\pi^\dagger)
    &\lesssim  \mathfrak{C}_\dyn(\pi) \left( \sqrt{\frac{\log (|\FF||\GG||\Pi|/\delta)} {|D_\dyn|}} +  \sqrt{\frac{\log(|\FF||\GG||\Pi|/\delta)} {|D_\goal|}}\right) +  \mathfrak{C}_\goal(\pi) \sqrt{\frac{\log (|\GG|/\delta) }{|D_\goal|}}
\end{align*}
where 
% $\square \equiv \NN_\infty\left(\FF, 1/|D_\goal||D_\dyn|\right) \NN_\infty\left(\GG, 1/|D_\goal||D_\dyn|\right) \NN_{\infty,1}\left(\Pi,1/|D_\goal||D_\dyn|\right)$, and
$\mathfrak{C}_\dyn(\pi)$ and $\mathfrak{C}_\goal(\pi) $ are concentrability coefficients\footnote{We state a more general result for non-finite function classes in Theorem~\ref{thm:main_theorem_app} in the appendix}.
\end{theorem}

% \begin{theorem}\label{th:main theorm pspi (informal)}
% Let $\pi^\dagger$ denote the learned policy of \algo + PSPI with datasets $D_\dyn$ and $D_\goal$, using value function classes \footnote{We use the finite function class only for simplicity and state a more general result for non-finite function classes in Appendix~\ref{app:theory}.} $\FF=\{\XX\times\AA\to[0,1]\}$ and $\GG=\{\XX\to[0,1]\}$. 
% %
% Under Assumption~\ref{as:realizability},~\ref{as:completeness} and~\ref{def:concentrability},% \st{realizability and completeness assumptions}
% with probability $1-\delta$, it holds, for any $\pi\in\Pi$,
% \begin{align}
%     J(\pi) - J(\pi^\dagger)
%     &\leq  \mathfrak{C}_\dyn(\pi) \sqrt{\epsilon_\dyn} +  \mathfrak{C}_\goal(\pi) \sqrt{\epsilon_\goal}
% \end{align}
% where $\epsilon_\dyn = O\left( {\frac{\log \left( |\FF||\GG||\Pi|/\delta\right) }{|D_\dyn|}}  \right)$ and $\epsilon_\goal = O\left(   \frac{\log \left(|\GG|/\delta\right) }{|D_\goal|} \right)$ are statistical errors, and $\mathfrak{C}_\dyn(\pi)$ and $\mathfrak{C}_\goal(\pi) $ are concentrability coefficients. 
% % \st{which decrease as the data coverage increases} 
% % \yingc{what does ``data coverage increases" mean? Statistical errors would decrease as we have more data, but the density ratio is not something that naturally increases when we have more data. We only need the ratios to be finite. Please confirm.}.
% \end{theorem}

\vspace{-2mm}
\paragraph{Interpretation.}We can interpret Theorem~\ref{th:main theorm pspi (informal)} as follows: The statistical errors in value function estimation would decrease as we have more data from $\mu_\goal$ and $\mu_\dyn$; For any comparator $\pi$ with finite coefficients $\mathfrak{C}_{\dyn}(\pi),\mathfrak{C}_{\goal}(\pi)$, the final regret upper bound would also decrease.
Taking $\pi = \pi^*$ as an example. For the coefficients $\mathfrak{C}_{\dyn}(\pi),\mathfrak{C}_{\goal}(\pi)$ to be finite, it indicates 1) the state-action distribution from the dynamics data ``covers'' the trajectories generated by $\pi^*$, which includes the case of stitching\footnote{This does not mean the dynamics data have to be generated by the optimal policy; they can be generated by highly suboptimal policies so long as they together provide sufficient coverage.}; 2) the support of $\mu_{\goal}$ ``covers'' the goals $\pi^*$ would reach. We note that these conditions are \emph{not} any stronger than general requirements to solve offline algorithms: The ``coverage'' above is measured based on the generalization ability of $f$ and $g$ respectively as in Definition~\ref{def:concentrability}; e.g., if $f(x_1)$ and $f(x_2)$ are similar for $x_1\neq x_2$, then $x_2$ is within the coverage of $\mu$ so long as $x_1$ can be generated by $\mu$ in terms of the generalization ability of $f$. Such a coverage condition is weaker than coverage conditions based on density ratios.
Besides, \cref{th:main theorm pspi (informal)} simultaneously apply to all $\pi\in\Pi$ not just $\pi^*$. Therefore, as long as the above ``coverage'' conditions hold for any policy $\pi$ that can reach the goal set, the agent can learn to reach the goal set. 
% The condition can be further relaxed if we consider all possible goal-reaching policies as the comparator policy $\pi$ in addition to $\pi^*$.
% \ying{By setting $\pi$ to be any goal-reaching policy in \cref{th:main theorm pspi (informal)},
% we see that the policy learned by SDS+PSPI has a small regret wrt a goal-reaching policy given that 1)the state distribution from dynamics ``covers" the goal set  }
% \st{the dynamics data $D_\dyn$ covers the trajectory of the optimal policy, and the context-goal dataset $D_\goal$ covers goals the optimal policy would reach}. 
% \yingc{Too strict; this is where the reviewers are confused about: ``goal dataset contains all the goals the agent will encounter and that the trajectories dataset contains trajectories leading to those goals", in their own words. In fact 3.4 can have two explanations: first is the previous optimal policy one but also points out that this "coverage" is in the sense of function approximation, not about containing the exact same thing. Another is that if we only care about reaching the goal set finally,  we can also extend the definition of optimal policy to be all goal-reaching policies and taking the minimum regret, and we can choose the policy that reaches the goal covered by dynamics, then we only need assumptions below: }  
Thus, we show that \algo with PSPI can provably solve CGO without the need for additional non-goal samples, i.e., CGO is learnable with positive data only.

\begin{remark} Here we only require function approximation assumptions made in the original MDP, without relying on functions defined on the fictitious action or completeness assumptions based on the fictitious transition. As a result, our theoretical results are comparable with those of other approaches. 
\end{remark}

\begin{remark}
MAHALO~\citep{li2023mahalo} is a SOTA offline RL algorithm that can provably learn from unlabeled data. One version of MAHALO is realized on top of PSPI in theory; however, their theoretical result (Theorem D.1) requires a stronger version concentrability, $ \max_{g\in\GG} \nicefrac{\| g-r \|_{\rho^\pi_{\notin G}}^2}{ \| g-r \|_{\mu_\goal}^2}$, to be small. In other words, it needs negative examples of (context, non-goal state) tuples for learning. 
\end{remark}

\vspace{-2mm}
\paragraph{Intuition for other base algorithms. } Notice that PSPI is just one instantiation. Conceptually, the coverage conditions above also make sense for other pessimistic offline RL instantiations based on the Bellman equation (like IQL), since the key ideas used in the above analyses are that the regret relationship (\cref{th:regret}) between the original MDP  and the action augmented MDP (which is algorithm agnostic) and that pessimism together with Bellman equations can effectively propagate information from the context-goal dataset 
(without the need for negative data). However, performing complete theoretical analyses of \algo for all different offline RL algorithms is out of the scope of this paper. 

\begin{figure*}[t]
\centering
\vspace{-1em}
\includegraphics[width=0.9\textwidth]{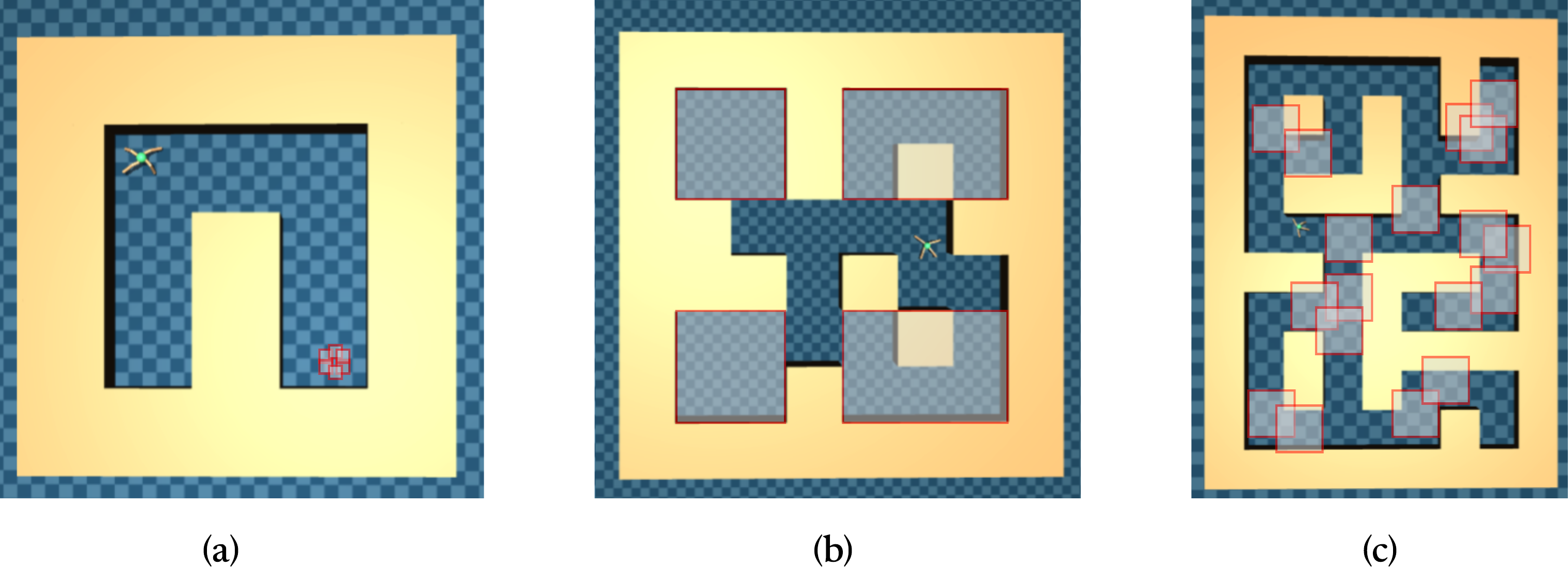}
% \begin{subfigure}[b]{0.285\textwidth}
%     \centering
%          \includegraphics[width=\textwidth]{Fig1a.jpg}
%          \caption{Similar goal sets with different contexts}
%          \label{fig:1a}
% \end{subfigure}
% \begin{subfigure}[b]{0.345\textwidth}
%     \centering
%          \includegraphics[width=\textwidth]{Fig1b.png}
%          \caption{Distinct goal sets with {finite} 
%          % \st{different but small number of} 
%          contexts}
%          \label{fig:1b}
% \end{subfigure}
% \begin{subfigure}[b]{0.345\textwidth}
%     \centering
%          \includegraphics[width=\textwidth]{Fig1c.jpg}
%          \caption{
%          % \st{Overlapping}
%          Goal sets {(could be overlapping) given continuous} 
%          % \st{across} 
%          contexts 
%          % \st{but with an empty intersection} 
%          % \yingc{the empty intersection thing does not make any sense here. also deleted every copy of it}
%          }
%          \label{fig:1c}
% \end{subfigure}
\caption{Illustration of the context-goal relationship with increasing complexity (Each red boundary defines a goal set with its center location as context). (a) Contexts and goal sets are very similar such that it could be approximately solved by a context-agnostic policy. 
(b) Contexts are finite, and different contexts map to distinct goal sets, which requires context-dependent policies. 
% \st{(e.g., general-purpose robotics)} 
(c) Contexts are continuous and infinite. The context-goal mapping is neither one-to-many nor many-to-one, creating a CGO problem with full complexity.}
\label{fig:cgo}
\vspace{-1.5em}
\end{figure*}

\vspace{-1em}
\section{Experiments}
\label{sec:experiments}
\vspace{-0.5em}
In this section, we present the experimental setup and results for CODA. Code is publicly available at: \url{https://github.com/yingfan-bot/coda}.

For a comprehensive empirical study, we first introduce the diverse spectrum of practical CGO setups.

\textbf{Diverse spectrum of practical CGO problems.} The main challenge of the CGO problem compared with traditional goal-conditioned RL is the potential complexity in the context-goal relationship. Therefore, to showcase the efficacy of different methods, we construct three levels with \emph{increasing difficulty} as shown in Figure~\ref{fig:cgo}: (a) has a similar complexity as a single-task problem where the context does not play a significant role; (b) requires a context-dependent policy but only has finite contexts; (c) has infinite continuous context, requiring a context-dependent policy and generalization ability to contexts outside the offline data set.
% In the following section, 
We aim to answer the following questions:
1) Does our method work under the data assumptions in Section~\ref{sec:formulation}, with different levels of context-goal complexity? 
2) Is there any empirical benefit from using \algo, compared with baseline methods including reward learning, goal prediction, etc?

\vspace{-0.75em}
\subsection{Environments and Datasets}
\label{exp:dataset}
\vspace{-0.25em}
% \vspace{-2mm}
\textbf{Dynamics dataset.} For all experiments, we use the original AntMaze-v2 datasets (3 different mazes and 6 offline datasets) of D4RL~\citep{fu2020d4rl} as dynamics datasets $D_\dyn$, removing all rewards and terminals. 
\textbf{Context-goal dataset.} We construct three levels of context and goal relationships as shown in Figure~\ref{fig:cgo}. For each setup, we first define the context set, and then sample a fixed set of states from the offline trajectory dataset that satisfies the context-goal relationship, and then randomly \emph{perturb} the states such that there would be no way to directly match goal examples to some states in the trajectories given contexts. Notice that this context-goal relationship is only used for dataset construction and is not accessible to the learning algorithm.\footnote{Also note that the state space in Antmaze not only includes the 2D location; it also includes data from robotic arms, etc. We define the context-goal relationship only on the 2D location and ignore other information.} The specific context-goal relationship are discussed in Section~\ref{sec:results} with the construction/evaluation details in Appendix~\ref{app:context_setup}.

\vspace{-0.25em}
\subsection{Method and Baselines}
\label{exp:baseline}
\vspace{-0.25em}
For controlled experiments, we use  IQL~\citep{kostrikov2021offline} 
% \footnote{We choose IQL since it is a popular and efficient algorithm that performs well in AntMaze, while PSPI is harder to train and does not perform as well in AntMaze.}
as the same backbone offline algorithm for all the methods with the same set of hyperparameters. Our choice of IQL is motivated by both its benchmarked performance on several RL domains and its structural similarity to PSPI (use of value/policy function classes along with pessimism).
Please see Appendix~\ref{app:hyper} for hyperparameters. 

We describe the algorithms compared in the experiments.

% \vspace{-2mm}
\textbf{\algo.} We apply \algo in Algorithm~\ref{alg:sds} with IQL as the offline RL algorithm to solve the augmented MDP defined in Section~\ref{sec:augmented_mdp}
% , where we can think of IQL as optimizing Eq.~(\ref{eq:empirical losses app}) via expectile regression given the offline dataset.
More specifically, we set $a^+$ to be an extra dimension in the action space of the action-value function, and model the policy with the original action space. Empirically, we found that equally balancing the samples $\bar{D}_\dyn$ and $\bar{D}_\goal$ generates the best result\footnote{We study the effect of this sampling ratio on \algo's performance in Table~\ref{tab:sample_ratio} in Appendix~\ref{app:hyper}}.
% We can think of IQL as optimizing Eq.~(\ref{eq:empirical losses}) via expectile regression given the offline dataset.. %See details in Appendix~\ref{app:hyper}. 
% \vspace{-2mm}
% \footnote{Note the reward shift following conventions in AntMaze environments: we use -1/0 reward instead of 0/1 reward; one can show these two rewards induce the same policy ordering. We apply the same reward shift in Algorithm~\ref{alg:sds} for our method and all other baselines.}. 
Then we apply IQL on this labeled dataset. 

% \vspace{-2mm}

\textbf{Reward prediction.} For this family of baselines, we need to use the learned reward to predict the label of context-goal samples in the randomly sampled context-transition pairs during training, so we need to pre-train a reward model using the context-goal dataset. We use PDS~\citep{hu2023provable} for reward modeling, and learn a \emph{pessimistic} reward function using ensembles of models on the context-goal dataset. Then we apply the reward model to label the transitions with contexts, run IQL on this labeled dataset, and get a context-dependent policy. Besides PDS, we also test naive reward prediction (RP, which follows the same setup of PDS but without ensembles) and UDS~\citep{yu2022leverage} +RP in Section~\ref{sec:results} (See details in Appendix~\ref{app:hyper}).  Additionally, we add results from training with the oracle reward (marked as ``Oracle Reward'') where we provide the oracle reward for any query context-goal pairs, as a reference of the performance upper bound for reward prediction methods.

% \vspace{-2mm}
\textbf{Goal prediction.}{ We consider another GCRL-based baseline. Notice that the relationship between contexts and goals is unknown in CGO, we cannot directly apply traditional GCRL methods to CGO problems. Therefore, we adopt a workaround to use GCRL methods: We learn a conditional generative model as the goal predictor using classifier-free diffusion guidance~\cite{ho2022classifier}, where the contexts serve as the condition, and the goal examples are used to train the generative model. We also learn a general goal-conditioned policy with the dynamics-only dataset using HER~\cite{andrychowicz2017hindsight}+IQL. Given a test context, the goal predictor samples the goal given the context, which is then passed as the condition to the policy.}

\vspace{-0.25em}
\subsection{Results}
\label{sec:results}
\vspace{-0.25em}
\textbf{Original AntMaze: Figure~\ref{fig:cgo}(a).} In the original AntMaze, 2D goal locations (contexts) are limited to a small area as in Figure~\ref{fig:cgo} (a). To make it a CGO problem, we make the test context visible to the agent.
This setting in Figure~\ref{fig:cgo} is approximately a single-task problem. 
% In order to get a baseline estimate of the performance, we also run IQL on the dynamics dataset with the correct goal/reward labels without any use of the context. We show this as a reference and it is \emph{not} used for comparison.

% \vspace{-3mm}
% \paragraph{\algo matches the performance of the context-agnostic method under the setup of Fig~\ref{fig:cgo}(a),.}

{\algo generally achieves better performance than reward learning and goal prediction methods.}
Comparing the normalized return in each AntMaze environment for all methods, our method consistently achieves equivalent or better performance in each environment compared to other baselines (Table~\ref{tab:original antmaze}). \footnote{We find umaze is too easy: even if the reward labeling is bad it still has a relatively high reward, so we also omit it in other experiments. 
% Our method also achieves comparable average performance to the context-agnostic method, showing that it learns to solve this special case of the CGO problem. 
% \citet{li2023survival} show offline RL algorithms can learn good with goal-reaching data even when the rewards are wrong. 
We also find UDS and RP are not very effective in our data setup, so we also omit them in other experiments.} Moreover, the performance of Goal Prediction is rather poor, which mainly comes from not enough goal examples to learn from in this setup due to a limited goal area.

\begin{table*}[h]
\centering
\vspace{-0.5em}
\caption{Average success rate (\%) in AntMaze-v2 from all environments.}
\scalebox{0.9}{
\begin{tabular}{l|ccccc|c}
\toprule
  Env/Method             & \algo (Ours)      & PDS &Goal Prediction       & RP & UDS+RP &  Oracle Reward\\
   \hline
umaze          & \textbf{94.8±1.3}  & 93.0±1.3 & 46.4±6.0& 50.5±2.1        & 54.3±6.3            &          94.4±0.61\\
umaze diverse  & \textbf{72.8±7.7} & 50.6±7.8  &42.8±4.4& \textbf{72.8±2.6}        & 71.5±4.3               &  76.8±5.44     \\
medium play    & \textbf{75.8±1.9}  & 66.8±4.9 &43.8±4.7& 0.5±0.3         & 0.3±0.3                       & 80.6±1.56     \\
medium diverse & \textbf{84.5±5.2}  & 22.8±2.4 & 28.6±3.9   & 0.5±0.5         & 0.8±0.5                & 72.4±4.26           \\
large play     & \textbf{60.0±7.6} & 39.6±4.9 &13.0±4.0 & 0±0             & 0±0                    &    41.2±3.58       \\
large diverse  & \textbf{36.8±6.9} & 30.0±5.3&12.6±2.7 & 0±0             & 0±0                    & 34.2±2.59    \\
\hline
average & \textbf{70.8}&50.5&31.2&20.7&21.2 & 66.6\\
\bottomrule
\end{tabular}
}

\label{tab:original antmaze}
\end{table*}

\textbf{Four Rooms: Figure~\ref{fig:cgo}(b).} We partition the maze into four rooms as in Figure~\ref{fig:cgo}(b), where the discrete room numbers (1,2,3,4) serve as contexts and we uniformly select test contexts. A context-dependent policy is needed, but there is no generalization required for unseen contexts in this setup.
% As the agent always starts in Room 1, the training and test context sets are Room 2,3,4.
% We use medium and large environments. 
% \ying{Additionally, we perturb all the state dimensions with $\mathcal{N}(0, 0.1)$ Gaussian noises when constructing the context-goal dataset, so the goal sets are not subsets of the dynamics datasets.}

%\footnote{Since ``umaze" environments are too easy, we only evaluate medium and large mazes with offline datasets marked with ``play". We do the same thing for Section~\ref{exp:cell}.}.
% \vspace{-2mm}
% \paragraph{\algo achieves better performance than baselines under the setup in Figure~\ref{fig:cgo}(b).} 
We show the normalized return (average success rate in percentage) in each modified Four Rooms environment for our method and baseline methods in Table~\ref{tab:four_rooms}, where our method consistently outperforms the performances of baseline methods. 
% We observe that the context agnostic method achieves rather high performance under this setting. This is because the number of rooms is only three, and the context agnostic method will learn to reach one room always with a high successful rate so the average is roughly $1/3$, but it will not be the case in Section~\ref{exp:cell} when we have more test contexts. 
% \vspace{-3mm}
% \ying{Notice that the reward function is also harder to learn than other settings }
% \begin{table}[ht]
% \centering
% \scalebox{0.9}{
% \begin{tabular}{l|cccc|c}
% \toprule
% Env/Method   & \algo (Ours)     & PDS      & UDS+RP& Goal Prediction&Context-agnostic IQL \\
% \hline
% medium & \textbf{78.2±1.2} & 39.0±5.2 & 14.0±0.9  &59.5±2.4&32.6±0.8        \\
% large  & \textbf{73.3±1.9} & 15.1±3.5 & 21.6±21.3 &40.3±3.3&28.1±0.3 \\
%  \bottomrule
% \end{tabular}
% }
% \caption{Average scores with standard errors over 5 random seeds from Four Rooms. The score for each run is the average success rate (\%) of the other three rooms.}
% \label{tab:four_rooms}
%  \vspace{-4mm}
% \end{table}

\begin{table*}[h]
\centering
\vspace{-0.5em}
\caption{Average scores from Four Rooms with perturbation. The score for each run is the average success rate (\%) of the other three rooms.}
\scalebox{0.9}{
\begin{tabular}{l|ccc|cc}
\toprule
Env/Method   & \algo (Ours)     & PDS      & Goal Prediction &Oracle Reward\\
\hline
medium-play & \textbf{78.7±0.9} &46.0±4.47 & 59.3±2.6&77.7±2.0 \\
medium-diverse & \textbf{83.6±1.9}&51.3±3.6 & 66.7±2.4 & 87.4±1.2\\
large-play  &\textbf{65.5±2.5}&13.9±2.4&41.4±3.6 &67.2±2.7\\
large-diverse &\textbf{72.2±2.9} &11.1±3.8&42.0±3.0 & 69.6±3.1\\
\hline
average & \textbf{75.0} & 30.6& 52.4 &75.5\\
 \bottomrule
\end{tabular}
\vspace{-10mm}
}

\label{tab:four_rooms}
% \vspace{-5mm}
\end{table*}

\textbf{Random Cells: Figure~\ref{fig:cgo}(c).} We use a diverse distribution of contexts as shown in Figure~\ref{fig:cgo}(c), where the contexts are randomly sampled from non-wall states. 
% To construct the goal set given context, we obtain states with the 2D locations within the $L_2$ ball with a certain radius. 
For test contexts, we have two settings: 1) sampling from the training distribution; 2) sampling from a far-away area from the start states. 

% \vspace{-1mm}
% \paragraph{Overall, \algo outperforms the baselines under the setup in Figure~\ref{fig:cgo}(c).}
Overall, \algo outperforms the baselines under the setup in Figure~\ref{fig:cgo}(c).
We show the normalized return (average success rate in percentage) in each modified Random Cells environment in Table~\ref{tab:cell_same}, which also shows the generalization ability of our method in the context space. 
% \begin{table}[ht]
% \centering
% \scalebox{0.9}{
% \begin{tabular}{l|cccc|c}
% \toprule
% Env/Method   & \algo (Ours)     & PDS      & UDS+RP &Goal Prediction& Context-agnostic IQL \\
% \hline
% medium & \textbf{70.5±8.7} & 61.0±9.7 & 14.8±5.8 & 60.0±9.65& 18.8±5.5      \\
% large  & \textbf{55.0±9.3} & 51.2±3.6 & 10.1±3.5 &40.6±7.1 &17.8±3.7  \\
%  \bottomrule
% \end{tabular}
% }
% \caption{
% % Train with random cells
% % (15 for medium, 18 for large)
% Average scores with standard errors over 5 random seeds from Random Cells. The score for each run is the average success rate (\%) of 5 random test contexts from the same training distribution.
% % , and the evaluation of each context requires 100 episodes. We use the same random seeds for each algorithm to generate the same random training and test contexts
% }
% \label{tab:cell_same}
%  \vspace{-4mm}
% \end{table}
% \vspace{-3mm}
% \paragraph{\algo also generalizes to a different test context distribution.}
\algo also generalizes to a different test context distribution: We also test with a distribution shift of the contexts in Table~\ref{tab:cell_diff}. We can observe that when tested with this different context distribution, \algo still generates better overall results compared to reward learning and goal prediction baselines.
% \begin{table}[H]
% \centering
% \scalebox{0.9}{
% \begin{tabular}{l|cccc|c}
% \toprule
% Env/Method   & \algo (Ours)     & PDS      & UDS+RP &Goal Prediction& Context-agnostic IQL \\
% \hline
% medium & 63.8±11.9&45.7±13.3&2.2±0.9&\textbf{72.6±2.6}& 4.3±1.7          \\
% large  &  \textbf{62.6±6.4}&48.0±7.3&1.1±0.6&51.8±5.4&0.8±0.8\\
%  \bottomrule
% \end{tabular}
% }
% \caption{Average scores with standard errors over 5 random seeds from Random Cells. The score for each run is the average success rate (\%) of 5 random test contexts of cells far away from the start.
% % , and the evaluation of each context requires 100 episodes. We use the same random seeds for each algorithm to generate the same random training and test contexts.
% }
% \vspace{-4mm}
% \end{table}
% \subsection{Discussions}

% \subsection{OOD goal examples}
% % add out-of-range states in context-goal dataset
% % ours, pds, UDS + reward learning
% \begin{table}[H]
% \centering
% \scalebox{0.9}{
% \begin{tabular}{l|ccc}
% \toprule
% Env/Method   & Ours      & PDS      & Reward learning (UDS) \\
% \hline
% medium & 78.9±3.3  & 23.5±2.5 & 13.4±2.5               \\
% large  & 70.0±11.5 & 9.0±5.3  & 22.5±1.8  \\ 
% \bottomrule
% \end{tabular}
% }
% \end{table}

% \vspace{-3mm}
% \st{For PDS, we can observe that the reward distribution for positive and negative samples are better separated in the large one than the medium one, explaining that it has better performance in the large-diverse environment than the medium-diverse one.}

\begin{table*}[h]
\centering
\caption{{Average scores from Random Cells. The score for each run is the average success rate (\%) of random test contexts from the same training distribution.}}
\scalebox{0.9}{
\begin{tabular}{l|ccc|cc}
\toprule
Env/Method   & \algo (Ours)     & PDS      & Goal Prediction &Oracle Reward\\
\hline
medium-play & \textbf{76.8±6.1}&52.0±8.8&66.7±7.2 &71.9±0.1\\
medium-diverse &\textbf{78.2±6.5}&60.9±11.3 & 69.7±8.7&79.3±6.1 \\
large-play  &\textbf{57.6±12.4} & 50.6±6.4 & 42.4±8.2 &49.4±9.3\\
large-diverse & 54.7±8.8 & \textbf{58.3±9.2} & 44.2±8.1&58.2±3.4\\
\hline
average &\textbf{66.8}&55.5&55.8 & 64.7\\
 \bottomrule
\end{tabular}
}

\label{tab:cell_same}
% \vspace{-5mm}
\end{table*}

\begin{table*}[h]
\centering
\caption{{Average scores from Random Cells with perturbation. The score for each run is the average success rate (\%) of random test contexts with a distribution shift.}}
\scalebox{0.9}{
\begin{tabular}{l|ccc|cc}
\toprule
Env/Method   & \algo (Ours)     & PDS      & Goal Prediction &Oracle Reward\\
\hline
medium-play & 67.9±8.2 & 50.1±13.4&\textbf{70.5±1.9} &67.2±7.2\\
medium-diverse &\textbf{72.5±6.5}&57.5±14.8&63.0±7.2 &68.7±7.9\\
large-play  &\textbf{60.2±4.8}&48.1±8.0&44.3±4.1&59.8±4.4\\
large-diverse &\textbf{58.0±5.8}&44.1±9.9&55.4±5.7 &57.6±7.6\\
 \hline
average & \textbf{64.7 }& 49.9 & 58.3 & 63.3 \\
 \bottomrule
\end{tabular}
}
\label{tab:cell_diff}
% \vspace{-5mm}
\end{table*}
\textbf{Reference to training with oracle reward.} Notice that training with oracle reward is the skyline performance. From the results, training with oracle reward does not generally improve the performance much compared to \algo, though it generally outperforms PDS and Goal Prediction. This is mainly due to the sparsity of the positive samples in the randomly sampled context-transition pairs. On the other hand, \algo easily uses these positive examples via our augmentation, which is another advantage of our method over reward prediction baselines.

\textbf{Evaluation of the Reward Model.} We also visualize the learned reward model from reward learning baselines in Appendix~\ref{app:reward_eval}: PDS is consistently better at separating positive and negative datasets than UDS and naive RP, but PDS can still fail at fully separating positive and negative examples. Intuitively, our method does not require reward learning thanks to the construction of the augmented MDP, which avoids the extra errors in reward prediction and leads to better performance.

\vspace{-0.25em}
\subsection{Discussion and Limitation}
\vspace{-0.25em}
\label{sec:limitation}
Our experiments are limited to low-dimensional simulations. Nevertheless, the success of our method with diverse context-goal relationships serves as a first milestone to showcase its effectiveness, and we believe \algo would be useful in real-world settings (e.g., learning visual-language robot policies) for its simplicity and theoretical guarantees. Potential scaling up by incorporating features from large pretrained models would be an exciting future direction, which can make our method generalizable to the real world.

\vspace{-2mm}
\section{Conclusion}
\vspace{-0.5em}
We propose \algo for offline CGO problems, and prove \algo can learn near-optimal policies without the need for negative labels with natural assumptions. We also validate the efficacy of \algo experimentally, and find it outperforms other reward-learning and goal prediction baselines across various CGO complexities. We believe our method has the potential to generalize to real-world applications by further scaling up.
%Note that our context-goal dataset construction is designed in a way that it has overlapped supports with the dynamics dataset. 
% We highlight \algo works under certain assumptions. As shown in our theoretical result in Section~\ref{sec:analysis}, the \algo technique would fail 
% 1) if the dynamics dataset does not contain trajectories leading to the goal set of a given context, 2) the context-goal dataset does not cover the contexts and goals faced at test time, or 3) if the goal set does not cover reachable goals from initial states. 

% Nevertheless, we believe our empirical results 

% \section{Impact Statements}
% This paper presents work whose goal is to advance the field of offline Reinforcement Learning. There are many potential societal consequences of our work, none of which we feel must be specifically highlighted here.

% \ying{we verify the two questions at the the beginning of the exp section}
% \ying{mention the method would fail if goal sets are not covered by traj}

% \ying{potential experiments to add for rebuttal: 1) perturbing the goal sets in robo states, 2) failing case 3) regularizing v without fake action, 4) UDS, 5)goal prediction.}

% \ying{future work: exp on language as contexts, visual examples as goal sets}

% \section{Conclusion}

\label{sec:conclusion}

\clearpage
\bibliographystyle{plainnat}
\bibliography{icml_2024}
% \onecolumn
\appendix

\clearpage
\section{Theoretical Analysis}  \label{app:theory}

In this section, we provide a detailed analysis for the instantiation of \algo using PSPI \citep{xie2021bellman}. We follow the same notation for the value functions, augmented MDP, and extended function classes as stated in \cref{sec:background} and \cref{sec:approach} in the main text.

\subsection{Equivalence Relations between Original and Augmented MDP}
\label{app:augmented_mdp}

We begin by showing that the optimal policy and any value function in the augmented MDP can be expressed using their analog in the original MDP. With the augmented MDP defined as $\overline{\MM} \coloneq (\bar\XX, \bar\AA, \bar{R}, \bar{P}, \gamma)$ in \cref{sec:augmented_mdp}, we first define the value function in the augmented MDP. For a policy $\bar{\pi}: \bar\XX \to \bar\AA$, we define the Q function for the augmented MDP as 
\begin{align*}
    \bar{Q}^{\bar{\pi}}(x,a) \coloneqq \E_{\bar\pi,\bar{P}} \left[ \sum_{t=0}^\infty \gamma^t \bar{R}(x,a) | x_0 = x, a_0 = a  \right]
\end{align*}
Notice that we don't have a reaching time random variable $T$ in this definition; instead the agent would enter an absorbing state $s^+$ after taking $a^+$ in the augmented MDP. We can define similarly $\bar{V}^{\bar\pi}(s) \coloneq \bar{Q}^{\bar{\pi}}(x,\bar{\pi})$.

\begin{remark}
Let $\bar{Q}^\pi_R$ be the extension of $Q^\pi$ based on $R$.
We have, for $x \notin G$, $\bar{Q}^\pi_{{R}} (x,a) = \bar{Q}^{\bar{\pi}} (x,a)$ $\forall a \in \bar\AA$, and for $x \in G$, $\bar{Q}^\pi_{{R}} (x,a) =  \bar{Q}^{\bar{\pi}} (x,a^+) = 1$, $\forall a\in\bar\AA$.
\end{remark}

By the construction of the augmented MDP, it is obvious that the following is true.
\begin{lemma} \label{lm:equivalence}
    Given $\pi:\XX\to\Delta(\AA)$, let $\bar\pi$ be its extension.
    For any $h:\XX\times\AA\to \R$, it holds
    \begin{align*}
        \E_{\pi,P} \left[ \sum_{t=0}^T \gamma^t h(x,a) \right] = \E_{\bar\pi,\bar{P}} \left[ \sum_{t=0}^\infty \gamma^t \tilde{h}^\pi(x,a) |  x \notin \XX^+ \right]
    \end{align*}
    where $T$ is the goal-reaching time (random variable) and we define $\tilde{h}^\pi(x,a^+)=h(x,\pi)$. 
\end{lemma}
% \chingan{Relate to $d^\pi$}

%\adith{Nitpick: Off-by-one error? The time-step $T$ that a policy, after landing in $s \in G_c$, takes any action and gets reward of $1$ should be the same time-step that the augmented policy takes $a^+$ action getting the reward of $1$, so that the discounting works out ok. However, the original $Q^\pi(s,c,a)$ definition may need the summation to go to $T+1$ because it is only on the timestep after entering $G_c$ that the policy should receive the reward in the original MDP? }

We can now relate the value functions between the two MDPs. 
\begin{proposition} \label{lm:Q equivalence}
% \chingan{corrected}
For a policy $\pi:\XX\to\Delta(\AA)$, let $\bar{\pi}$ be its extension (defined above). We have for all $x\in\XX$, $a\in\AA$,  
\begin{align*}
    Q^\pi(x,a) &\geq \bar{Q}^{\bar{\pi}}(x,a)\\
    V^\pi(x) &= \bar{V}^{\bar{\pi}}(x)
\end{align*}
Conversely, for a policy $\xi:\bar\XX\to\Delta(\bar{\AA})$, define its restriction $\underline{\xi}$ on $\XX$ and $\AA$ by translating probability of $\xi$ originally on $a^+$ to be uniform over $\AA$.  Then we have for all $s\in\SS$, $a\in\AA$
\begin{align*}
    Q^{\underline\xi}(x,a) &\geq  \bar{Q}^{\xi}(x,a)\\
    V^{\underline\xi}(x) &\geq  \bar{V}^{\xi}(x)
\end{align*}
\end{proposition}
\begin{proof}
The first direction follows from \cref{lm:equivalence}. For the latter, whenever $\xi$ takes $a^+$ at some $x\notin G$, it has $\bar{V}^\xi(x) = 0$ but $\bar{V}^{\underline\xi}(x)\geq 0$ since there is no negative reward in the original MDP.  By performing a telescoping argument, we can derive the second claim.
\end{proof}

By this lemma, we know the extension of $\pi^*$ (i.e., $\bar{\pi}^*$) is also optimal to the augmented MDP and $V^*(x) = \bar{V}^*(x)$ for $x\in\XX$. Furthermore, we have a reduction that we can solve for the optimal policy in the original MDP by the solving augmented MDP, since 
\begin{align*}
    V^{\underline\xi}(d_0) - V^*(d_0) \leq V^{\xi}(d_0) - \bar{V}^*(d_0)
\end{align*}
for all $\xi:\bar{\XX}\to\Delta(\bar{\AA})$. 
In particular, 
\begin{align}
    \text{Regret}(\pi) \coloneq  V^{\pi}(d_0) - V^*(d_0) = V^{\bar\pi}(d_0) - \bar{V}^*(d_0) \eqcolon \overline{\text{Regret}}(\bar\pi)
\end{align}

Since the augmented MDP replaces the random reaching time construction with an absorbing-state version, the Q function $\bar{Q}^{\bar\pi}$ of the extended policy $\bar\pi$ satisfies the Bellman equation
\begin{align}
    \bar{Q}^{\bar\pi}(x,a) 
    &= \bar{R}(x,a) + \gamma \mathbb{E}_{x'\sim \bar{P}(\cdot|x,a)} [ \bar{Q}^{\pi}(x',\bar{\pi}) ] \nonumber\\
    &\eqcolon \bar{\TT}^\pi \bar{Q}^{\pi}(x,a)    \label{eq:augment_BE}
\end{align}
For $x\in\XX$ and $a\in\AA$, we show how the above equation can be rewritten in $Q^\pi$ and $R$. 
% 
% We have for $x\in\XX$ and $a\in\AA$, by \cref{lm:max Q and r},
% \begin{align*}
%     Q^{\pi}(x,a) 
%     &= \bar{Q}^{\bar\pi}(x,a) = R(x) + \E_{x'\sim \bar{P}(\cdot|x,a)} [   \max(R(x'),Q^{\pi}(x', \pi))  ]    
% \end{align*}
% where $R(x) = R(s,c)$
% For $x\in \XX^+$, $ \bar{Q}^{\bar\pi}(x,a) = 0$.
\begin{proposition} \label{eq:Bellman eqation mod}
% \chingan{corrected: should be $0$ instead of $R(x)$}
For $x\in\XX$ and $a\in\AA$,
    \begin{align*}
    \bar{Q}^{\bar\pi}(x,a) = 0 + \gamma \E_{x'\sim \bar{P}(\cdot|x,a)} [   \max(R(x'),Q^{\pi}(x', \pi))  ]    
\end{align*}
For $a = a^+$, $\bar{Q}^{\bar\pi}(x,a^+) = \bar{R}(x,a^+) = R(x)$.
For $x \in \XX^+$, $\bar{Q}^{\bar\pi}(x,a) = 0$.
\end{proposition}

\begin{proof}
    The proof follows from \cref{lm:max Q and r} and the definition of $\bar{P}$.
\end{proof}
\begin{lemma} \label{lm:max Q and r}
For $x\in\XX$, $\bar{Q}^{\bar\pi}(x,\bar{\pi})  = \max(R(x),Q^{\pi}(x, \pi))$
\end{lemma}
\begin{proof}
     For $x\in\XX$,
\begin{align*}
     \bar{Q}^{\bar\pi}(x,\bar{\pi}) 
     &= 
     \begin{cases}
         \bar{Q}^{\bar\pi}(x, a^+), \quad &\text{if $x\in G$}\\
         \bar{Q}^{\bar\pi}(x, \pi), \quad &\text{otherwise}
     \end{cases} &\text{(Because of definition of $\bar{\pi}$) }
     \\
     &= 
     \begin{cases}
         \bar{Q}^{\bar\pi}(x, a^+), \quad &\text{if $x\in G$}\\
         Q^{\pi}(x, \pi), \quad &\text{otherwise}
     \end{cases}    \qquad &\text{(Because of \cref{lm:Q equivalence}) }
     \\
     &=
     \begin{cases}
         \bar{R}(x,a^+ ), \quad &\text{if $x\in G$}\\
         Q^{\pi}(x, \pi), \quad &\text{otherwise}
     \end{cases}    \qquad &\text{(Definition of augmented MDP) }
  \\
  &=
     \begin{cases}
         R(x), \quad &\text{if $x\in G$}\\
         Q^{\pi}(x, \pi), \quad &\text{otherwise}
     \end{cases}  \\
     % &= \max(\bar{R}(x,a^+ ),Q^{\pi}(x, \pi))\\
     &= \max(R(x),Q^{\pi}(x, \pi))
\end{align*}
where in the last step we use $\bar{R}(x)=1$ for $x\in G$ and $\bar{R}(x) = 0$ otherwise.
\end{proof}

% We also define the reframed Bellman backup operator
% \begin{align*}
%     \KK^\pi \bar{f}_g (x,a) \coloneq R(x) + \gamma \E_{x'\sim P(\cdot|x,a)} [ \max(g(x'),f(x',\pi)) ]
% \end{align*}
% We summarize the above as follows
% \begin{proposition}
% \chingan{this summarization is wrong}
% For $x\in\XX$ and $a\in\AA$
% \begin{align}
%     Q^\pi(x,a) = \TT^\pi Q^\pi(x,a) =  \KK^\pi \bar{Q}_R^\pi (x,a)
% \end{align}    
% \end{proposition}

\subsection{Function Approximator Assumptions}

In \cref{th:main theorm pspi (informal)}, we assume access to a policy class $\Pi = \{ \pi: \XX \to \Delta(\AA)\}$.
We also assume  access to a function class $\FF = \{ f: \XX\times\AA \to [0,1]\}$ and a function class $\GG = \{ g: \XX \to [0,1]\}$. We can think of them as approximators for the Q function and the reward function of the original MDP. 

For an action value function $f:\XX\times\AA \to [0,1]$, define its extension:
\begin{align}
    \bar{f_g}(x,a) = 
    \begin{cases}
        g(x), & \text{$a= a^+$ and $x\notin\XX^+$}\\
        0, & \text{$x \in \XX^+$}\\
        f(x,a), & \text{otherwise}.
    \end{cases}
\end{align}
The extension of $f$ is based on a state value function $ g:\XX \to [0,1]$ which determines the action value of $x$ only at $a^+$. One could also view $g(x)$ as a goal indicator: after taking $a^+$ the agent would always transit to the zero-reward absorbing state $s^+$, so $g(x) = \bar{R}(x,a^+)$ which is the indicator of whether $s \in G_c$.

% Let $d^\pi(x,a) = (1-\gamma)\sum_{t=0}^\infty \gamma^t d_t^\pi(x,a)$, where $d_t^\pi$ is the state-action distribution of running policy $\pi$ in the MDP $\MM = (\XX,\AA,r,P,\gamma)$. 
%
% Let $\Omega \coloneq \{ d^\pi : \pi \in \Pi \} $ denote the set of all possible average state-action distributions. 

Recall the zero-reward Bellman backup operator $\TT^\pi$ with respect to $P(s'|s,a)$ as defined in \cref{as:completeness}:
\begin{align*}
    \TT^\pi f (x,a) \coloneq \gamma \E_{x'\sim P_0(\cdot|x,a)} [ f(x',\pi) ]
\end{align*}
where $P_0(x'|x,a) \coloneq P(s'|s,a) \one(c'=c)$. Note this definition is different from the one with absorbing state $s^+$ in \cref{sec:formulation}. Using this modified backup operator, we can show that the following realizability assumption is true for the augmented MDP: %\chingan{updated}

\begin{proposition}[Realizability] \label{lm:realizability}
By \cref{as:realizability} and \cref{as:completeness}, there is $f\in\FF$ and $g\in\GG$ such that $\bar{Q}^{\bar{\pi}} = \bar{f}_g$. 
\end{proposition}
\begin{proof}
By \cref{as:completeness}, there is $h\in\FF$ such that $h(x,a) =  \max(R(x),Q^{\pi}(x, a)) $.
By \cref{eq:Bellman eqation mod}, we have for $x\in\XX$, $a\neq a^+$
\begin{align*}
    \bar{Q}^{\bar\pi}(x,a) 
    &= 0 + \gamma \E_{x'\sim \bar{P}(\cdot|x,a)} [   \max(R(x'),Q^{\pi}(x', \pi))  ]    \\
    &= 0 + \gamma \E_{x'\sim P_0(\cdot|x,a)} [  h(x,\pi)  ]  \\
    &= \TT^\pi h \in \FF
\end{align*}
For $a=a^*$, we have
$ \bar{Q}^{\bar\pi}(x,a^*) = \bar{R}(x,a^+) = R(x) \in \GG$. 
Finally $\bar{Q}^{\bar\pi}(x^+,a) = 0$ for $x^+ \in \XX^+$. 
Therefore, $\bar{Q}^{\bar{\pi}} = \bar{f}_g$ for some $f\in\FF$ and $g\in\GG$.
\end{proof}

\subsection{CODA+PSPI Algorithm}
\label{app:sds+pspi}

In this section, we describe the instantiation of PSPI with \algo in detail along with the necessary notation. The main theoretical result and its proof is then given in Section \ref{sec:analysis_app}. As discussed in \cref{sec:analysis}, our algorithm is based on the idea of reduction, which turns the offline CGO problem into a standard offline RL problem in the augmented MDP. To this end, we construct augmented datasets $\bar D_\dyn$ and $\bar D_\goal$ in \cref{alg:sds} as follows:
\begin{align*}
    \bar{D}_\dyn &= \{ (x_n,a_n,r_n,x_n') |  r_n=0, x_n = (s_i, c_j), x_n' = (s_i', c_j), a_n = a_i, (s_i,a_i,s_i') \in D_\dyn, (\cdot, c_j) \in D_\goal \}\\
    \bar{D}_\goal &= \{ (x_n, a^+,r_n, x_n^+) | r_n=1, x_n = (s_n, c_n), x_n^+ = (s^+, c_n), (s_n, c_n)\in D_\goal  \}
\end{align*} 
With this construction, we have: $\bar{D}_\dyn \sim \mu_\dyn(s,a,s') \mu_\goal(c)$ and $\bar{D}_\goal \sim \mu_\goal(c,s) \one(a=a^+)\one(s'=s^+) $.
We use the notation, $\bar\mu_\dyn(x,a,x') = \mu_\dyn(s,a,s') \mu_\goal(c)$ and $\bar \mu_\goal(x,a,x') = \mu_\goal(c,s) \one(a=a^+)\one(s'=s^+)$. We will also use the notation $x_{ij}\equiv(s_i,c_j)$, $x'_{ij}\equiv (s'_i,c_j)$ in the above construction.
These two datasets have the standard tuple format, so we can run offline RL on $\bar{D}_\dyn \bigcup \bar{D}_\goal$. Also, note that $|\bar D_{\dyn}| = |D_\dyn||D_\goal|$ and $|\bar D_\goal| = |D_\goal|$.

% \chingan{Need to update the presentation a bit.}
\paragraph{PSPI.}
We consider the information theoretic version of PSPI~\citep{xie2021bellman} which can be summarized as follows: For an MDP $(\XX,\AA,R,P,\gamma)$, given a tuple dataset $D =\{ (x,a,r,x')\}$, a policy class $\Pi$, and a value class $\FF$, it finds the policy through solving the two-player game:
\begin{align}\label{eq:pspi_app}
    \max_{\pi\in\Pi}  \min_{f\in\FF} \quad f(d_0, \pi) \qquad \text{s.t.}  \qquad 
    \ell(f,f;\pi,D) - \min_{f'\in\FF} \ell(f',f;\pi,D) \leq \epsilon_b
\end{align}
where $f(d_0, \pi) = \E_{x_0 \sim d_0}[f(x_0, \pi)]$, $\ell(f,f';\pi,D) \coloneq \frac{1}{|D|} \sum_{(x,a,r,x')\in D}  (f(x,a) - r - f'(x',\pi))^2
$.
The term $\ell(f,f;\pi,D) - \min_{f'} \ell(f',f;\pi,D)$ in the constraint is an empirical estimation of the Bellman error on $f$ with respect to $\pi$ on the data distribution $\mu$, i.e. $ \E_{x,a\sim\mu} [(f(x,a) - \TT^\pi f(x,a))^2]$. It constrains the Bellman error to be small, since $ \E_{x,a\sim\mu} [(Q^\pi(x,a) - \TT^\pi Q^\pi(x,a))^2] = 0$.

\paragraph{CODA+PSPI.}
Below we show how to run PSPI to solve the augmented MDP with offline dataset $\bar{D}_\dyn \bigcup \bar{D}_\goal$.
To this end, we extend the policy class from $\Pi$ to $\bar{\Pi}$, and the value class from $\FF$ to $\bar{\FF}_\GG$ using the function class $\GG$ based on the extensions defined in \cref{sec:augmented_mdp}.
One natural attempt is to implement \eqref{eq:pspi_app} with the extended policy and value classes  $\bar{\Pi}$ and $\bar{\FF}$ and $\bar D = \bar{D}_\dyn  \bigcup \bar{D}_\goal$.
This would lead to the two player game: 
\begin{align}\label{eq:pspi (augemnted)}
    \max_{\bar\pi\in\bar\Pi} \min_{\bar{f}_g\in\bar\FF_\GG} \quad 
    \bar{f}_g (d_0, \bar\pi) \qquad \text{s.t.}  \qquad 
    \ell( \bar{f}_g, \bar{f}_g ;\bar\pi,\bar{D}) - \min_{\bar{f}_{g'}'\in\bar\FF_\GG} \ell(\bar{f}_{g'}',\bar{f}_g;\bar\pi,\bar{D}) \leq \epsilon_b
\end{align}
However, \eqref{eq:pspi (augemnted)} is not a well-defined algorithm, because its usage of the extended policy $\bar\pi$ in the constraint requires knowledge of $G$, which is unknown to the agent. 

Fortunately, we show that \eqref{eq:pspi (augemnted)} can be slightly modified so that the implementation does not actually require knowing $G$. 
Here we use a property (\cref{eq:Bellman eqation mod}) that the Bellman equation of the augmented MDP: 
\begin{align*}
    \bar{Q}^{\bar\pi}(x,a) 
    &= \bar{R}(x,a) + \gamma \mathbb{E}_{x'\sim \bar{P}(\cdot|x,a)} [ \bar{Q}^{\pi}(x',\bar{\pi}) ] \\
    &= 0 + \gamma \E_{x'\sim \bar{P}(\cdot|x,a)} [   \max(R(x'),Q^{\pi}(x', \pi))  ]    
    % &= Q^\pi(x,a)
\end{align*}
for $x\in\XX$ and $a\neq a^+$, and  $    \bar{Q}^{\bar\pi}(x,a) = 1 $ for $x\in G$ and $a=a^+$.

We can rewrite the squared Bellman error on these two data distributions, $\bar D_\dyn$ and $\bar D_\goal$, using the Bellman backup defined on the augmented MDP (see eq.\ref{eq:augment_BE}) as below:
\begin{align}
 \E_{\mu_\dyn} [(\bar{Q}^{\bar\pi}(x,a) - \bar\TT^{\bar\pi} \bar{Q}^{\bar\pi}(x,a))^2] \nonumber
 = {} &  
 \E_{\mu_\dyn} [(\bar{Q}^{\bar\pi}(x,a) - 0 - \gamma \E_{x'\sim \bar{P}(\cdot|x,a)} [ \max(R(x),Q^{\pi}(x, \pi)) ] )^2]
\end{align}
\begin{align}
 \E_{x,a\sim\mu_\goal} [(\bar{Q}^{\bar\pi}(x,a) - \bar\TT^{\bar\pi} \bar{Q}^{\bar\pi}(x,a))^2] 
 = {} & \E_{x,a\sim\mu_\goal} [(\bar{Q}^{\bar\pi}(x,a^+) - 1)^2] \nonumber
\end{align}

We can construct an approximator $\bar{f}_g(x,a)$ for $\bar{Q}^{\bar\pi}(x,a)$. 
Substituting the estimator $\bar{f}_g(x,a)$  for $\bar{Q}^{\bar\pi}(x,a)$ in the squared Bellman errors above and approximating them by finite samples, we derive the empirical losses below. 
 
\begin{align} \label{eq:empirical losses app}
    \ell_\dyn(\bar{f}_g,\bar{f}'_{g'}; \bar\pi) &\coloneq \frac{1}{|\bar D_\dyn|} \sum_{(x,a,r,x')\in \bar D_\dyn} ( f(x,a) - \gamma    \max(g'(x'), f'(x', \pi))    )^2\\
    \ell_\goal(\bar{f}_g) &\coloneq \frac{1}{|\bar D_\goal|} \sum_{(x,a,r,x')\in \bar D_\goal} ( g(x) -1   )^2
\end{align}
where we have $\bar{f}_g(x,a) = f(x,a) \one(a\neq a^+) + g(x) \one(a=a^+)$ for $x\notin\XX^+$. 
% \chingan{Relate this to our IQL implementation}

Using this loss, we define the two-player game of PSPI for the augmented MDP: 
\begin{align}\label{eq:pspi (augemnted) (real) app}
    &\max_{\pi\in\Pi} \min_{\bar{f}_g\in\bar\FF} \bar{f}_g (d_0, \bar\pi) \qquad \\
    \text{s.t.} \quad  & \ell_\dyn( \bar{f}_g, \bar{f}_g ;\bar\pi) - \min_{\bar{f}_{g'}'\in\bar\FF} \nonumber \ell_\dyn(\bar{f}_{g'}',\bar{f}_g;\bar\pi) \leq \epsilon_\dyn \\
    &  \ell_\goal (\bar{f}_g) \leq  0 \nonumber
\end{align}
Notice $\bar{f}_g (d_0, \bar\pi) = f(d_0, \pi)$. Therefore, this problem can be solved using samples from $D$ without knowing $G$.

\subsection{Analysis of CODA+PSPI}
\label{sec:analysis_app}

\paragraph{Covering number.}
We first define the covering number on the function classes $\FF$, $\GG$, and $\Pi$\footnote{For finite function classes, the resulting performance guarantee will depend on $|\FF|, |\GG|$ and $|\Pi|$ instead of the covering numbers as stated in \cref{th:main theorm pspi (informal)}.}. 
For $\FF$ and $\GG$, we use the $L_\infty$ metric. We use $\NN_\infty(\FF,\epsilon)$ and $\NN_\infty(\GG,\epsilon)$ to denote the their $\epsilon$-covering numbers. 
For $\Pi$, we use the $L_\infty$-$L_1$ metric, i.e., $\| \pi_1 -  \pi_2\|_{\infty,1} \coloneq  \sup_{x\in\XX} \| \pi_1(\cdot|s) - \pi_2(\cdot|s) \|_1$. We use $\NN_{\infty,1}(\Pi,\epsilon)$ to denote its $\epsilon$-covering number. 

\paragraph{High-probability events.}
In CODA+PSPI (eq. \ref{eq:pspi (augemnted) (real) app}), we choose the policy in class $\pi$ which has the best \emph{pessimistic} value function estimate. In order to show this, we will need two high probability results (we defer their proofs to Section~\ref{sec:hp_proof}). To that end, we will use the following notation for the expected value of the empirical losses:
\begin{align*}
    \ell_{\bar \mu_{\dyn}}(\bar{f}_g,\bar{f}'_{g'}; \bar\pi) &\coloneq \E_{(x,a,x') \sim \bar\mu_{\dyn}} ( f(x,a) - \gamma    \max(g'(x'), f'(x', \pi))    )^2\\
    \ell_{\bar \mu_{\goal}}(\bar{f}_g) &\coloneq \E_{(x,a^+,x^+)\sim \bar \mu_{\goal}} ( g(x) -1   )^2
\end{align*}

First, we show that for any policy $\pi \in \Pi$, the true value function $\bar{Q}^{\bar{\pi}}$ satisfies the two empirical constraints specified in eq.~\eqref{eq:pspi (augemnted) (real) app}.
\begin{lemma}\label{lm:Qpi in the set}
With probability at least $1-\delta$, it holds for all $\pi\in\Pi$,
\begin{align*}
    \ell_\dyn( \bar{Q}^{\bar{\pi}},  \bar{Q}^{\bar{\pi}} ;\bar\pi) - \min_{\bar{f}_{g'}'\in\bar\FF} \ell_\dyn(\bar{f}_{g'}', \bar{Q}^{\bar{\pi}} ;\bar\pi) \leq {} O\left( \left( \sqrt{\frac{\square} {|D_\dyn|}} +  \sqrt{\frac{\square} {|D_\goal|}}\right)^2 \right)
\end{align*}
\begin{align*}
    \ell_\goal (\bar{Q}^{\bar{\pi}}) \leq {} 0  
\end{align*}
%\aditya{Is $\tilde f^\pi_g$ supposed to be $\bar f_R$?}
%where $\bar{Q}_R^\pi $ denotes the extension of $\bar{Q}^\pi$ based on $R$, and
where\footnote{Technically, we can remove $ \NN_\infty\left(\GG, \frac{1}{|D_\dyn||D_\goal|}\right) $ in the upper bound, but we include it here for a cleaner presentation.} $\square \equiv \log\left(\frac{\NN_\infty\left(\FF, 1/|D_\goal||D_\dyn|\right) \NN_\infty\left(\GG, 1/|D_\goal||D_\dyn|\right) \NN_{\infty,1}\left(\Pi,1/|D_\goal||D_\dyn|\right)}{\delta}\right)$.
\end{lemma}
We use the notation $\epsilon_\dyn \coloneq \left( \sqrt{\frac{\square} {|D_\dyn|}} +  \sqrt{\frac{\square} {|D_\goal|}}\right)^2$ for the first upper bound in Lemma~\ref{lm:Qpi in the set}.

Next, we show that for every pair of value function $\bar{f}_g \in \bar \FF$ and policy $\bar \pi \in \bar \Pi$ which satisfies the constraints in eq.~\eqref{eq:pspi (augemnted) (real) app}, the empirical estimates provide a bound on the population error with high probability.
\begin{lemma}\label{lm:generalization error)}
For all $f\in\FF,g\in\GG$ and $\pi \in \Pi$ satisfying 
\begin{align*}
        &\ell_\dyn( \bar{f}_g, \bar{f}_g ;\bar\pi) - \min_{\bar{f}_{g'}'\in\bar\FF} \ell_\dyn(\bar{f}_{g'}',\bar{f}_g;\bar\pi) \leq \epsilon_\dyn
\end{align*}
\begin{align*}
    &\ell_\goal (\bar{f}_g) \leq  0,
\end{align*}
with probability at least $1-\delta$, we have:
\begin{align*}
    &\left\| \bar{f}_g(x,a) - \gamma \E_{x'\sim \bar{P}(\cdot|x,a)} \left[\max(g(x'), f(x',\pi))\right] \right\|_{\bar \mu_\dyn}
    \leq O\left( \sqrt{\epsilon_\dyn} \right) 
\end{align*}
\begin{align*}
    &\left\| g(x) - 1\right\|_{\bar \mu_\goal}
    \leq O\left(   \sqrt{\frac{\log \frac{ \NN_\infty(\GG, 1/|D_\goal|)}{\delta} }{|D_\goal|}} \right) \eqcolon \sqrt{\epsilon_\goal}
\end{align*}
\end{lemma}

\paragraph{Pessimistic estimate.}
Our next step is to show that the solution of the constrained optimization problem in equation \ref{eq:pspi (augemnted) (real) app} is pessimistic and that the amount of pessimism is bounded. 
\begin{lemma}
Given $\pi$, let $\bar{f}_g^{\pi}$ denote the minimizer in \eqref{eq:pspi (augemnted) (real) app}. 
With high probability, 
    $\bar{f}_g^\pi (d_0, \bar\pi) \leq Q^{\pi}(d_0, \pi) $%\sqrt{\epsilon_{\FF,\GG}}$
\end{lemma}
\begin{proof}
By \cref{lm:Qpi in the set}, for any policy $\pi \in \Pi$, we know that $\bar Q^{\bar \pi}_R$ satisfies the constraints in equation \eqref{eq:pspi (augemnted) (real) app}. Therefore, we have
\begin{align*}
    \bar{f}_g^\pi (d_0, \bar\pi) \leq \bar{Q}_R^{\pi}(d_0, \bar\pi) = Q^{\pi}(d_0, \pi). 
\end{align*}
\end{proof}

We will now bound the amount of underestimation for the minimizer $\bar f^\pi_g$ in the above lemma.
\begin{lemma}
Suppose $x_0\sim d_0$ is not in $G$ almost surely. For any $\pi\in\Pi$,
\begin{align*}
    &Q^\pi(d_0, \pi) - \bar{f}_g^\pi(d_0, \bar\pi) \\
    &\leq \E_{\pi} \left[ \sum_{t=0}^{T-1} \gamma^{t}  \left( \gamma \max(g^\pi(x_{t+1}), f^\pi(x_{t+1}, \pi)) - f^\pi(x_t,a_t)\right)  + \gamma^T ( R(x_T) - g^\pi(x_T) )\right]
\end{align*}
Note that in a trajectory $x_T\in G$ whereas $x_t\notin G$ for $t<T$ by definition of $T$.
\end{lemma}

\begin{proof}
Let $\bar{f}_g^{\pi} = (f^\pi, g^\pi)$ be the empirical minimizer.
By performance difference lemma, we can write %\aditya{Should be $x_0$ throughout instead of $d_0$?}
\begin{align*}
    (1-\gamma) Q^\pi(d_0, \pi)  - (1-\gamma) \bar{f}_g^\pi(d_0, \bar{\pi}) = {} & (1-\gamma) \bar{Q}^\pi(d_0, \bar{\pi}) - (1-\gamma)\bar{f}_g^\pi(d_0, \bar{\pi}) \\
    = {} & \E_{\bar{d}^{\bar\pi}}[ \bar{R}(x,a) + \gamma \bar{f}_g^\pi(x',\bar{\pi}) - \bar{f}_g^\pi(x,a)]
\end{align*}
where with abuse of notation we define $\bar{d}^{\bar\pi}(x,a,x') \coloneq \bar{d}^{\bar\pi}(x,a) \bar{P}(x'|x,a)$, where $\bar{d}^{\bar\pi}(x,a)$ is the average state-action distribution of $\bar{\pi}$ in the augmented MDP.

In the above expectation, for $x\in G$, we have $a=a^+$ and $x^+ = (s^+, c)$ after taking $a^+$ at $x = (s,c)$, which leads to 
\begin{align*}
     \bar{R}(x,a) + \gamma \bar{f}_g^\pi(x',\bar{\pi}) - \bar{f}_g^\pi(x,a)
    = \bar{R}(x,a^+) + \gamma \bar{f}_g^\pi(x^+,\bar{\pi}) - \bar{f}_g^\pi(x,a^+)
    = R(x) - g^\pi(x)
\end{align*}
For $x\notin G$ and $x\notin\XX^+$, we have $a\neq a^+$ and $x'\notin \XX^+$; therefore
\begin{align*}
    \bar{R}(x,a) + \gamma \bar{f}_g^\pi(x',\bar{\pi}) - \bar{f}_g^\pi(x,a)
    &= R(x) + \gamma \bar{f}_g^\pi(x',\bar{\pi}) - f^\pi(x,a)\\
    &\leq \gamma \max(g^\pi(x'), f^\pi(x',\pi)) - f^\pi(x,a)
\end{align*}
where the last step is because of the definition of $\bar{f}_g^\pi$.
For $x\in\XX^+$, we have $x\in\XX^+$ and the reward is zero, so 
\begin{align*}
    \bar{R}(x,a) + \gamma \bar{f}_g^\pi(x',\bar{\pi}) - \bar{f}_g^\pi(x,a)
    = 0 
\end{align*}

Therefore, we can derive
\begin{align*}
    &(1-\gamma) Q^\pi(x_0, \pi) - (1-\gamma) \bar{f}_g^\pi(x_0, \bar\pi) \\
    &\leq \E_{\bar{d}^{\bar\pi}}[     \gamma \max(g^\pi(x'), f^\pi(x',\pi)) - f^\pi(x,a)  | x \notin G, x\notin \XX^+ ]
    + \E_{\bar{d}^{\bar\pi}}[ R(x) - g^\pi(x) | x\in G]
\end{align*}
Finally, using \cref{lm:equivalence} we can have the final upper bound.
\end{proof}

% Suppose $\epsilon_b$ is chosen correctly such that the following is true.
% \begin{align*}
%         &\ell_\dyn( \bar{f}_g, \bar{f}_g ;\bar\pi) - \min_{\bar{f}_{g'}'\in\bar\FF} \ell_\dyn(\bar{f}_{g'}',\bar{f}_g;\bar\pi) \leq \epsilon_b\\
%     &\ell_\goal (\bar{f}_g) \leq \epsilon_b
% \end{align*}
% would imply 
% \begin{align*}
%     &\E_{\mu_\dyn} ( f(x,a) - \gamma \E_{x'\sim \bar{P}(\cdot|x,a)} [   \max(g(x'), f(x', \pi))  ]   )^2 \leq \epsilon_b \\
%     &\E_{\mu_\goal}  ( g(x) -1   )^2 \leq \epsilon_b
% \end{align*}
% We denote this set on $\FF$ on $\GG$ as $\Omega_\stat$

\paragraph{Main Result: Performance Bound.} Let $\pi^\dagger$ be the learned policy and let $\bar{f}_g^{\pi^\dagger}$ be the learned function approximators. 
For any comparator policy $\pi$, let $\bar{f}_g^\pi = (f^\pi,g^\pi)$ be the estimator of $\pi$ on the data. We have.
\begin{align*}
&  V^\pi(d_0) - V^{\pi^\dagger}(d_0)\\
&=  Q^\pi(d_0, \pi) - Q^{\pi^\dagger}(d_0, \pi^\dagger) \\
&=  Q^\pi(d_0, \pi) - \bar{f}_g^{\pi^\dagger}(d_0, \bar{\pi}^\dagger) + \bar{f}_g^{\pi^\dagger}(d_0, \bar{\pi}^\dagger) - Q^{\pi^\dagger}(d_0, \pi^\dagger) \\
&\leq Q^\pi(d_0, \pi) - \bar{f}_g^{\pi^\dagger}(d_0, \bar{\pi}^\dagger)\\
&\leq Q^\pi(d_0, \pi) - \bar{f}_g^{\pi}(d_0, \bar{\pi})\\
&\leq \E_{\pi,P} \left[ \sum_{t=0}^{T-1} \gamma^t ( \gamma \max(g^\pi(x_{t+1}), f^\pi(x_{t+1}, \pi)) - f^\pi(x_t,a_t)) + \gamma^T (  R(x_T) - g^\pi(x_T)) \right]\\
&\leq \E_{\pi,P} \left[ \sum_{t=0}^{T-1} \gamma^t |\gamma \max(g^\pi(x_{t+1}), f^\pi(x_{t+1}, \pi)) - f^\pi(x_t,a_t)|  + \gamma^T | R(x_T) - g^\pi(x_T)| \right]\\
&\leq \mathfrak{C}_\dyn(\pi) \E_{\mu_\dyn}[ |\gamma \max(g^\pi(x'), f^\pi(x', \pi)) - f^\pi(x,a)|]
+  \mathfrak{C}_\goal(\pi) \E_{\mu_\goal}[ |g(x) - 1|]\\
&\lesssim  \mathfrak{C}_\dyn(\pi) \sqrt{\epsilon_\dyn} + 
    +  \mathfrak{C}_\goal(\pi) \sqrt{\epsilon_\goal}
\end{align*}
where $\mathfrak{C}_\dyn(\pi)$ and $\mathfrak{C}_\goal(\pi)$ are the concentrability coefficients defined in \cref{def:concentrability}.
\begin{theorem}\label{thm:main_theorem_app}
Let $\pi^\dagger$ denote the learned policy of \algo + PSPI with datasets $D_\dyn$ and $D_\goal$, using value function classes $\FF=\{\XX\times\AA\to[0,1]\}$ and $\GG=\{\XX\to[0,1]\}$. 
Under realizability and completeness assumptions as stated in \cref{as:realizability} and \cref{as:completeness} respectively, with probability $1-\delta$, it holds, for any $\pi\in\Pi$,
\begin{align*}
    J(\pi) - J(\pi^\dagger)
    &\lesssim  \mathfrak{C}_\dyn(\pi) \left( \sqrt{\frac{\square} {|D_\dyn|}} +  \sqrt{\frac{\square} {|D_\goal|}}\right) +  \mathfrak{C}_\goal(\pi) \sqrt{\frac{\log \frac{ \NN_\infty(\GG, 1/|D_\goal|)}{\delta} }{|D_\goal|}}
\end{align*}
where $\square \equiv \log\left(\frac{\NN_\infty\left(\FF, 1/|D_\goal||D_\dyn|\right) \NN_\infty\left(\GG, 1/|D_\goal||D_\dyn|\right) \NN_{\infty,1}\left(\Pi,1/|D_\goal||D_\dyn|\right)}{\delta}\right)$, and $\mathfrak{C}_\dyn(\pi)$ and $\mathfrak{C}_\goal(\pi) $ are concentrability coefficients which decrease as the data coverage increases.
\end{theorem}
% \chingan{Need to double check the notations.}

\subsubsection{Proof of Lemmas \ref{lem:sq_err_diff_dyn} and \ref{lem:sq_err_diff_goal}}
\label{sec:hp_proof}
We first show the following complementary lemma where we use a concentration bound on the constructed datasets $\bar D_{\dyn}$ and $\bar D_{\goal}$. Lemmas \ref{lm:Qpi in the set} and \ref{lm:generalization error)} will follow deterministically from this main auxiliary result.

\begin{lemma}
\label{lem:sq_err_diff_dyn}
    With probability at least $1-\delta$, for any $f,f_1,f_2 \in \FF$ and $g \in \GG$, we have:
    \begin{align*}
        & \ell_{\bar \mu_{\dyn}}(f_1, \bar{f}_g, \bar \pi) - \ell_{\bar \mu_{\dyn}}(f_2, \bar{f}_g, \bar \pi) - \ell_{\dyn}(f_1, \bar{f}_g, \bar \pi) + \ell_{\dyn}(f_2, \bar{f}_g, \bar \pi) & \\
        & \le \mathcal{O}\left(\|f_1-f_2\|_{\bar \mu_{\dyn}}\left(\sqrt{\frac{ \square}{|D_\goal|}} + \sqrt{\frac{ \square}{|D_\dyn|}}\right) + \frac{ \square}{\sqrt{|D_\goal||D_\dyn|}} + \frac{ \square}{|D_\goal|} + \frac{ \square}{|D_\dyn|} \right)
    \end{align*}
where $\square \equiv \log\left(\frac{\NN_\infty\left(\FF, 1/|D_\goal||D_\dyn|\right) \NN_\infty\left(\GG, 1/|D_\goal||D_\dyn|\right) \NN_{\infty,1}\left(\Pi,1/|D_\goal||D_\dyn|\right)}{\delta}\right)$.
\end{lemma}
\begin{proof}
Our proof is similar to proof of corresponding results in \citet{xie2021bellman} (Lemma A.4) and \citet{cheng2022adversarially} (Lemma 10) but we derive the result for the product distribution $\bar \mu_{\dyn} = \mu_{\dyn} \times \mu_{\goal}$ and its empirical approximation using $\bar D_{\dyn}$. Throughout this proof, we omit the bar on $\bar \pi$ as $\ell_{\dyn}$ does not use the extended definition of the policy $\pi$ and further use $M,N$ for the dataset sizes $|D_\goal|,|D_\dyn|$. For any observed context $(c_j, s_j) \in D_{\goal}$, we define the following quantity:
\begin{align*}
    \ell^j_{\mu_{\dyn}}(f, \bar{f'}_{g'}, \pi) = \E_{(s,a,s') \sim \mu_{\dyn}}\left[\left(f((s,c_j),a) - \gamma \max(g'((s',c_j))), f'((s',c_j),\pi))\right)^2\right]
\end{align*}
For conciseness, we use notation $x_{\circ j}$ for $(s,c_j)$ and $x'_{\circ j}$ for $(s',c_j)$ where $(s,a,s')$ is sampled from a dynamics distribution and $c_j \in D_\goal$. We first start with the following:
\begin{align}
    & \ell_{\bar \mu_{\dyn}}(f_1, \bar{f}_g, \pi) - \ell_{\bar \mu_{\dyn}}(f_2, \bar{f}_g, \pi) - \ell_{\dyn}(f_1, \bar{f}_g, \pi) + \ell_{\dyn}(f_2, \bar{f}_g, \pi) & \nonumber\\
    & \le \ell_{\bar \mu_{\dyn}}(f_1, \bar{f}_g, \pi) - \ell_{\bar \mu_{\dyn}}(f_2, \bar{f}_g, \pi) - \frac{1}{M}\sum_{j=1}^{M} \ell^j_{\mu_{\dyn}}(f_1, \bar{f}_g, \pi) + \frac{1}{M}\sum_{j=1}^{M} \ell^j_{\mu_{\dyn}}(f_2, \bar{f}_g, \pi) \label{line:conc_term1}\\
& \quad + \sum_{j=1}^{M} \ell^j_{\mu_{\dyn}}(f_1, \bar{f}_g, \pi) - \sum_{j=1}^{M} \ell^j_{\mu_{\dyn}}(f_2, \bar{f}_g, \pi) - \ell_{\dyn}(f_1, \bar{f}_g, \pi) + \ell_{\dyn}(f_2, \bar{f}_g, \pi) \label{line:conc_term2}
\end{align}
We will derive the final deviation bound by bounding each of these two empirical deviations in lines~\eqref{line:conc_term1},\eqref{line:conc_term2}. First, we will bound the term in line~\eqref{line:conc_term1}:
\begin{align}
   & \sum_{j=1}^{M} \ell^j_{\mu_{\dyn}}(f_1, \bar{f}_g, \pi) - \sum_{j=1}^{M} \ell^j_{\mu_{\dyn}}(f_2, \bar{f}_g, \pi) & \nonumber\\
   & = \sum_{j=1}^{M} \ell^j_{\mu_{\dyn}}(f_1, \bar{f}_g, \pi) - \ell^j_{\mu_{\dyn}}(f_2, \bar{f}_g, \pi) & \nonumber\\
   & = \sum_{j=1}^{M} \E_{\mu_{\dyn}}\big[ (f_1(x_{\circ j},a) - \gamma \max(g(x'_{\circ j}), f(x'_{\circ j},\pi)))^2 - (f_2(x_{\circ j},a) - \gamma \max(g(x'_{\circ j}), f(x'_{\circ j},\pi)))^2 \big] & \nonumber\\
   & = \sum_{j=1}^{M} \E_{\mu_\dyn} \left[ (f_1(x_{\circ j},a) - f_2(x_{\circ j},a))(f_1(x_{\circ j},a) + f_2(x_{\circ j},a) - 2\gamma \max(g(x'_{\circ j}), f(x'_{\circ j},\pi))) \right] & \nonumber\\
   & = \sum_{j=1}^{M} \E_{(s,a,\cdot) \sim \mu_\dyn} \left[ (f_1(x_{\circ j},a) - f_2(x_{\circ j},a))(f_1(x_{\circ j},a) + f_2(x_{\circ j},a) - 2\bar \TT^{\pi} \bar{f}_g)(x_{\circ j},a) \right] & \label{eq:cj_var_term}\\
   & = \sum_{j=1}^{M} \E_{(s,a,\cdot) \sim \mu_\dyn} \left[ (f_1(x_{\circ j},a) - \bar \TT^{\pi} \bar{f}_g(x_{\circ j},a))^2 - (f_2(x_{\circ j},a) - \bar \TT^{\pi} \bar{f}_g(x_{\circ j},a))^2 \right] & \nonumber
\end{align}
Using a similar argument, we can show that:
\begin{align}
    & \ell_{\bar \mu_{\dyn}}(f_1, \bar{f}_g, \pi) - \ell_{\bar \mu_{\dyn}}(f_2, \bar{f}_g, \pi) & \nonumber \\
    & = \E_{\bar \mu_\dyn} \left[ (f_1((s,c),a) - \bar \TT^{\pi} \bar{f}_g((s,c),a))^2 - (f_2((s,c),a) - \bar \TT^{\pi} \bar{f}_g((s,c),a))^2 \right] & \label{eq:cj_sqdiff_expectation}
\end{align}
Let $\FF_{\epsilon}$, $\GG_{\epsilon}$ be $\epsilon$-cover of $\FF$ and $\GG$, and $\Pi_{\epsilon}$ be $\epsilon$-cover of $\Pi$, i.e., $\exists \tilde f_1, \tilde f_2, \tilde f \in \FF_{\epsilon}$, $\tilde g \in GG_{\epsilon}$ and $\tilde \pi \in \Pi_{\epsilon}$ such that $\|f-\tilde f\|_{\infty}, \|f_1-\tilde f_1\|_{\infty}, \|f_2-\tilde f_2\|_{\infty} \le \epsilon$ and $\|\pi \tilde \pi\|_{\infty,1} \le \epsilon$. 

Then, for any $f,f_1,f_2 \in \FF$, $g \in \GG$, $\pi \in \Pi$ and their corresponding $\tilde f,\tilde f_1,\tilde f_2 \in \FF_{\epsilon}$, $\tilde g \in \GG_{\epsilon}$, $\tilde \pi \in \Pi_{\epsilon}$,:
\begin{align}
    & \ell_{\bar \mu_{\dyn}}(\tilde f_1, \bar{\tilde f}_{\tilde g}, \tilde\pi) - \ell_{\bar \mu_{\dyn}}(\tilde f_2, \bar{\tilde f}_{\tilde g}, \tilde\pi) - \frac{1}{M}\sum_{j=1}^{M} \left(\ell^j_{\mu_{\dyn}}(\tilde f_1, \bar{\tilde f}_{\tilde g}, \tilde\pi) - \ell^j_{\mu_{\dyn}}(\tilde f_2, \bar{\tilde f}_{\tilde g}, \tilde\pi)\right) & \nonumber \\
    & = \E_{\bar \mu_\dyn} \left[ (\tilde f_1((s,c),a) - \bar \TT^{\tilde\pi} \bar{\tilde f}_{\tilde g}((s,c),a))^2 - (\tilde f_2((s,c),a) - \bar \TT^{\tilde\pi} \bar{\tilde f}_{\tilde g}((s,c),a))^2 \right] & \nonumber \\
    & \quad - \frac{1}{M}\sum_{j=1}^{M} \E_{(s,a,\cdot) \sim \mu_\dyn} \left[ (\tilde f_1(x_{\circ j},a) - \tilde f_2(x_{\circ j},a))(\tilde f_1(x_{\circ j},a) + \tilde f_2(x_{\circ j},a) - 2\bar \TT^{\tilde\pi} \bar{\tilde f}_{\tilde g}) \right] & \nonumber \\
    & \le \sqrt{\frac{4\mathbf{V} \log\left(\frac{\NN_\infty\left(\FF, \epsilon\right) \NN_\infty\left(\GG, \epsilon\right) \NN_{\infty,1}\left(\Pi,\epsilon\right)}{\delta}\right)}{M}} + \frac{2\log\left(\frac{\NN_\infty\left(\FF, \epsilon\right) \NN_\infty\left(\GG, \epsilon\right) \NN_{\infty,1}\left(\Pi,\epsilon\right)}{\delta}\right)}{3M} \nonumber.&
\end{align}
where the first equation follows from eqs.~\eqref{eq:cj_var_term} and \eqref{eq:cj_sqdiff_expectation}, and the last inequality follows from Bernstein's inequality with a union bound over the classes $\FF_{\epsilon}, \GG_{\epsilon}, \Pi_{\epsilon}$ where $\mathbf{V}$ is the variance term as follows:
\begin{align*}
    & \text{Var}_{c \sim \mu_\goal}\left[\E_{(s,a,\cdot) \sim \mu_\dyn} \left[ (f_1((s,c),a) - f_2((s,c),a))(f_1((s,c),a) + f_2((s,c),a) - 2\bar \TT^{\pi} \bar{f}_g((s,c),a)) \right]\right] & \\
    & \le \E_{c \sim \mu_\goal}\left[\E_{(s,a,\cdot) \sim \mu_\dyn} \left[ (f_1((s,c),a) - f_2((s,c),a))(f_1((s,c),a) + f_2((s,c),a) - 2\bar \TT^{\pi} \bar{f}_g((s,c),a)) \right]^2\right] & \\
    & \le 4 \E_{\bar \mu_\dyn} \left[(f_1((s,c),a) - f_2((s,c),a))^2\right]& 
\end{align*}
where we used that fact that $f, g \in [0,1]$. 

Thus, with probability $1-\delta$, 
\begin{align*}
    & \ell_{\bar \mu_{\dyn}}(\tilde f_1, \bar{\tilde f}_{\tilde g}, \tilde\pi) - \ell_{\bar \mu_{\dyn}}(\tilde f_2, \bar{\tilde f}_{\tilde g}, \tilde\pi) - \frac{1}{M}\sum_{j=1}^{M} \left(\ell^j_{\mu_{\dyn}}(\tilde f_1, \bar{\tilde f}_{\tilde g}, \tilde\pi) - \ell^j_{\mu_{\dyn}}(\tilde f_2, \bar{\tilde f}_{\tilde g}, \tilde\pi)\right) & \nonumber \\
    & \le 2\|\tilde f_1-\tilde f_2\|_{\bar \mu_{\dyn}}\sqrt{\frac{ \log\left(\frac{\NN_\infty\left(\FF, \epsilon\right) \NN_\infty\left(\GG, \epsilon\right) \NN_{\infty,1}\left(\Pi,\epsilon\right)}{\delta}\right)}{M}} + \frac{2\log\left(\frac{\NN_\infty\left(\FF, \epsilon\right) \NN_\infty\left(\GG, \epsilon\right) \NN_{\infty,1}\left(\Pi,\epsilon\right)}{\delta}\right)}{3M} .&
\end{align*}
Using the property of the set covers of $\FF,\GG, \Pi$, we can easily conclude that:
\begin{align}
    & \ell_{\bar \mu_{\dyn}}(f_1, \bar{f}_{g}, \pi) - \ell_{\bar \mu_{\dyn}}(f_2, \bar{f}_{ g}, \pi) - \frac{1}{M}\sum_{j=1}^{M} \left(\ell^j_{\mu_{\dyn}}( f_1, \bar{ f}_{ g}, \pi) - \ell^j_{\mu_{\dyn}}( f_2, \bar{ f}_{ g}, \pi)\right) & \nonumber \\
    & \lesssim \|f_1-f_2\|_{\bar \mu_{\dyn}}\sqrt{\frac{ \log\left(\frac{\NN_\infty\left(\FF, \epsilon\right) \NN_\infty\left(\GG, \epsilon\right) \NN_{\infty,1}\left(\Pi,\epsilon\right)}{\delta}\right)}{M}} + \frac{\log\left(\frac{\NN_\infty\left(\FF, \epsilon\right) \NN_\infty\left(\GG, \epsilon\right) \NN_{\infty,1}\left(\Pi,\epsilon\right)}{\delta}\right)}{M} & \nonumber \\
    & \quad + \epsilon \sqrt{\frac{ \log\left(\frac{\NN_\infty\left(\FF, \epsilon\right) \NN_\infty\left(\GG, \epsilon\right) \NN_{\infty,1}\left(\Pi,\epsilon\right)}{\delta}\right)}{M}} + \epsilon. \label{eq:conc_bound1}& 
\end{align}
Now, we bound the second deviation term in eq. line~\eqref{line:conc_term2}:
\begin{align}
    & \frac{1}{M}\sum_{j=1}^{M} \left(\ell^j_{\mu_{\dyn}}(f_1, \bar{f}_g, \pi) -  \ell^j_{\mu_{\dyn}}(f_2, \bar{f}_g, \pi)\right) - \left(\ell_{\dyn}(f_1, \bar{f}_g, \pi) - \ell_{\dyn}(f_2, \bar{f}_g, \pi)\right) & \nonumber\\
    & = \frac{1}{M}\sum_{j=1}^{M}\Bigg[\E_{ \mu_{\dyn}}\left[\left(f_1(x_{\circ j},a) - \gamma \max(g(x'_{\circ j}), f(x'_{\circ j},\pi))\right)^2 - \left(f_2(x_{\circ j},a) - \gamma \max(g(x'_{\circ j}), f(x'_{\circ j},\pi))\right)^2\right] & \nonumber \\
    & \quad  - \frac{1}{N}\sum_{i=1}^{N}\left[\left(f_1(x_{ij},a) - \gamma \max(g(x'_{ij}), f(x'_{ij},\pi))\right)^2 - \left(f_2(x_{ij},a) - \gamma \max(g(x'_{ij}), f(x'_{ij},\pi))\right)^2 \right]\Bigg]
\end{align}
For any fixed $c_j$, using the same strategy as we used for bounding the first term in eq. line~\eqref{line:conc_term1}, for any $f,f_1,f_2 \in \FF$, $g \in \GG$, $\pi \in \Pi$ and their corresponding $\tilde f,\tilde f_1,\tilde f_2 \in \FF_{\epsilon}$, $\tilde g \in \GG_{\epsilon}$, $\tilde \pi \in \Pi_{\epsilon}$, with probability at least $1-\delta$:
\begin{align}
    & \left(\ell^j_{\mu_{\dyn}}(\tilde f_1, \bar{\tilde f}_{\tilde g}, \tilde \pi) -  \ell^j_{\mu_{\dyn}}(\tilde f_2, \bar{\tilde f}_{\tilde g}, \tilde \pi)\right) - \left(\ell_{\dyn}(\tilde f_1, \bar{\tilde f}_{\tilde g}, \tilde \pi) - \ell_{\dyn}(\tilde f_2, \bar{\tilde f}_{\tilde g}, \tilde \pi)\right) & \nonumber \\
    & \lesssim \|\tilde f_1-\tilde f_2\|_{\mu_{\dyn} \times \{c_j\}}\sqrt{\frac{ \log\left(\frac{\NN_\infty\left(\FF, \epsilon\right) \NN_\infty\left(\GG, \epsilon\right) \NN_{\infty,1}\left(\Pi,\epsilon\right)}{\delta}\right)}{N}} + \frac{\log\left(\frac{\NN_\infty\left(\FF, \epsilon\right) \NN_\infty\left(\GG, \epsilon\right) \NN_{\infty,1}\left(\Pi,\epsilon\right)}{\delta}\right)}{N} & \nonumber \\
    & \quad + \epsilon \sqrt{\frac{ \log\left(\frac{\NN_\infty\left(\FF, \epsilon\right) \NN_\infty\left(\GG, \epsilon\right) \NN_{\infty,1}\left(\Pi,\epsilon\right)}{\delta}\right)}{N}} + \epsilon. \nonumber&
\end{align}
We can now consider the sum in the second term in eq. line \eqref{line:conc_term2} for $\tilde f, \tilde f_1, \tilde f_2, \tilde g, \tilde \pi$ as:
\begin{align}
    & \frac{1}{M}\sum_{j=1}^{M} \left(\ell^j_{\mu_{\dyn}}(\tilde f_1, \bar{\tilde f}_{\tilde g}, \tilde \pi) -  \ell^j_{\mu_{\dyn}}(\tilde f_2, \bar{\tilde f}_{\tilde g}, \tilde \pi)\right) - \left(\ell_{\dyn}(\tilde f_1, \bar{\tilde f}_{\tilde g}, \tilde \pi) - \ell_{\dyn}(\tilde f_2, \bar{\tilde f}_{\tilde g}, \tilde \pi)\right) & \nonumber\\
    & \lesssim \frac{1}{M}\sum_{j=1}^{M}\|\tilde f_1-\tilde f_2\|_{\mu_{\dyn} \times \{c_j\}}\sqrt{\frac{ \log\left(\frac{\NN_\infty\left(\FF, \epsilon\right) \NN_\infty\left(\GG, \epsilon\right) \NN_{\infty,1}\left(\Pi,\epsilon\right)}{\delta}\right)}{N}}  & \nonumber \\
    &\quad  + \frac{\log\left(\frac{\NN_\infty\left(\FF, \epsilon\right) \NN_\infty\left(\GG, \epsilon\right) \NN_{\infty,1}\left(\Pi,\epsilon\right)}{\delta}\right)}{N} + \epsilon \sqrt{\frac{ \log\left(\frac{\NN_\infty\left(\FF, \epsilon\right) \NN_\infty\left(\GG, \epsilon\right) \NN_{\infty,1}\left(\Pi,\epsilon\right)}{\delta}\right)}{N}} + \epsilon. & \nonumber \\
    & \lesssim \left(\frac{1}{M}\sum_{j=1}^{M}\|\tilde f_1-\tilde f_2\|_{\mu_{\dyn} \times \{c_j\}}-\|\tilde f_1-\tilde f_2\|_{\bar \mu_{\dyn}}\right)\sqrt{\frac{ \log\left(\frac{\NN_\infty\left(\FF, \epsilon\right) \NN_\infty\left(\GG, \epsilon\right) \NN_{\infty,1}\left(\Pi,\epsilon\right)}{\delta}\right)}{N}}  & \nonumber \\
    & \quad + \|\tilde f_1-\tilde f_2\|_{\bar \mu_{\dyn}}\sqrt{\frac{ \log\left(\frac{\NN_\infty\left(\FF, \epsilon\right) \NN_\infty\left(\GG, \epsilon\right) \NN_{\infty,1}\left(\Pi,\epsilon\right)}{\delta}\right)}{N}} & \nonumber \\ 
    & \quad + \frac{\log\left(\frac{\NN_\infty\left(\FF, \epsilon\right) \NN_\infty\left(\GG, \epsilon\right) \NN_{\infty,1}\left(\Pi,\epsilon\right)}{\delta}\right)}{N} + \epsilon \sqrt{\frac{ \log\left(\frac{\NN_\infty\left(\FF, \epsilon\right) \NN_\infty\left(\GG, \epsilon\right) \NN_{\infty,1}\left(\Pi,\epsilon\right)}{\delta}\right)}{N}} + \epsilon. & \nonumber \\
    & \lesssim \|\tilde f_1-\tilde f_2\|_{\bar \mu_{\dyn}}\sqrt{\frac{ \log\left(\frac{\NN_\infty\left(\FF, \epsilon\right) \NN_\infty\left(\GG, \epsilon\right) \NN_{\infty,1}\left(\Pi,\epsilon\right)}{\delta}\right)}{N}} + \frac{ \log\left(\frac{\NN_\infty\left(\FF, \epsilon\right) \NN_\infty\left(\GG, \epsilon\right) \NN_{\infty,1}\left(\Pi,\epsilon\right)}{\delta}\right)}{\sqrt{NM}} & \nonumber \\ 
    & \quad + \frac{\log\left(\frac{\NN_\infty\left(\FF, \epsilon\right) \NN_\infty\left(\GG, \epsilon\right) \NN_{\infty,1}\left(\Pi,\epsilon\right)}{\delta}\right)}{N} + \epsilon \sqrt{\frac{ \log\left(\frac{\NN_\infty\left(\FF, \epsilon\right) \NN_\infty\left(\GG, \epsilon\right) \NN_{\infty,1}\left(\Pi,\epsilon\right)}{\delta}\right)}{N}} + \epsilon. & \nonumber \\
\end{align}
where the last inequality follows from Hoeffding's inequality. We can now bound the term in eq. line~\eqref{line:conc_term2} as:
\begin{align}
    & \frac{1}{M}\sum_{j=1}^{M} \left(\ell^j_{\mu_{\dyn}}(f_1, \bar{f}_g, \pi) -  \ell^j_{\mu_{\dyn}}(f_2, \bar{f}_g, \pi)\right) - \left(\ell_{\dyn}(f_1, \bar{f}_g, \pi) - \ell_{\dyn}(f_2, \bar{f}_g, \pi)\right) & \nonumber\\
    & \lesssim \|\tilde f_1-\tilde f_2\|_{\bar \mu_{\dyn}}\sqrt{\frac{ \log\left(\frac{\NN_\infty\left(\FF, \epsilon\right) \NN_\infty\left(\GG, \epsilon\right) \NN_{\infty,1}\left(\Pi,\epsilon\right)}{\delta}\right)}{N}} + \frac{ \log\left(\frac{\NN_\infty\left(\FF, \epsilon\right) \NN_\infty\left(\GG, \epsilon\right) \NN_{\infty,1}\left(\Pi,\epsilon\right)}{\delta}\right)}{\sqrt{NM}} & \nonumber \\ 
    & \quad  + \frac{\log\left(\frac{\NN_\infty\left(\FF, \epsilon\right) \NN_\infty\left(\GG, \epsilon\right) \NN_{\infty,1}\left(\Pi,\epsilon\right)}{\delta}\right)}{N} + \epsilon \sqrt{\frac{ \log\left(\frac{\NN_\infty\left(\FF, \epsilon\right) \NN_\infty\left(\GG, \epsilon\right) \NN_{\infty,1}\left(\Pi,\epsilon\right)}{\delta}\right)}{N}} + \epsilon. \label{eq:conc_bound2}
\end{align}
Combining eqs.~\eqref{eq:conc_bound1} and \eqref{eq:conc_bound2} with $\epsilon = \mathcal{O}(\frac{1}{M N})$, we get the final result.
\end{proof}

\begin{lemma}
\label{lem:sq_err_diff_goal}
With probability at least $1-\delta$, for any $g,g_1,g_2 \in \GG$ and $f \in \FF$, we have:
\begin{align*}
    & \ell_{\bar \mu_{\goal}}(\bar f_{g_1}) - \ell_{\bar \mu_{\goal}}(\bar f_{g_2}) - \ell_{\goal}(\bar f_{g_1}) + \ell_{\goal}(\bar f_{g_2}) & \\
    & \le \mathcal{O}\left( \|g_1-g_2\|_{\bar \mu_{\goal}} \sqrt{\frac{\log\left(\frac{\NN_\infty(\FF, 1/|D_\goal|)}{\delta} \right)}{|D_\goal|}} + \frac{\log\left(\frac{\NN_\infty (\FF, 1/|D_\goal|)}{\delta}\right)}{|D_\goal|} \right).
\end{align*}
\end{lemma}
\begin{proof}
This result can be proven using the same arguments as used in Lemma~\ref{lem:sq_err_diff_dyn} using a covering argument just over $\GG$.
\end{proof}

Using these two main concentration results, we can now prove Lemmas~\ref{lm:Qpi in the set} and \ref{lm:generalization error)}.

\begin{proof}[\textbf{Proof of Lemma~\ref{lm:Qpi in the set}}]
Note $\bar{Q}^{\bar\pi} = \bar{f}_g$ for some $f\in\FF$ and $g\in\GG$ (\cref{lm:realizability}) and 
\begin{align*}
 0 = {} & \E_{x,a\sim\mu_\dyn} [(\bar{Q}^{\bar\pi}(x,a) - \bar\TT^{\bar\pi} \bar{Q}^{\bar\pi}(x,a))^2] \\
 = {} & \E_{x,a\sim\mu_\dyn} [(\bar{Q}^{\bar\pi}(x,a) - 0 - \gamma \E_{x'\sim \bar{P}(\cdot|x,a)} [ \phi(\bar{Q}^{\bar\pi})(x',\pi) ] )^2]
\end{align*}
The lemma can now be proved by following a similar proof of Theorem A.1 of \citet{xie2021bellman}. The key difference is the use of our concentration bounds in Lemmas \ref{lem:sq_err_diff_dyn} and \ref{lem:sq_err_diff_goal} instead of Lemma A.4 in the proof of \citet{xie2021bellman}. On the other hand, $\ell_\goal (\bar{f}_g) = 0$ because the reward $R(x)$ is deterministic which results in the second inequality.
\end{proof}

\begin{proof}[\textbf{Proof of Lemma~\ref{lm:generalization error)}}]
This result can again be proved using the same steps as in Lemma A.5 from \citet{xie2021bellman} based on the concentration bound in Lemmas \ref{lem:sq_err_diff_dyn} and \ref{lem:sq_err_diff_goal}.
\end{proof}
\section{Experimental Details}

\subsection{Hyperparameters and Experimental Settings}
\label{app:hyper}

\paragraph{IQL.} For IQL, we keep the hyperparameter of $\gamma = 0.99$, $\tau=0.9$, $\beta=10.0$, and $\alpha = 0.005$ in~\cite{kostrikov2021offline}, and tune other hyperparameters on the antmaze-medium-play-v2 environment and choose batch size = 1024 from candidate choices \{256, 512, 1024, 2046\}, learning rate = $10^{-4}$ from candidate choices \{$5\cdot 10^{-5}, 10^{-4}, 3\cdot 10^{-4}$\} and 3 layer MLP with RuLU activating and 256 hidden units for all networks. We use the same set of IQL hyperparameters for both our methods and all the baseline methods included in Section~\ref{exp:baseline}, and apply it to all environments. 
In the experiments, we follow the convention of the $-1/0$ reward in the IQL implementation for Antmaze, which can be shown to be the same as the $0/1$ reward notion in terms of ranking policies under the discounted MDP setting.

% \paragraph{Reward prediction (RP).} For naive reward prediction, we first convert the context-goal set to a dataset with reward $0$ for all $(c,s) \sim D_\goal$, and then learn a reward function with the dataset. For policy training, we randomly sample $(s,a,s')\sim D_\dyn$ and $c \sim D_\goal$ and label the transition with the learned reward: if reward prediction of $(c,s')$ is larger then some threshold, we label the transition with $r = 0$ and $\texttt{terminal} = \text{True}$; otherwise we label the transition with $r = -1$ and $\texttt{terminal} = \text{False}$.

\paragraph{Reward Prediction (RP).}
For naive reward prediction, we use the full context-goal dataset as positive data, and train a reward model with 3-layer MLP and ReLU activations, learning rate = $10^{-4}$, batch size = 1024, and training for 100 epochs for convergence. 
% Although we use $-1/0$ reward for empirical study, we first train the 
To label the transition dataset, we need to find some appropriate threshold to label states predicted as goals given contexts. We choose the percentile as 5\% in the reward distribution evaluated by the context-goal set as the threshold to label goals (if a reward is larger than the threshold than it is labeled as terminal), from candidate choices \{0\%, 5\%, 10\%\}. Then we apply it to all environments. Another trick we apply for the reward prediction is that instead of predicting 0 for the context-goal dataset, we let it predict 1 but shift the reward prediction by -1 during reward evaluation, which prevents the model from learning all 0 weights. Similar tricks are also used in other reward learning baselines.
% \ying{add empirical reward shift in appendix}

\paragraph{UDS+RP.} We use the same structure and training procedure for the reward model as RP, except that we also randomly sample a minibatch of ``negative" contextual transitions with the same batch size for a balanced distribution, which is constructed by randomly sampling combinations of a state in the trajectory-only dataset and a context from the context-goal dataset. To create a balanced distribution of positive and negative samples, we sample from each dataset with equal probability. For the threshold, we choose the percentile as 5\% in the reward distribution evaluated by the context-goal set as the threshold to label goals in the antmaze-medium-play-v2 environment, from candidate choices \{0\%, 5\%, 10\%\}. Then we apply it to all environments.

\paragraph{PDS.} We use the same structure and training procedure for the reward model as RP, except that we train an ensemble of $10$ networks as in~\cite{hu2023provable}. To select the threshold percentile and the pessimistic weight $k$, we choose the percentile as 15\% in the reward distribution evaluated by the context-goal set as the threshold to label goals from candidate choices \{0\%, 5\%, 10\%, 15\%, 20\%\}, and $k=15$ from the candidate choices \{5,10,15,20\} in the antmaze-medium-play-v2 environment. Then we apply them to all environments.
\paragraph{\algo (ours).} We do not require extra parameters other than the possibility of sampling from the real and fake transitions. 
Intuitively, we should sample from both datasets with the same probability to create an overall balanced distribution. We ran additional experiments to study the effect of this sampling ratio hyperparameter: ratio of samples from the context-goal dataset $D_\goal$ to total samples in each minibatch. Table~\ref{tab:sample_ratio} shows that \algo well as long as the ratio is roughly balanced in sampling from both dataset.
\paragraph{Compute Resources.} For all methods, each training run takes about 8h on a NVIDIA T4 GPU. 

\begin{table*}[htpb]
\centering
\caption{Average success rate (\%) in AntMaze-v2 from all environments, with different sampling ratios from the context-goal dataset.}
\scalebox{0.9}{
\begin{tabular}{l|ccccc}
\toprule
  Env/Ratio             & 0.1      & 0.3 & 0.5       & 0.7 & 0.9  \\
   \hline
umaze         & 91.6±1.3&92.4±1.0&94.8±1.3&86.4±1.8&84.8±3.0 \\
umaze diverse  &76.8±1.9 &79.2±1.6&72.8±7.7&76.6±2.3&65.4±8.8\\
medium play    &82.3±2.1&85.0±1.8&75.8±1.9&72.8±1.3&76.6±1.3 \\
medium diverse & 79.4±1.6&76.6±3.0&84.5±5.2&75.6±2.0&72.0±3.5\\
large play  &50.8±2.0&45.2±3.7&60.0±7.6&43.6±2.3&46.6±2.3\\
large diverse &35.8±5.7&37.4±4.7&36.8±6.9&34.4±2.4&27.0±2.1 \\
\hline
average &69.5&68.9 &70.8&64.9& 62.1\\
 \bottomrule
\end{tabular}
}
\label{tab:sample_ratio}
% \vspace{-5mm}
\end{table*}

\subsection{Context-Goal dataset Construction and Environmental Evaluation.}
\label{app:context_setup}

Here we introduce the context-goal dataset in the three levels of context-goal setup mentioned in Section~\ref{sec:experiments} and how to evaluate in each setup. We also include our code implementation for reference.

\paragraph{Original Antmaze.} We extract the 2D locations from the states in the trajectory dataset with terminal=True as the context (in original antmaze, it suffices to reach the $L_2$ ball with radius 0.5 around the center), where the contexts are distributed very closely as visualized in Figure~\ref{fig:cgo}(a), and the corresponding states serve as the goal examples with Gaussian perturbations $N(0,0.05)$ on the dimensions other than the 2D location.

\paragraph{Four Rooms.} For each maze map, we partition 4 rooms like Figure~\ref{fig:cgo}(b) and use the room number as the context. To construct goal examples, we create a copy of all states in the trajectory dataset, perturb the states in the copy by $N(0,0.05)$ on each dimension, and then randomly select the states (up to 20K) according to the room partition.

% and randomly select up to 10K states in the unsupervised trajectori

\paragraph{Random Cells.} For each maze map, we construct a range of non-wall 2D locations in the maze map and uniformly sample from it to get the training contexts. To construct the goal set given context, we randomly sample up to 20K states with the 2D locations within the $L_2$ ball with radius $2$. Figure~\ref{fig:cgo}(C) is a intuitive visualization of the corresponding context-goal sets. For test distributions, we have two settings: 1) the same as the training distribution; 2) test contexts are drawn from a limited area that is far away from the starting point of the agent.

\paragraph{Evaluation.} We follow the conventional evaluation procedure in~\cite{kostrikov2021offline}, where the success rate is normalized to be 0-100 and evaluated with 100 trajectories. We report the result with standard error across 5 random seeds. The oracle condition we define in each context-goal setup is used to evaluate whether the agent has successfully reached the goal and also defines the termination of an episode.

\subsection{Reward Model Evaluation}
\label{app:reward_eval}
For reward learning baselines, we evaluate the learned reward model to showcase whether the learned reward function can successfully capture context-goal relationships. 

\paragraph{Evaluation dataset construction.}We construct the positive dataset from context-goal examples, and the negative dataset from the combination of the context set and all states in the trajectory-only data, then use the oracle context-goal definition in each setup to filter out the positive ones. We then evaluate the predicted reward on both positive and negative datasets, generating boxplots to visualize the distributions of the predicted reward for both datasets. 

% \subsection{Detailed context-goal setups}
% \label{app:context_goal}
% \ying{move whatever is too long in the experimental setup here}
% \paragraph{Original AntMaze.} 
% \paragraph{Modified AntMaze: Four Rooms.}
% \paragraph{Modified AntMaze: Random Cells.}

\paragraph{Results.} Here we present boxplots for reward models with experimental setups in Section~\ref{sec:results}.
% \ref{exp:original_ant},~\ref{exp:four_rooms} and~\ref{exp:cell}. 
Overall we observe that PDS+RP is consistently better at separating positive and negative distributions than UDS and naive reward prediction. However, PDS can still fail at fully separating positive and negative examples.

% \section{More reward evaluations}
% \label{}
% \ying{put whatever we don't have space for in the main paper.}
% \ying{also put other reward evals}
% \section{More reward model evaluations}

\begin{figure}[h]
\begin{subfigure}[b]{0.29\textwidth}
     \centering
     \includegraphics[width=\textwidth]{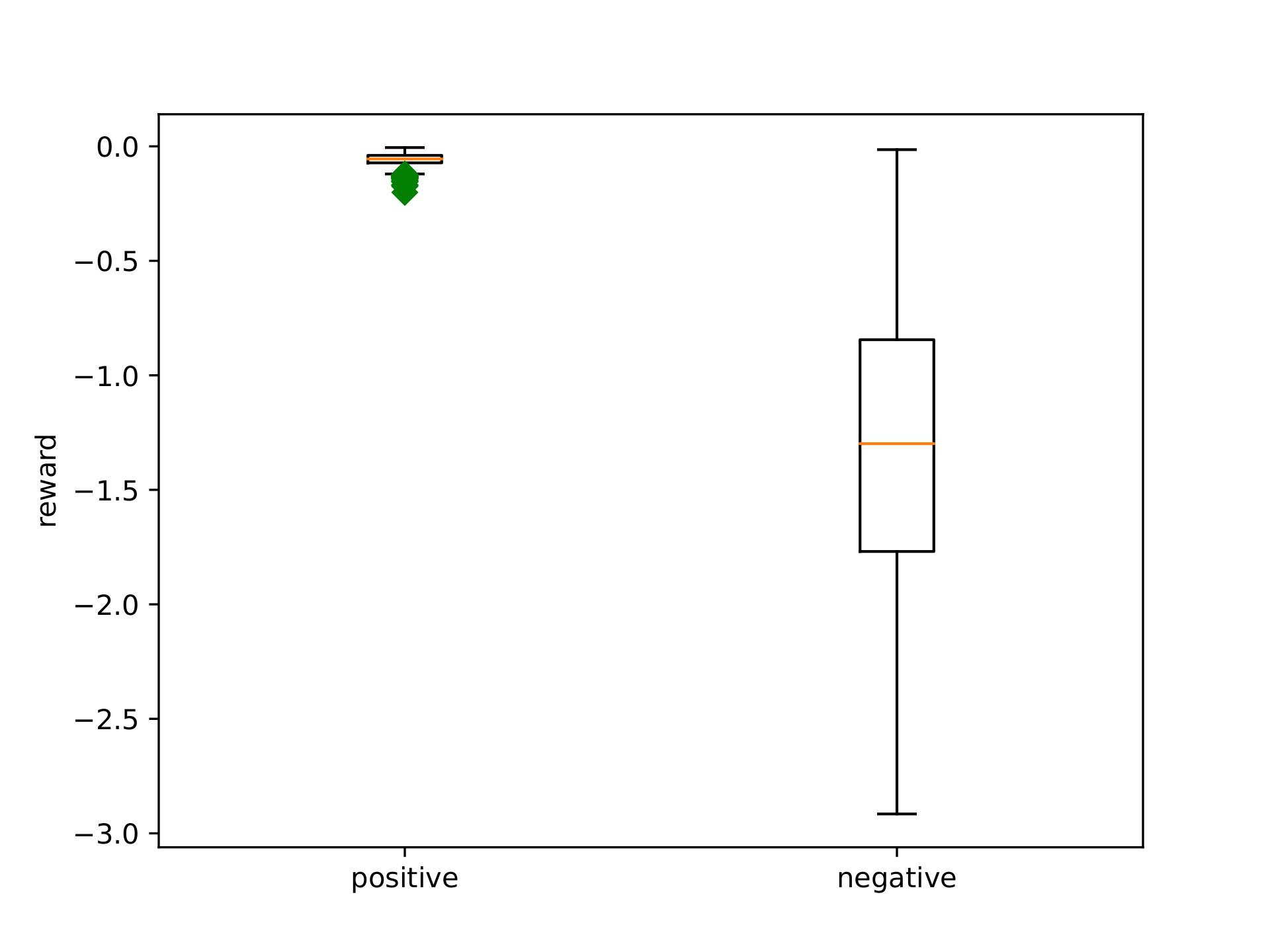}
     \caption{PDS}
 \end{subfigure}
 \hfill
 \begin{subfigure}[b]{0.29\textwidth}
     \centering
     \includegraphics[width=\textwidth]{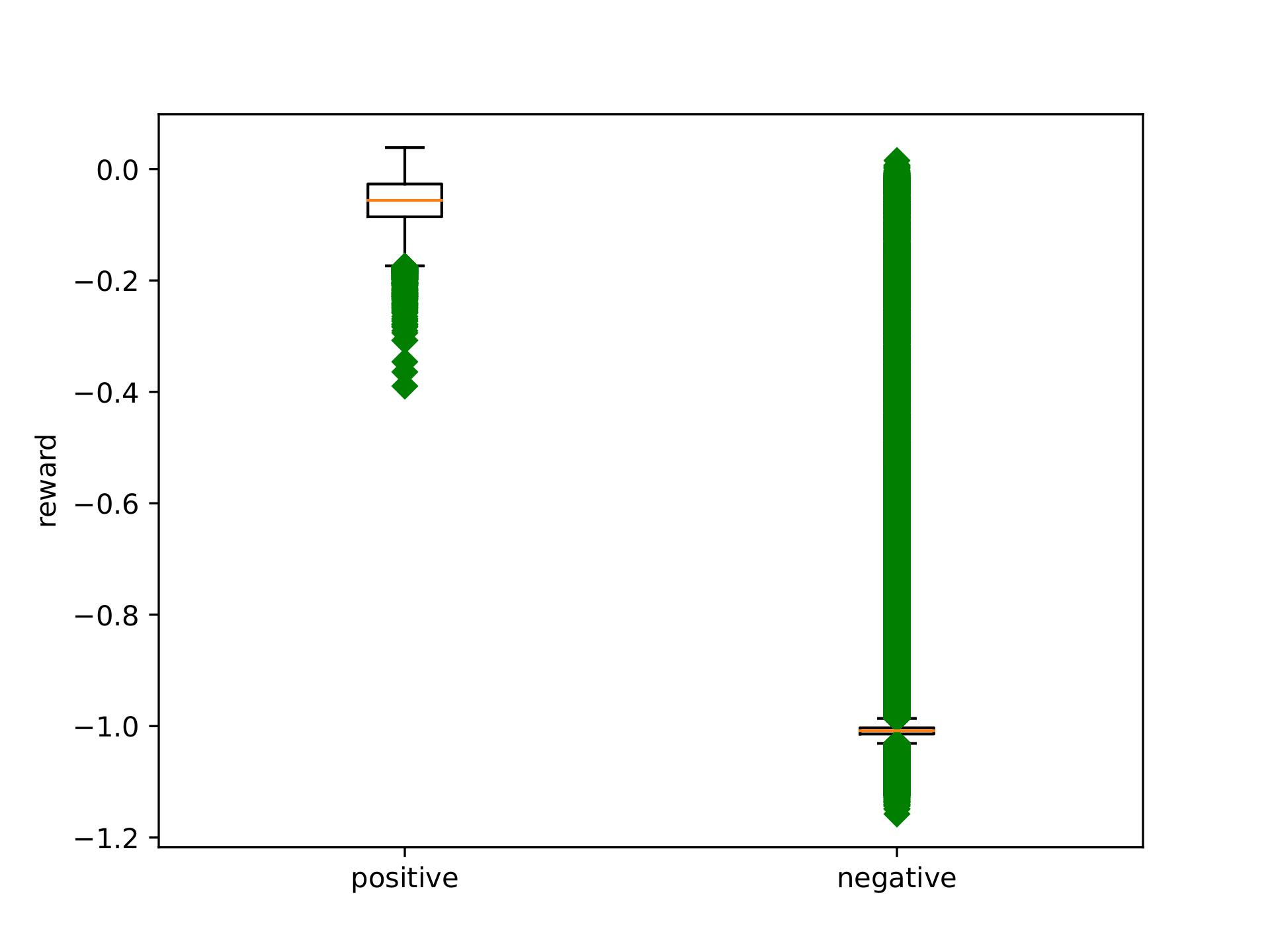}
     \caption{UDS+RP}
 \end{subfigure}
 \hfill
  \begin{subfigure}[b]{0.29\textwidth}
     \centering
     \includegraphics[width=\textwidth]{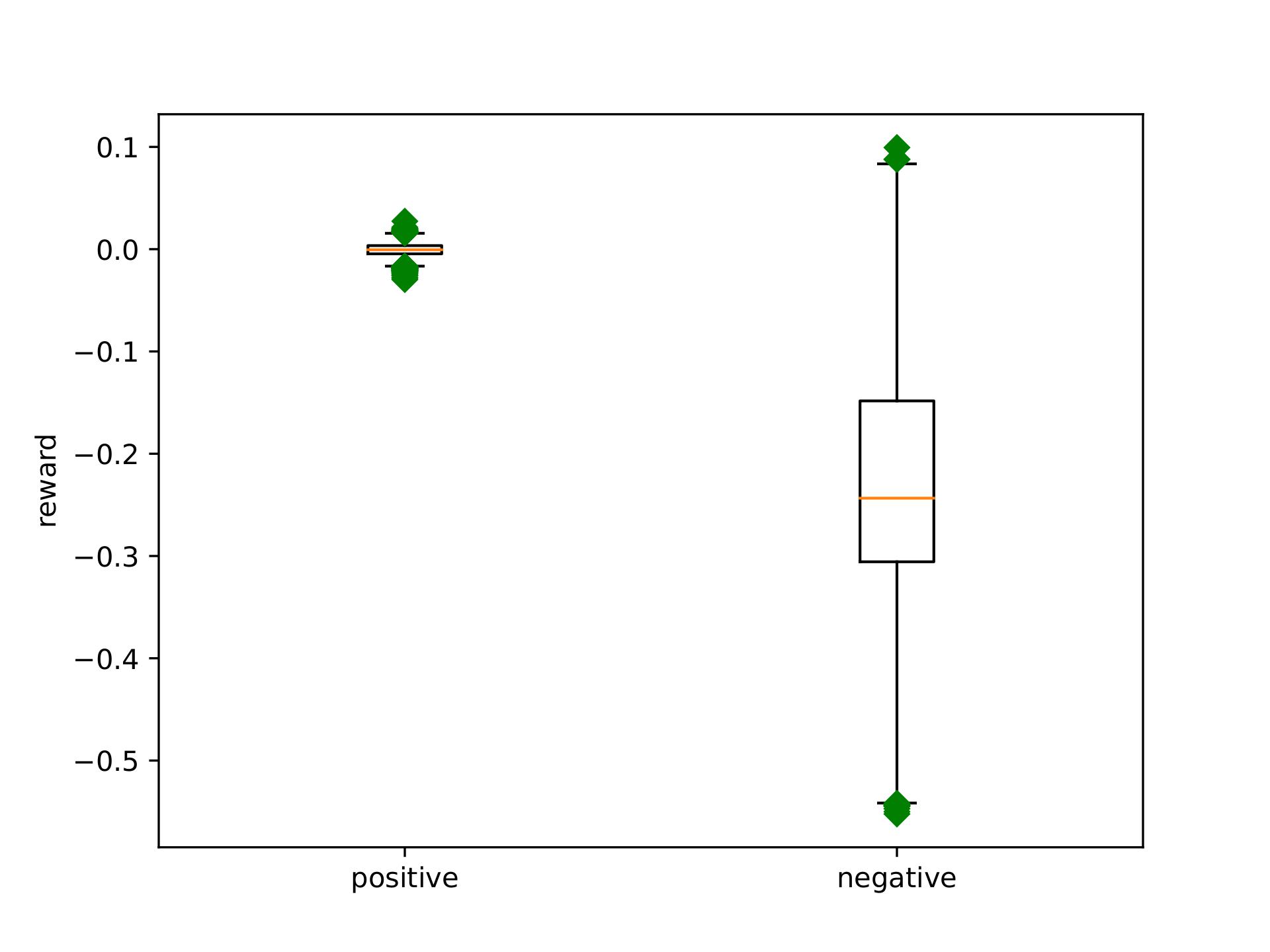}
     \caption{RP}
 \end{subfigure}
 \caption{Reward model evaluation for the large-diverse dataset for original AntMaze environment. Green dots are outliers.}
 \label{fig:large}
 \vspace{-2mm}
\end{figure}

\begin{figure}[h]
\begin{subfigure}[b]{0.32\textwidth}
     \centering
     \includegraphics[width=\textwidth]{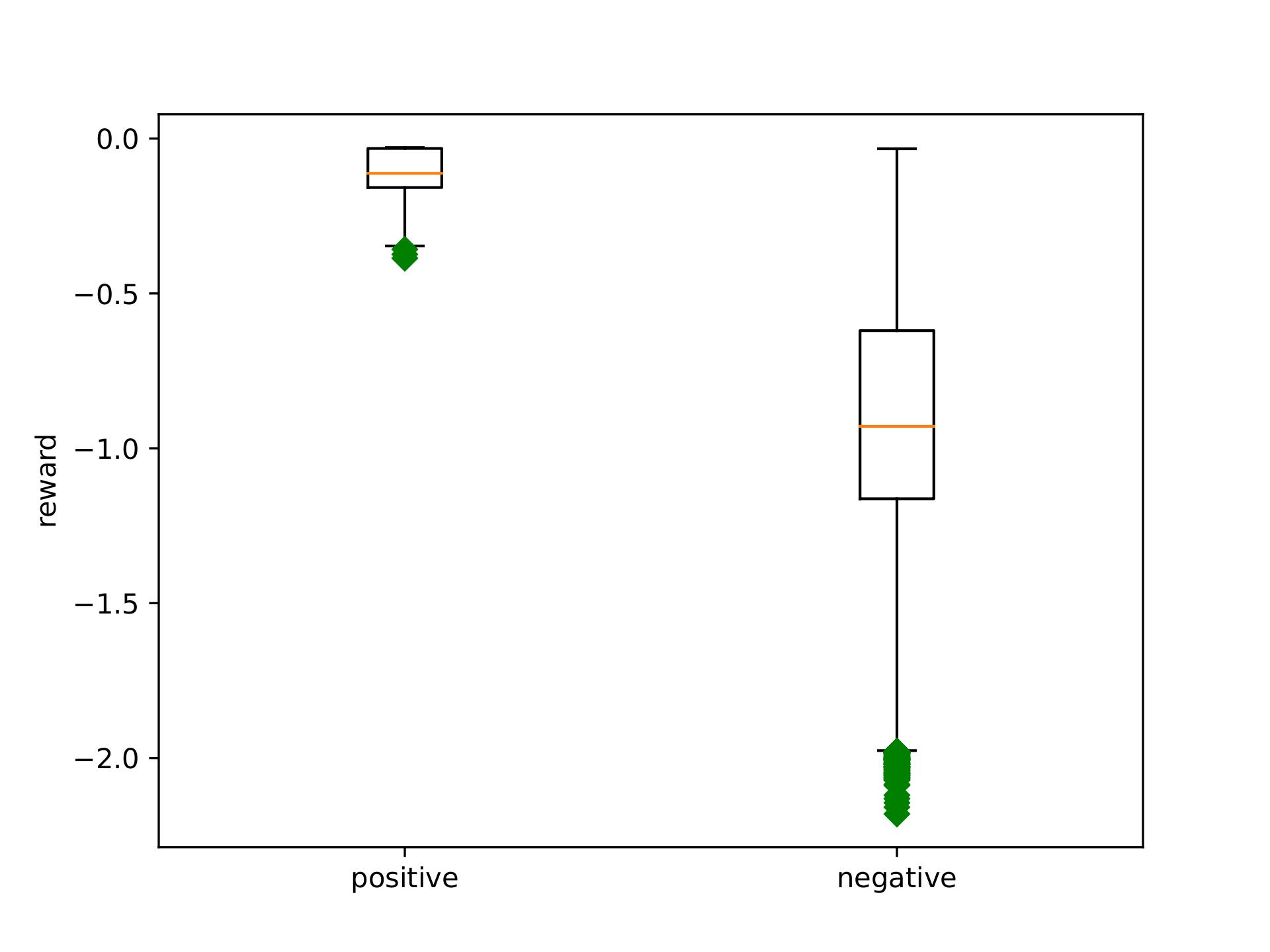}
     \caption{PDS}
 \end{subfigure}
 \hfill
 \begin{subfigure}[b]{0.32\textwidth}
     \centering
     \includegraphics[width=\textwidth]{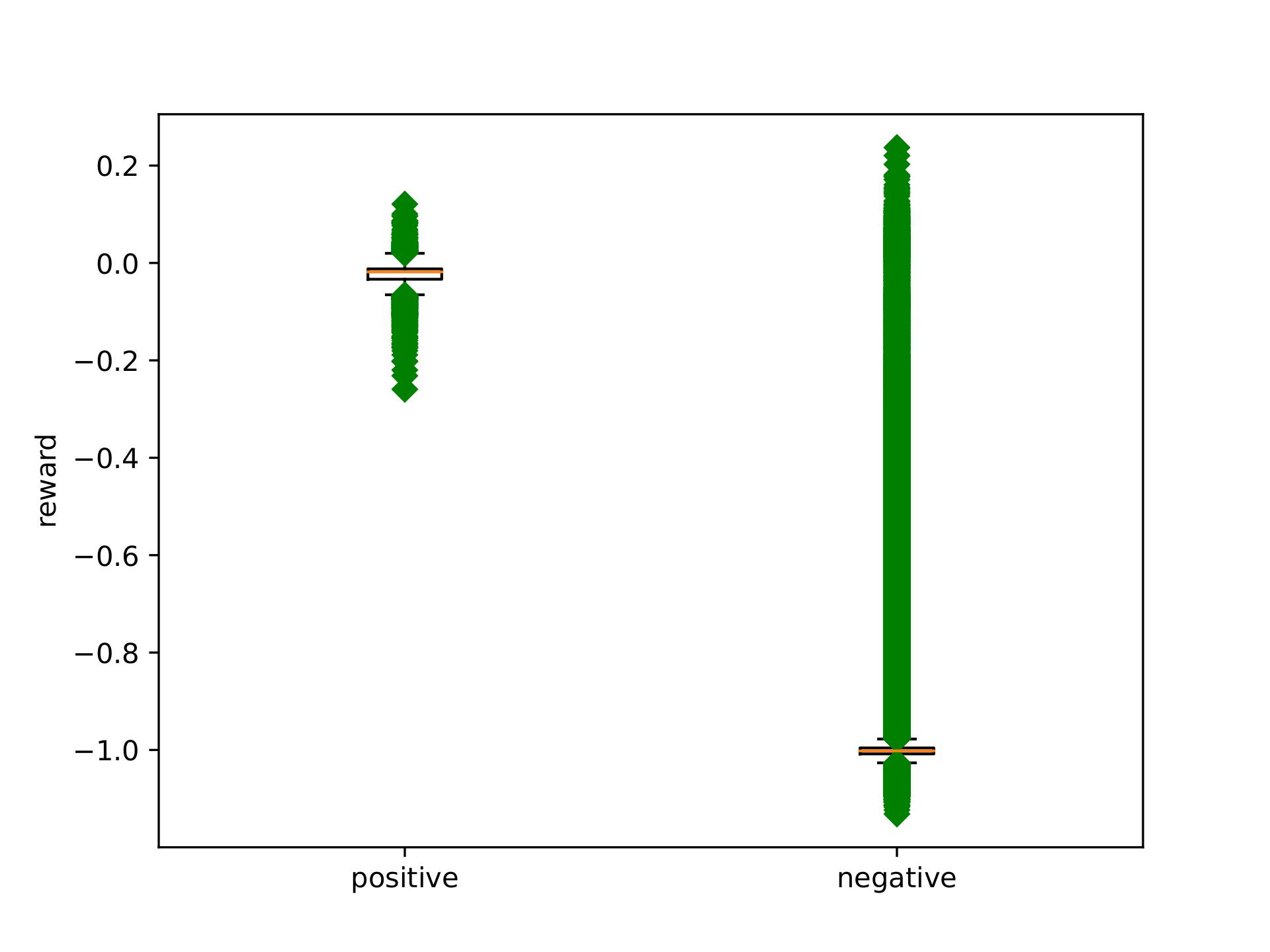}
     \caption{UDS+RP}
 \end{subfigure}
 \hfill
  \begin{subfigure}[b]{0.32\textwidth}
     \centering
     \includegraphics[width=\textwidth]{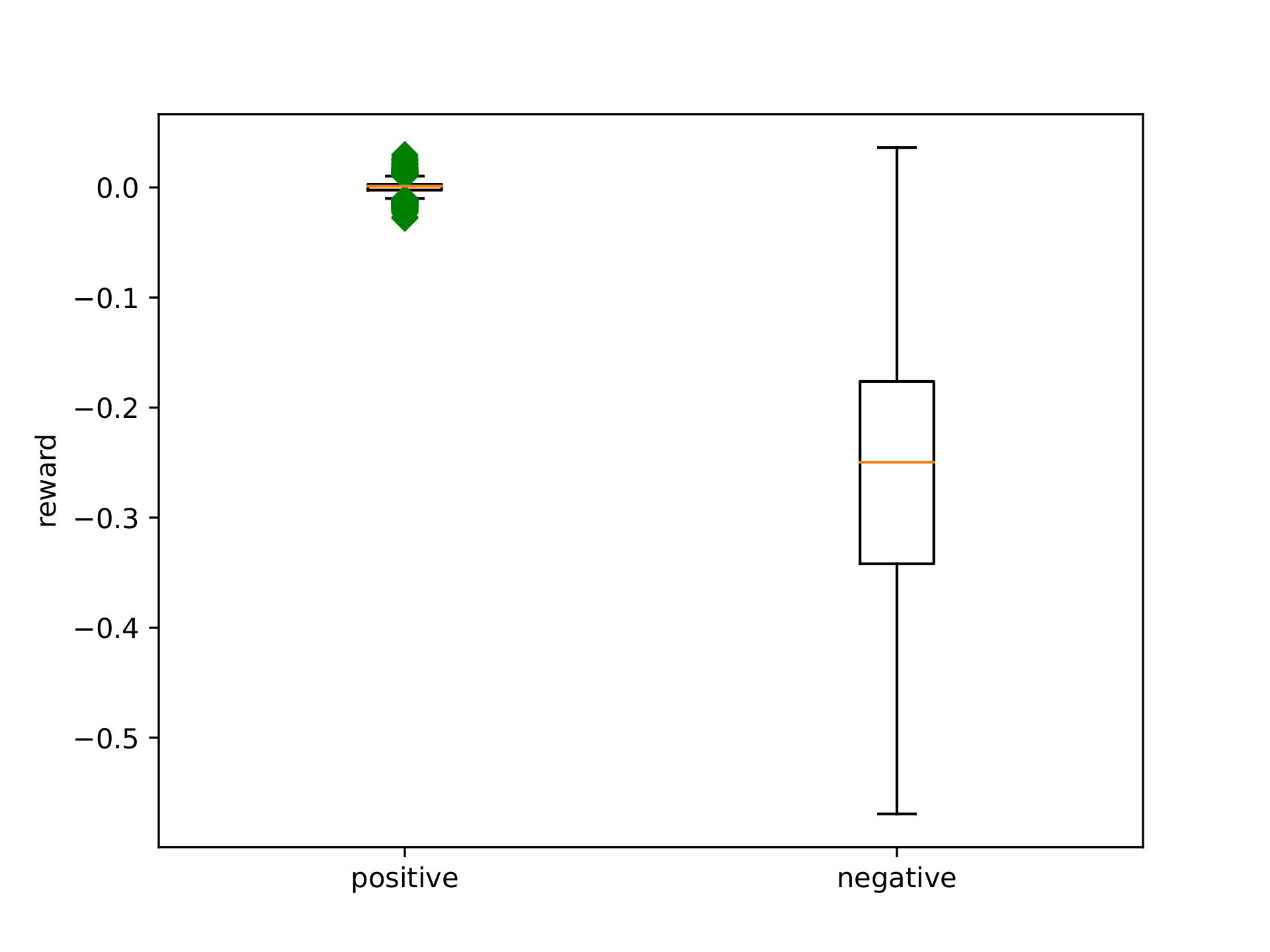}
     \caption{RP}
 \end{subfigure}
 \caption{Reward model evaluation for the medium-diverse dataset for the original AntMaze environment. Green dots are outliers. }
 \label{fig:medium}
\end{figure}

\begin{figure}[htpb]
\begin{subfigure}[b]{0.32\textwidth}
     \centering
     \includegraphics[width=\textwidth]{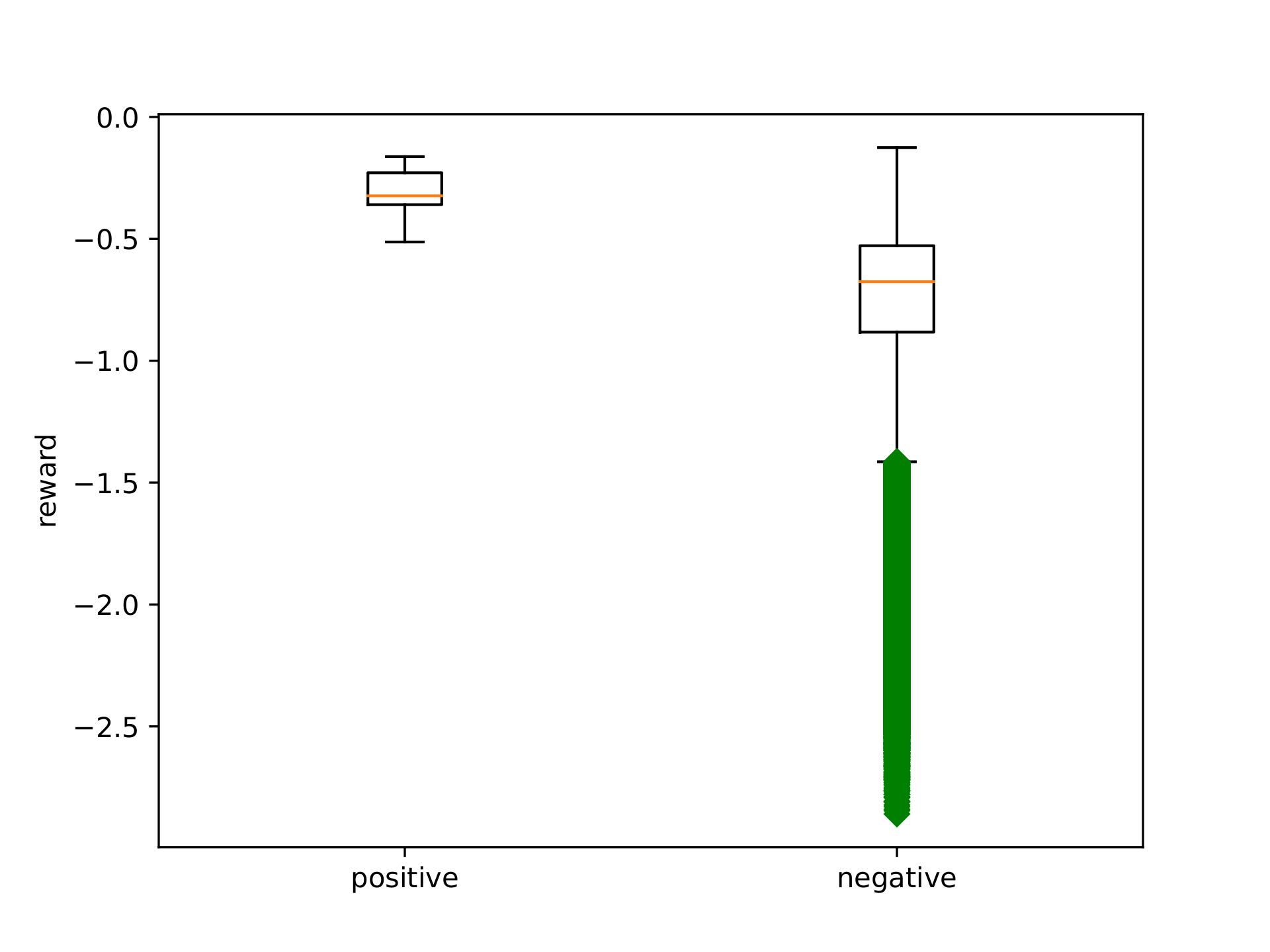}
     \caption{PDS}
 \end{subfigure}
 \hfill
 \begin{subfigure}[b]{0.32\textwidth}
     \centering
     \includegraphics[width=\textwidth]{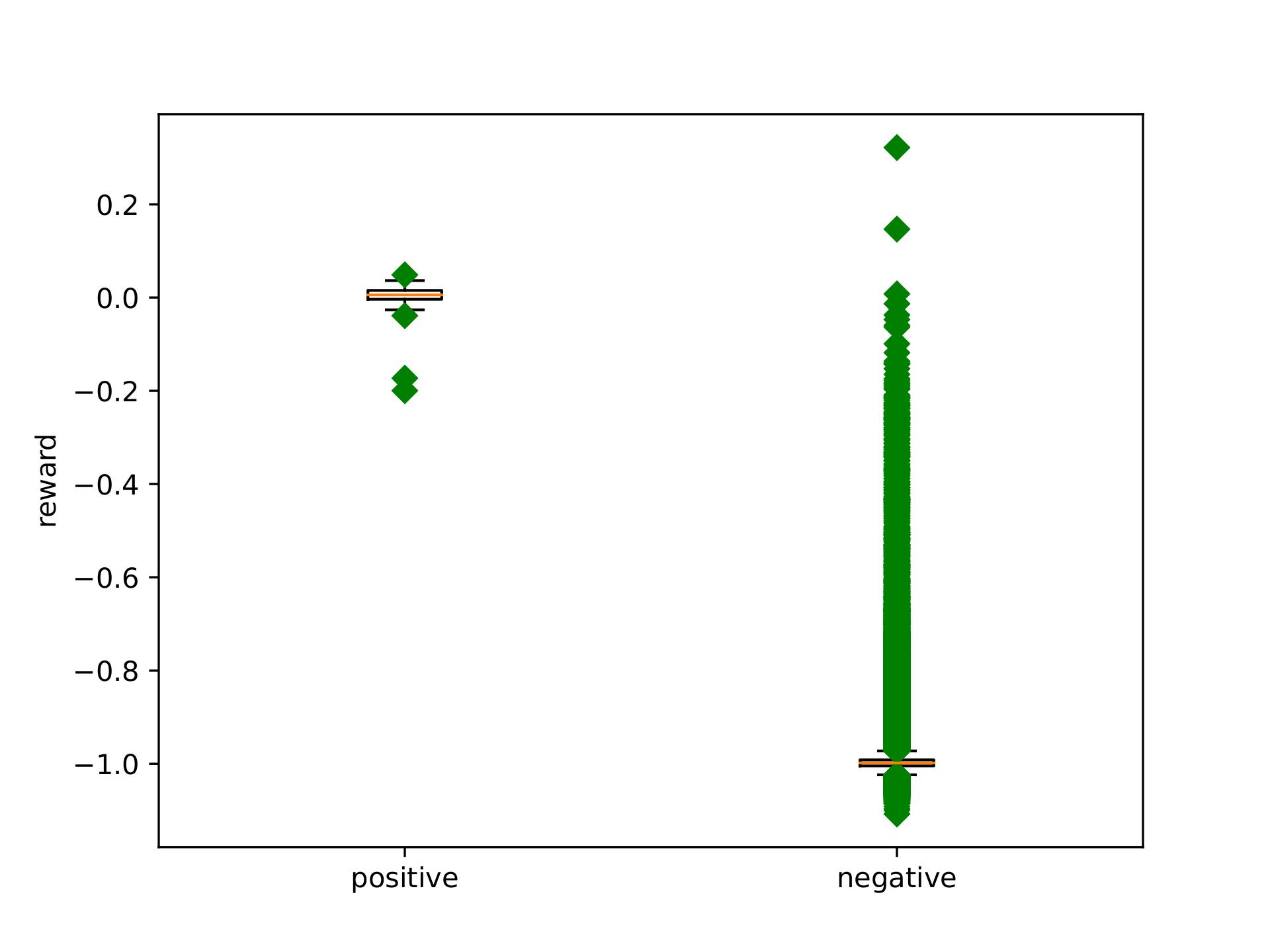}
     \caption{UDS+RP}
 \end{subfigure}
 \hfill
  \begin{subfigure}[b]{0.32\textwidth}
     \centering
     \includegraphics[width=\textwidth]{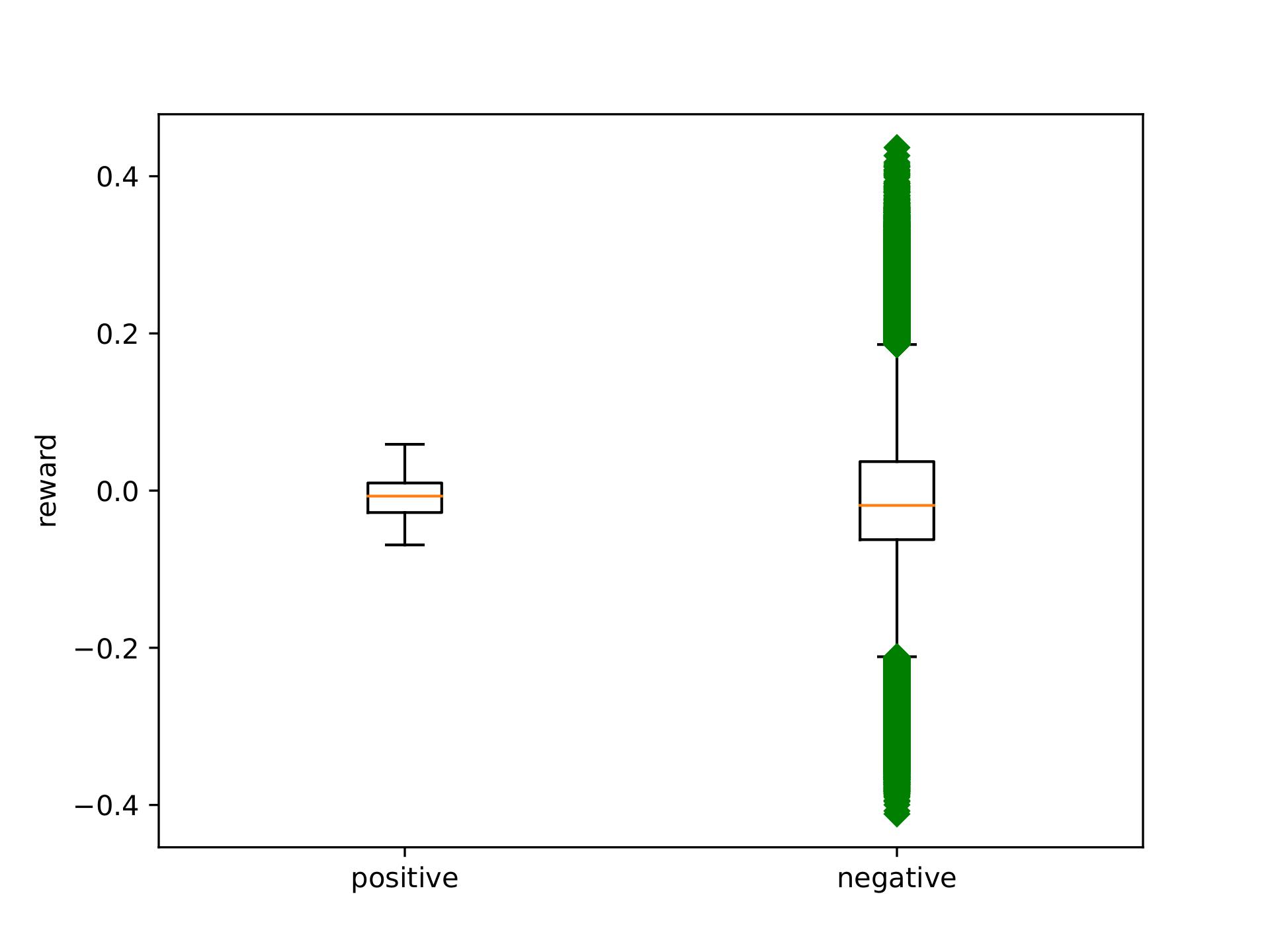}
     \caption{RP}
 \end{subfigure}
 \caption{Reward model evaluation for the umaze-diverse dataset for the original AntMaze environment. Green dots are outliers. }
 \label{fig:umaze}
\end{figure}

\begin{figure}[H]
\begin{subfigure}[b]{0.24\textwidth}
     \centering
     \includegraphics[width=\textwidth]{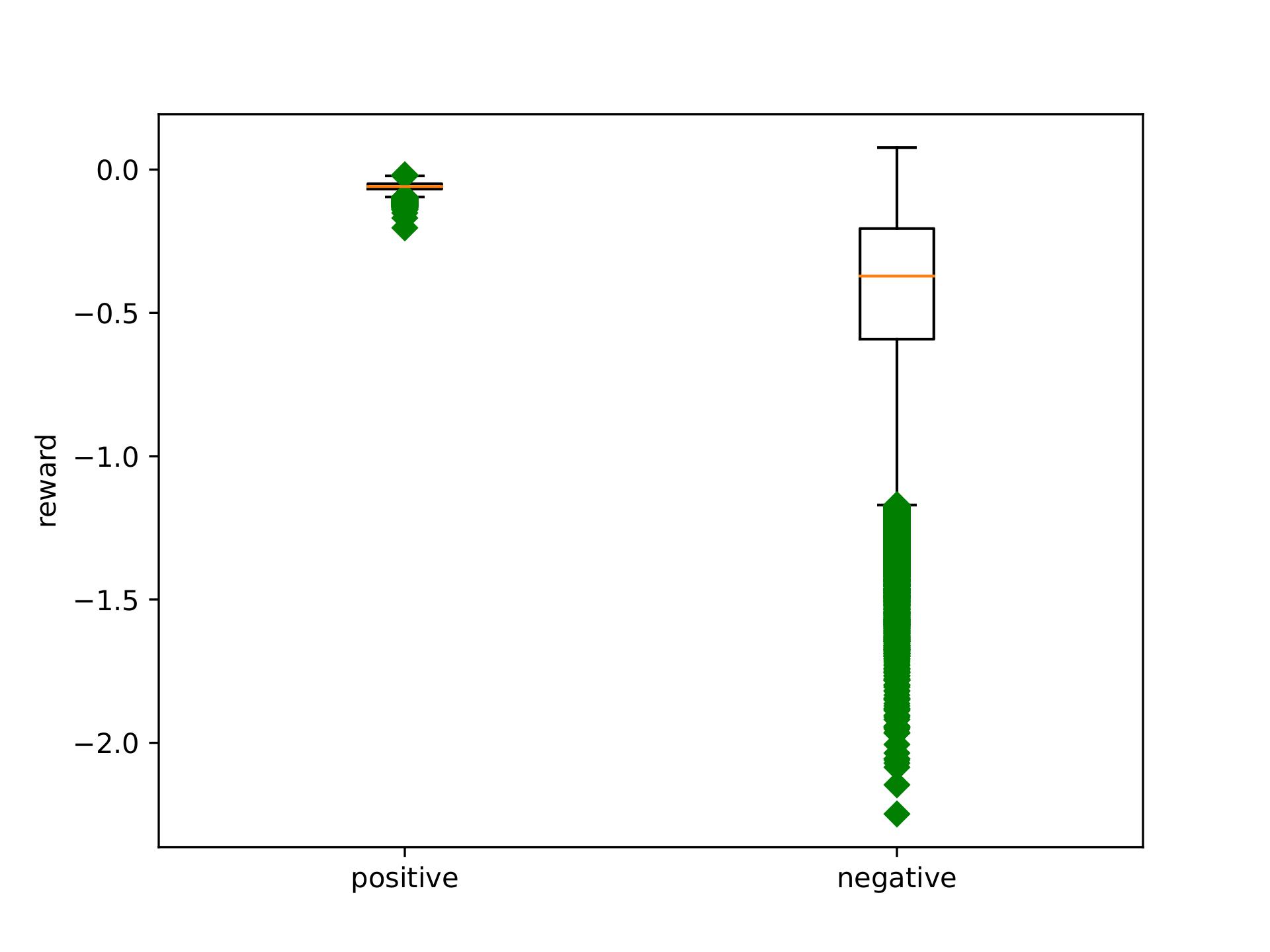}
     \caption{Large, PDS}
     % \label{fig:}
 \end{subfigure}
 \hfill
 \begin{subfigure}[b]{0.24\textwidth}
    \centering
\includegraphics[width=\textwidth]{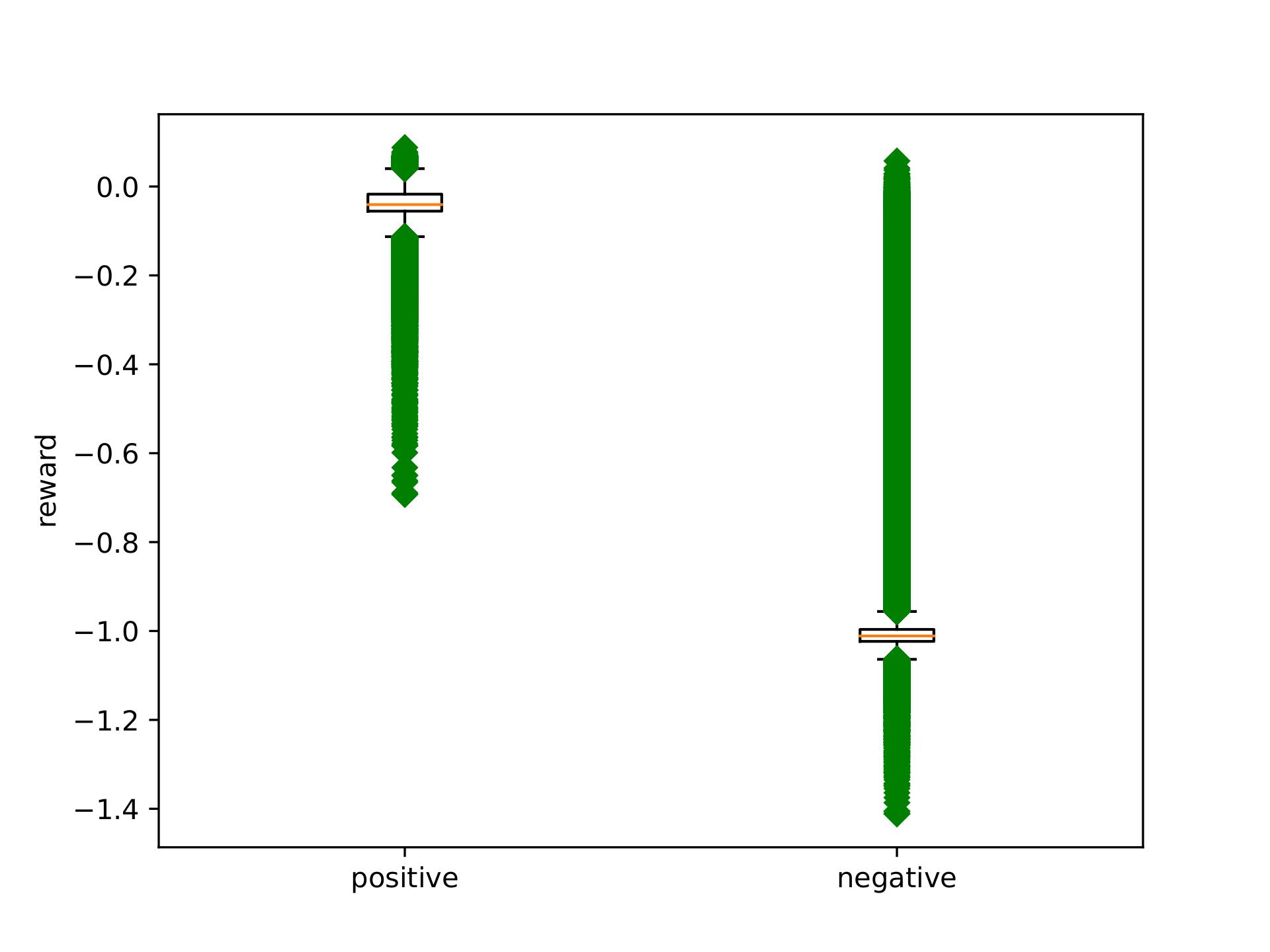}
     \caption{Large, UDS+RP}
     % \label{fig:}
 \end{subfigure}
\hfill
\begin{subfigure}[b]{0.24\textwidth}
     \centering
     \includegraphics[width=\textwidth]{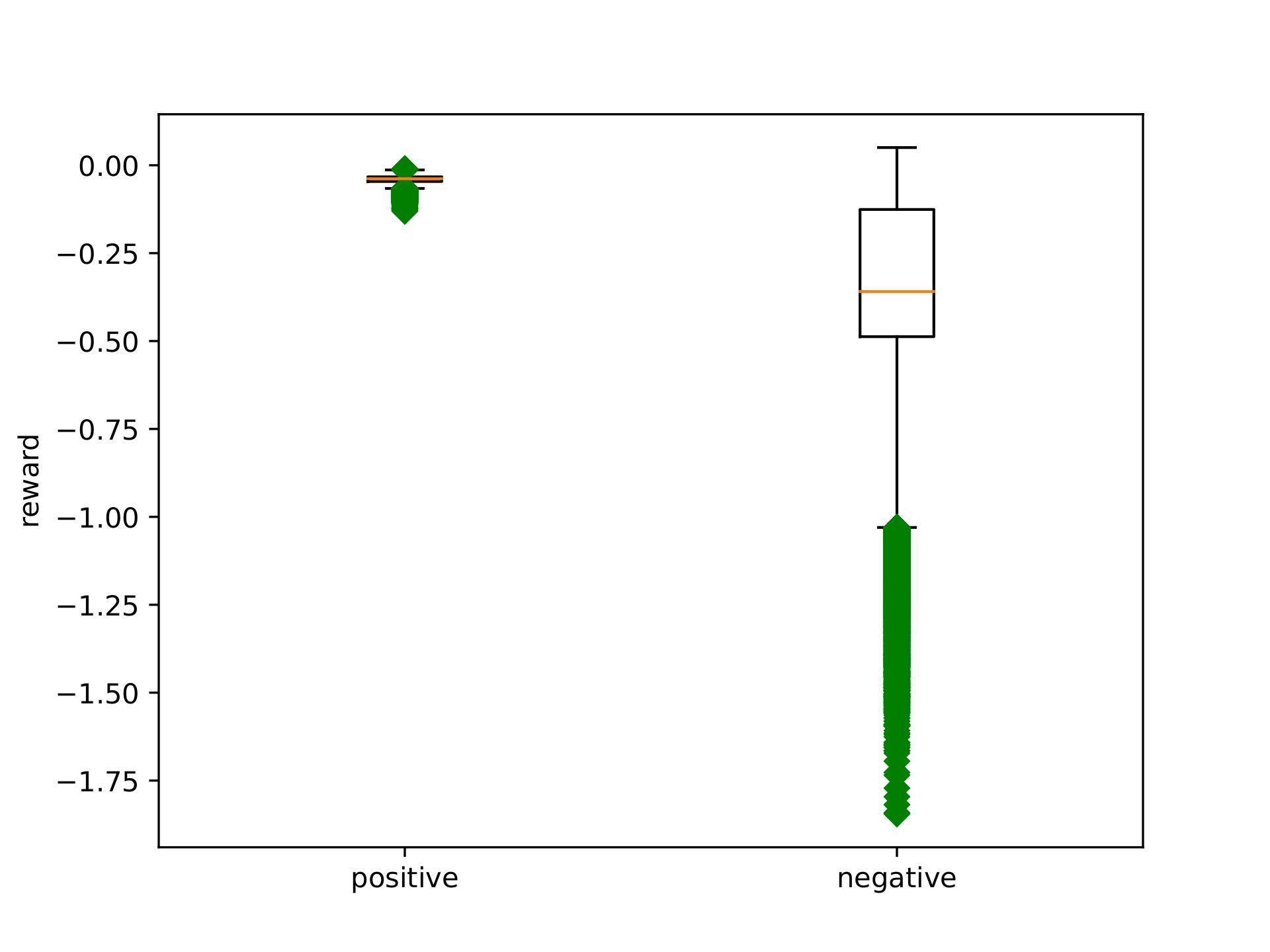}
     \caption{Medium, PDS}

 \end{subfigure}
 \hfill
 \begin{subfigure}[b]{0.24\textwidth}
    \centering
\includegraphics[width=\textwidth]{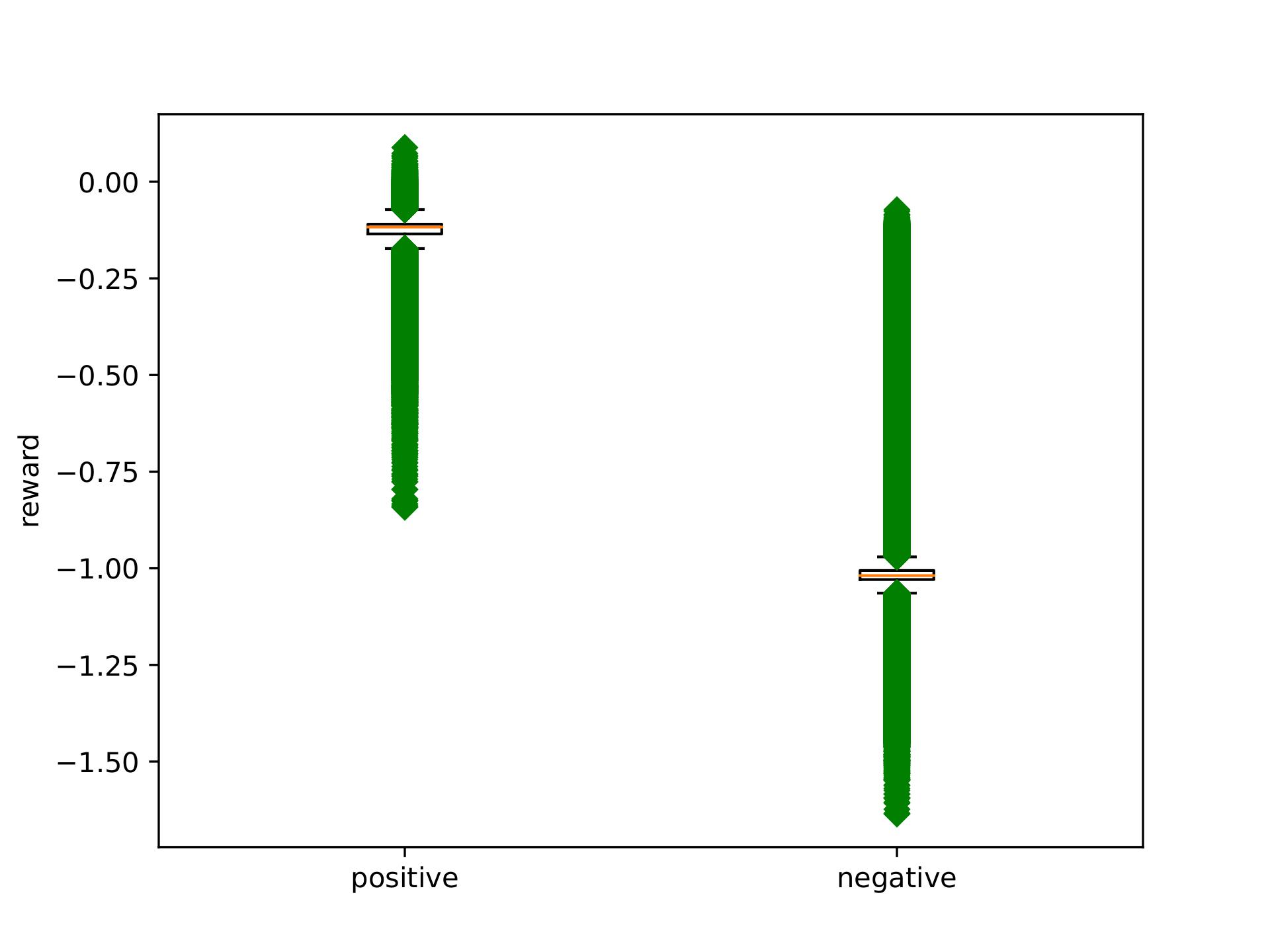}
     \caption{Medium, UDS+RP}

 \end{subfigure}
 \caption{Reward model evaluation for the Four Rooms environment. Green dots are outliers. }
 \label{fig:four_rooms}
\end{figure}

\begin{figure}[H]
\begin{subfigure}[b]{0.24\textwidth}
     \centering
     \includegraphics[width=\textwidth]{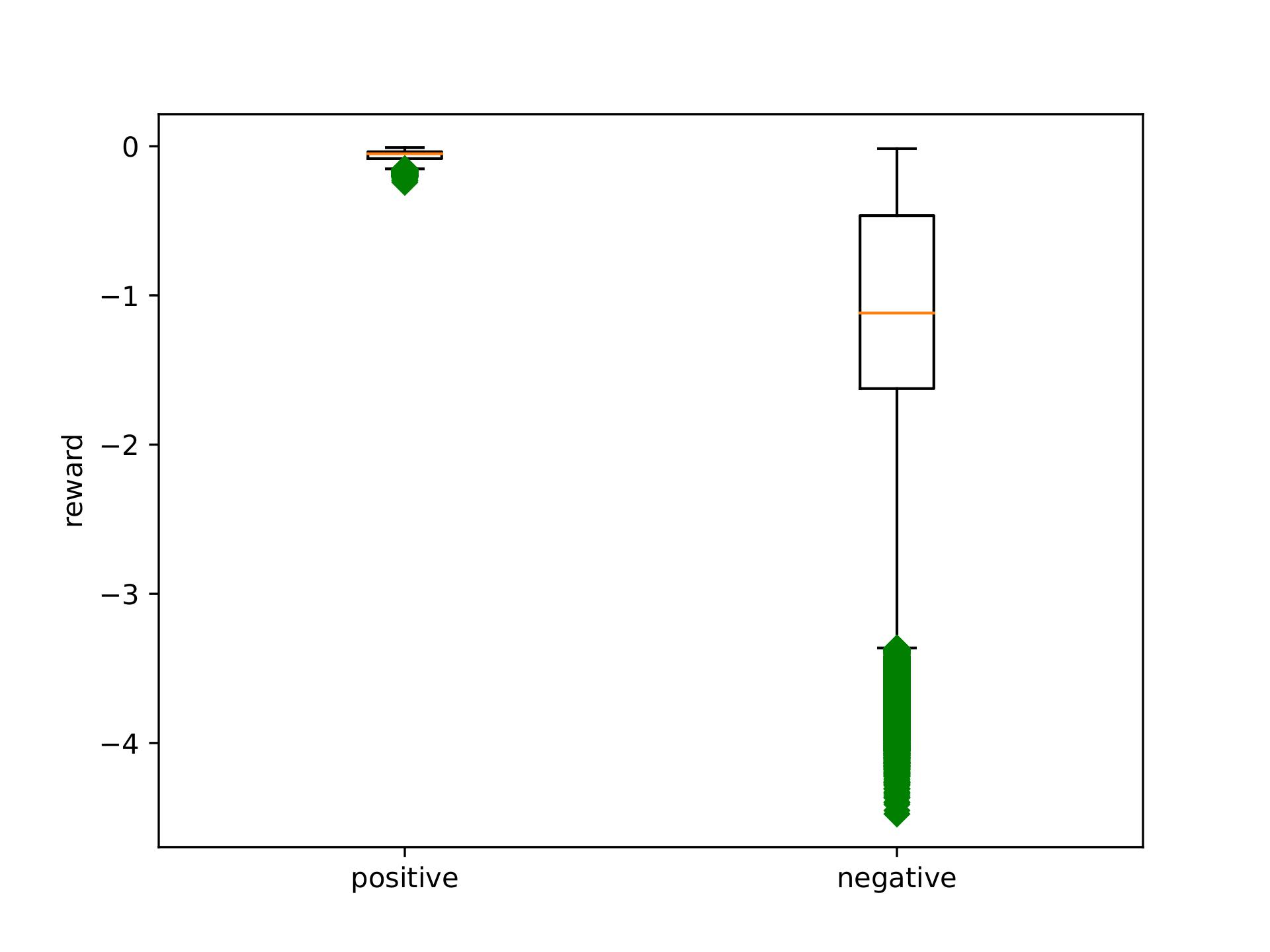}
     \caption{Large, PDS}

 \end{subfigure}
 \hfill
 \begin{subfigure}[b]{0.24\textwidth}
    \centering
\includegraphics[width=\textwidth]{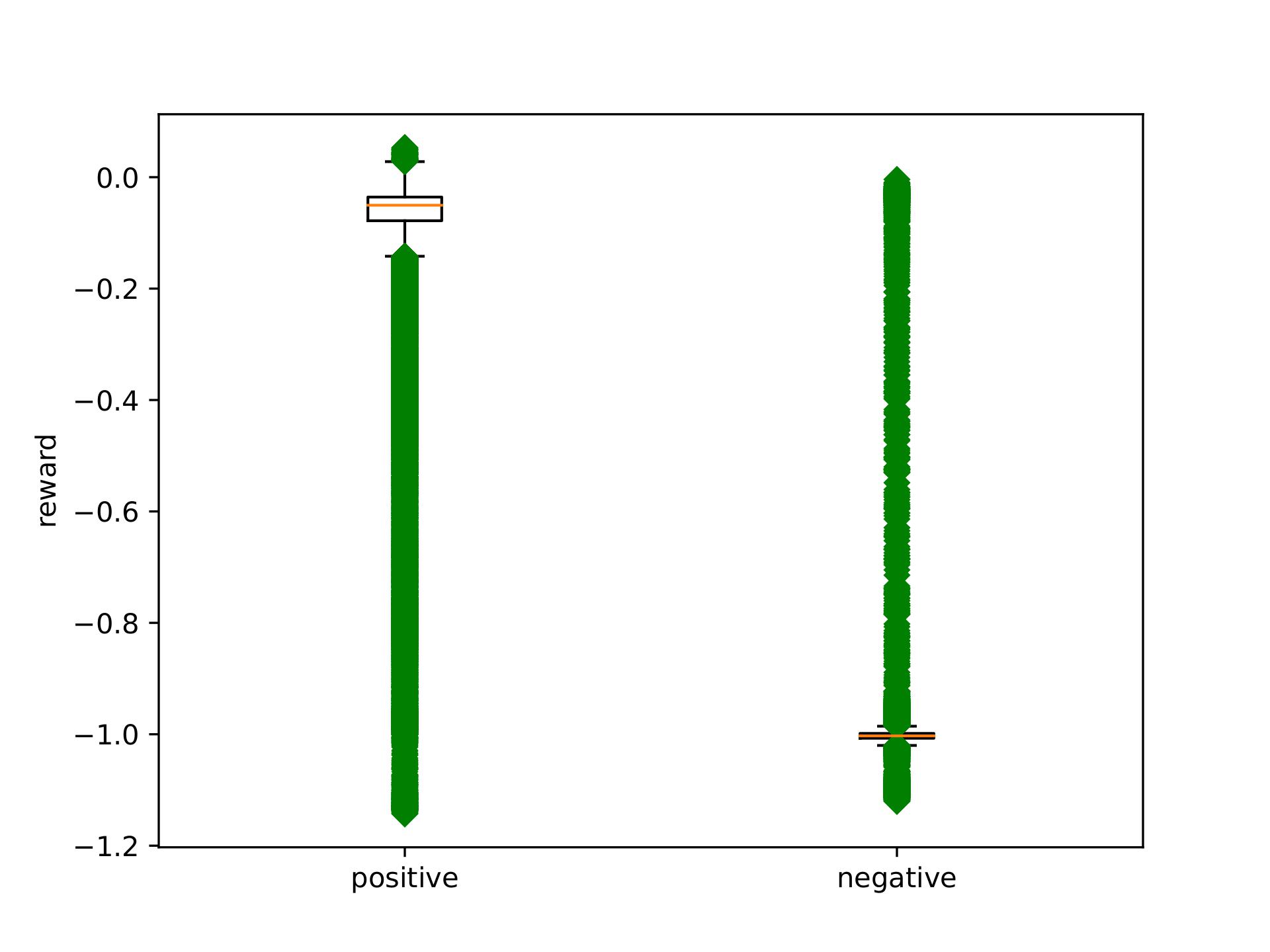}
     \caption{Large, UDS+RP}

 \end{subfigure}
\hfill
\begin{subfigure}[b]{0.24\textwidth}
     \centering
     \includegraphics[width=\textwidth]{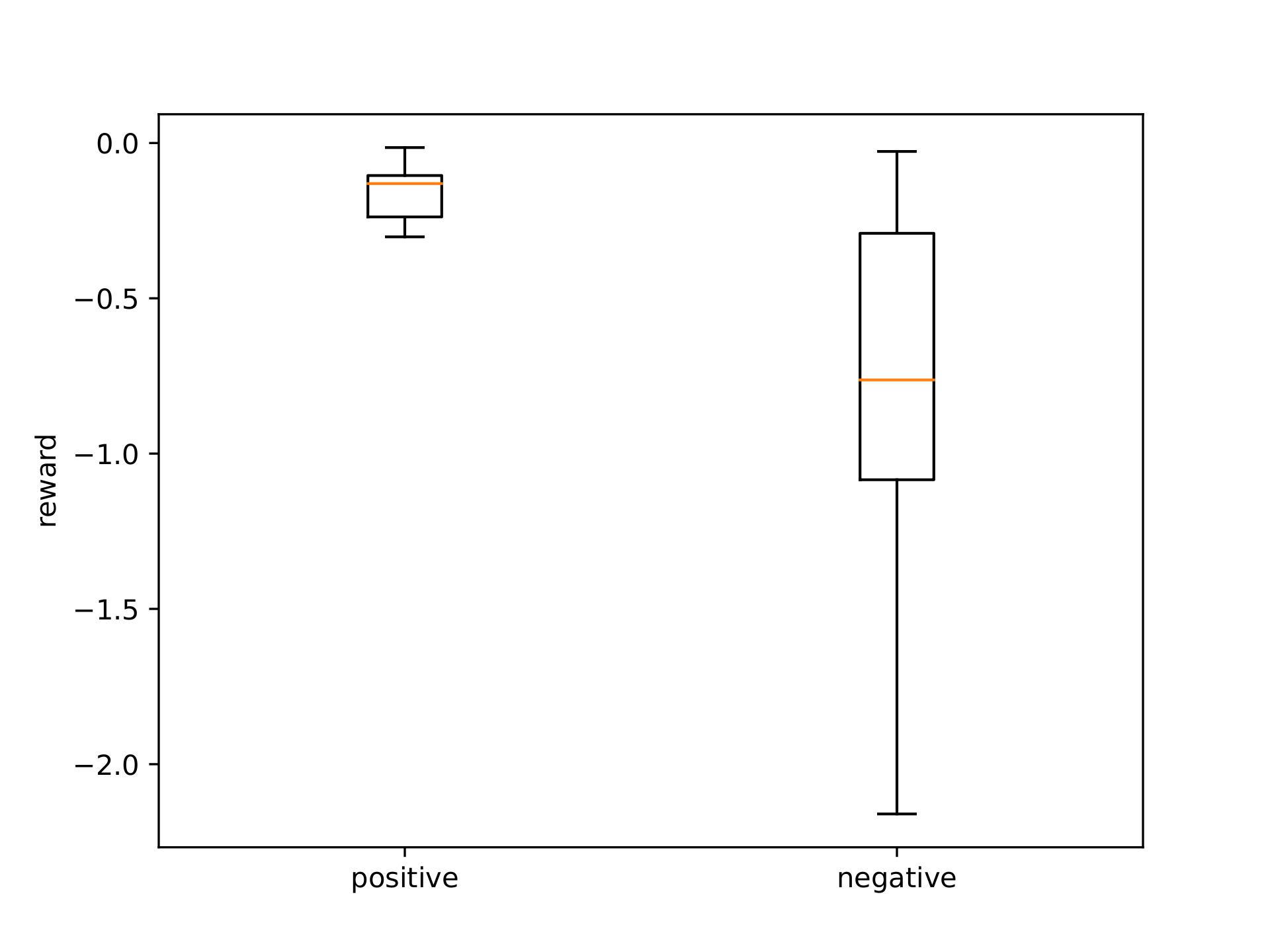}
     \caption{Medium, PDS}

 \end{subfigure}
 \hfill
 \begin{subfigure}[b]{0.24\textwidth}
    \centering
\includegraphics[width=\textwidth]{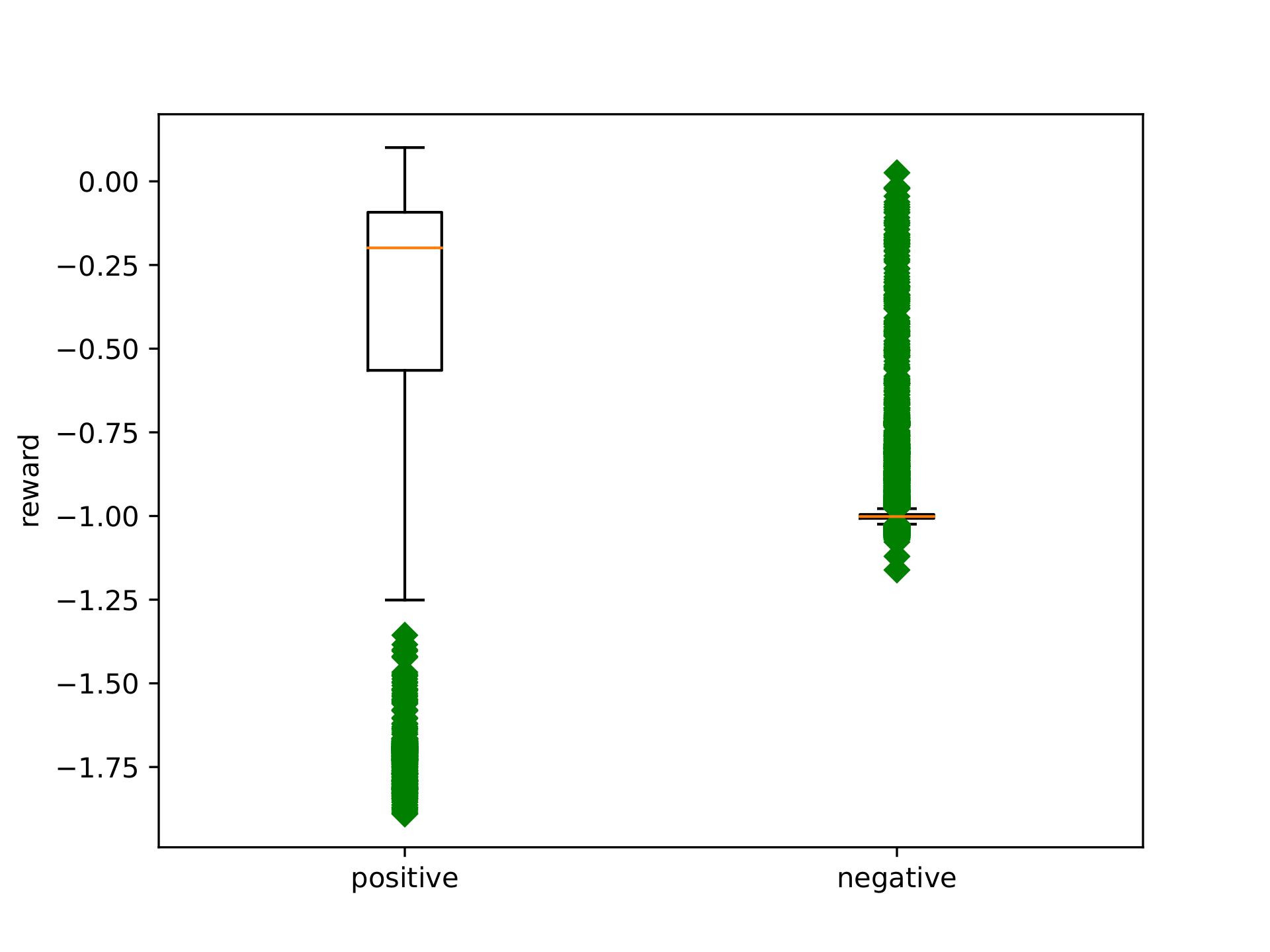}
     \caption{Medium, UDS+RP}

 \end{subfigure}
 \caption{Reward evaluation for Random Cells environment (the test context distribution is the same as training). Green dots are outliers. }
 \label{fig:cell}
\end{figure}

% \section{Adding out-of-distribution (OOD) goal examples in the context-goal set}
% We include another table with a slightly different setting compared with Section~\ref{exp:four_rooms}: for each goal set given the context in the training context-goal set, we add some extra random states that are out of the original range of the state space as out-of-distribution goal examples (which are not covered by the trajectory-only dataset). The results are shown in Table~\ref{tab:ood}, which is similar to the results in Section~\ref{exp:four_rooms}, showing that these methods are robust to extra OOD goal examples. 

% \label{app:ood}
% % \ying{put the results in four rooms ood.}
% \begin{table}[H]
% \centering
% \scalebox{0.9}{
% \begin{tabular}{l|ccc}
% \toprule
% Env/Method   & Ours      & PDS      & UDS+RP \\
% \hline
% medium & 78.9±1.6  & 23.5±1.2 & 13.4±1.2               \\
% large  & 70.0±5.7 & 9.0±2.6  & 22.5±0.9  \\ 
% \bottomrule
% \end{tabular}
% }
% \caption{Average scores with standard errors over 5 random seeds from Four Rooms, with extra OOD goal examples in the context-goal dataset. The reported score is the average success rate of three rooms, and the evaluation of each room requires 100 episodes.}
% \label{tab:ood}
% \end{table}

\newpage
\clearpage

\section*{NeurIPS Paper Checklist}

\begin{enumerate}

\item {\bf Claims}
    \item[] Question: Do the main claims made in the abstract and introduction accurately reflect the paper's contributions and scope?
    \item[] Answer: \answerYes{} % Replace by \answerYes{}, \answerNo{}, or \answerNA{}.
    \item[] Justification: All the theoretical results and experimental findings referenced in the abstract and introduction are included in the paper.
    \item[] Guidelines:
    \begin{itemize}
        \item The answer NA means that the abstract and introduction do not include the claims made in the paper.
        \item The abstract and/or introduction should clearly state the claims made, including the contributions made in the paper and important assumptions and limitations. A No or NA answer to this question will not be perceived well by the reviewers. 
        \item The claims made should match theoretical and experimental results, and reflect how much the results can be expected to generalize to other settings. 
        \item It is fine to include aspirational goals as motivation as long as it is clear that these goals are not attained by the paper. 
    \end{itemize}

\item {\bf Limitations}
    \item[] Question: Does the paper discuss the limitations of the work performed by the authors?
    \item[] Answer: \answerYes{} % Replace by \answerYes{}, \answerNo{}, or \answerNA{}.
    \item[] Justification: We have added a limitations section (see Section~\ref{sec:limitation}).
    \item[] Guidelines:
    \begin{itemize}
        \item The answer NA means that the paper has no limitation while the answer No means that the paper has limitations, but those are not discussed in the paper. 
        \item The authors are encouraged to create a separate "Limitations" section in their paper.
        \item The paper should point out any strong assumptions and how robust the results are to violations of these assumptions (e.g., independence assumptions, noiseless settings, model well-specification, asymptotic approximations only holding locally). The authors should reflect on how these assumptions might be violated in practice and what the implications would be.
        \item The authors should reflect on the scope of the claims made, e.g., if the approach was only tested on a few datasets or with a few runs. In general, empirical results often depend on implicit assumptions, which should be articulated.
        \item The authors should reflect on the factors that influence the performance of the approach. For example, a facial recognition algorithm may perform poorly when image resolution is low or images are taken in low lighting. Or a speech-to-text system might not be used reliably to provide closed captions for online lectures because it fails to handle technical jargon.
        \item The authors should discuss the computational efficiency of the proposed algorithms and how they scale with dataset size.
        \item If applicable, the authors should discuss possible limitations of their approach to address problems of privacy and fairness.
        \item While the authors might fear that complete honesty about limitations might be used by reviewers as grounds for rejection, a worse outcome might be that reviewers discover limitations that aren't acknowledged in the paper. The authors should use their best judgment and recognize that individual actions in favor of transparency play an important role in developing norms that preserve the integrity of the community. Reviewers will be specifically instructed to not penalize honesty concerning limitations.
    \end{itemize}

\item {\bf Theory Assumptions and Proofs}
    \item[] Question: For each theoretical result, does the paper provide the full set of assumptions and a complete (and correct) proof?
    \item[] Answer: \answerYes{} % Replace by \answerYes{}, \answerNo{}, or \answerNA{}.
    \item[] Justification: We include the full detailed proof of our theoretical result in the appendix.
    \item[] Guidelines:
    \begin{itemize}
        \item The answer NA means that the paper does not include theoretical results. 
        \item All the theorems, formulas, and proofs in the paper should be numbered and cross-referenced.
        \item All assumptions should be clearly stated or referenced in the statement of any theorems.
        \item The proofs can either appear in the main paper or the supplemental material, but if they appear in the supplemental material, the authors are encouraged to provide a short proof sketch to provide intuition. 
        \item Inversely, any informal proof provided in the core of the paper should be complemented by formal proofs provided in appendix or supplemental material.
        \item Theorems and Lemmas that the proof relies upon should be properly referenced. 
    \end{itemize}

    \item {\bf Experimental Result Reproducibility}
    \item[] Question: Does the paper fully disclose all the information needed to reproduce the main experimental results of the paper to the extent that it affects the main claims and/or conclusions of the paper (regardless of whether the code and data are provided or not)?
    \item[] Answer: \answerYes{} % Replace by \answerYes{}, \answerNo{}, or \answerNA{}.
    \item[] Justification: Our experiments are based on standard RL domains and we discuss all the modifications that we consider for our setup in detail. Further, all algorithmic details and hyperparameters are also discussed.
    \item[] Guidelines:
    \begin{itemize}
        \item The answer NA means that the paper does not include experiments.
        \item If the paper includes experiments, a No answer to this question will not be perceived well by the reviewers: Making the paper reproducible is important, regardless of whether the code and data are provided or not.
        \item If the contribution is a dataset and/or model, the authors should describe the steps taken to make their results reproducible or verifiable. 
        \item Depending on the contribution, reproducibility can be accomplished in various ways. For example, if the contribution is a novel architecture, describing the architecture fully might suffice, or if the contribution is a specific model and empirical evaluation, it may be necessary to either make it possible for others to replicate the model with the same dataset, or provide access to the model. In general. releasing code and data is often one good way to accomplish this, but reproducibility can also be provided via detailed instructions for how to replicate the results, access to a hosted model (e.g., in the case of a large language model), releasing of a model checkpoint, or other means that are appropriate to the research performed.
        \item While NeurIPS does not require releasing code, the conference does require all submissions to provide some reasonable avenue for reproducibility, which may depend on the nature of the contribution. For example
        \begin{enumerate}
            \item If the contribution is primarily a new algorithm, the paper should make it clear how to reproduce that algorithm.
            \item If the contribution is primarily a new model architecture, the paper should describe the architecture clearly and fully.
            \item If the contribution is a new model (e.g., a large language model), then there should either be a way to access this model for reproducing the results or a way to reproduce the model (e.g., with an open-source dataset or instructions for how to construct the dataset).
            \item We recognize that reproducibility may be tricky in some cases, in which case authors are welcome to describe the particular way they provide for reproducibility. In the case of closed-source models, it may be that access to the model is limited in some way (e.g., to registered users), but it should be possible for other researchers to have some path to reproducing or verifying the results.
        \end{enumerate}
    \end{itemize}

\item {\bf Open access to data and code}
    \item[] Question: Does the paper provide open access to the data and code, with sufficient instructions to faithfully reproduce the main experimental results, as described in supplemental material?
    \item[] Answer:  \answerYes{} % Replace by \answerYes{}, \answerNo{}, or \answerNA{}.
    \item[] Justification: We provide code in supplementary files.
    \item[] Guidelines:
    \begin{itemize}
        \item The answer NA means that paper does not include experiments requiring code.
        \item Please see the NeurIPS code and data submission guidelines (\url{https://nips.cc/public/guides/CodeSubmissionPolicy}) for more details.
        \item While we encourage the release of code and data, we understand that this might not be possible, so “No” is an acceptable answer. Papers cannot be rejected simply for not including code, unless this is central to the contribution (e.g., for a new open-source benchmark).
        \item The instructions should contain the exact command and environment needed to run to reproduce the results. See the NeurIPS code and data submission guidelines (\url{https://nips.cc/public/guides/CodeSubmissionPolicy}) for more details.
        \item The authors should provide instructions on data access and preparation, including how to access the raw data, preprocessed data, intermediate data, and generated data, etc.
        \item The authors should provide scripts to reproduce all experimental results for the new proposed method and baselines. If only a subset of experiments are reproducible, they should state which ones are omitted from the script and why.
        \item At submission time, to preserve anonymity, the authors should release anonymized versions (if applicable).
        \item Providing as much information as possible in supplemental material (appended to the paper) is recommended, but including URLs to data and code is permitted.
    \end{itemize}

\item {\bf Experimental Setting/Details}
    \item[] Question: Does the paper specify all the training and test details (e.g., data splits, hyperparameters, how they were chosen, type of optimizer, etc.) necessary to understand the results?
    \item[] Answer: \answerYes{} % Replace by \answerYes{}, \answerNo{}, or \answerNA{}.
    \item[] Justification: We provide all the necessary details in the main paper (Section~\ref{sec:experiments}) with additional information in the appendix.
    \item[] Guidelines:
    \begin{itemize}
        \item The answer NA means that the paper does not include experiments.
        \item The experimental setting should be presented in the core of the paper to a level of detail that is necessary to appreciate the results and make sense of them.
        \item The full details can be provided either with the code, in appendix, or as supplemental material.
    \end{itemize}

\item {\bf Experiment Statistical Significance}
    \item[] Question: Does the paper report error bars suitably and correctly defined or other appropriate information about the statistical significance of the experiments?
    \item[] Answer: \answerYes{} % Replace by \answerYes{}, \answerNo{}, or \answerNA{}.
    \item[] Justification: We report the standard errors across multiple runs for all our experiments.
    \item[] Guidelines:
    \begin{itemize}
        \item The answer NA means that the paper does not include experiments.
        \item The authors should answer "Yes" if the results are accompanied by error bars, confidence intervals, or statistical significance tests, at least for the experiments that support the main claims of the paper.
        \item The factors of variability that the error bars are capturing should be clearly stated (for example, train/test split, initialization, random drawing of some parameter, or overall run with given experimental conditions).
        \item The method for calculating the error bars should be explained (closed form formula, call to a library function, bootstrap, etc.)
        \item The assumptions made should be given (e.g., Normally distributed errors).
        \item It should be clear whether the error bar is the standard deviation or the standard error of the mean.
        \item It is OK to report 1-sigma error bars, but one should state it. The authors should preferably report a 2-sigma error bar than state that they have a 96\% CI, if the hypothesis of Normality of errors is not verified.
        \item For asymmetric distributions, the authors should be careful not to show in tables or figures symmetric error bars that would yield results that are out of range (e.g. negative error rates).
        \item If error bars are reported in tables or plots, The authors should explain in the text how they were calculated and reference the corresponding figures or tables in the text.
    \end{itemize}

\item {\bf Experiments Compute Resources}
    \item[] Question: For each experiment, does the paper provide sufficient information on the computer resources (type of compute workers, memory, time of execution) needed to reproduce the experiments?
    \item[] Answer:  \answerYes{}% Replace by \answerYes{}, \answerNo{}, or \answerNA{}.
    \item[] Justification: We provide compute resource in the appendix.
    \item[] Guidelines:
    \begin{itemize}
        \item The answer NA means that the paper does not include experiments.
        \item The paper should indicate the type of compute workers CPU or GPU, internal cluster, or cloud provider, including relevant memory and storage.
        \item The paper should provide the amount of compute required for each of the individual experimental runs as well as estimate the total compute. 
        \item The paper should disclose whether the full research project required more compute than the experiments reported in the paper (e.g., preliminary or failed experiments that didn't make it into the paper). 
    \end{itemize}
    
\item {\bf Code Of Ethics}
    \item[] Question: Does the research conducted in the paper conform, in every respect, with the NeurIPS Code of Ethics \url{https://neurips.cc/public/EthicsGuidelines}?
    \item[] Answer: \answerYes{} % Replace by \answerYes{}, \answerNo{}, or \answerNA{}.
    \item[] Justification: Based on the NeurIPS code of ethics, we do not see any direct ethical or societal impact considerations for our work.
    \item[] Guidelines:
    \begin{itemize}
        \item The answer NA means that the authors have not reviewed the NeurIPS Code of Ethics.
        \item If the authors answer No, they should explain the special circumstances that require a deviation from the Code of Ethics.
        \item The authors should make sure to preserve anonymity (e.g., if there is a special consideration due to laws or regulations in their jurisdiction).
    \end{itemize}

\item {\bf Broader Impacts}
    \item[] Question: Does the paper discuss both potential positive societal impacts and negative societal impacts of the work performed?
    \item[] Answer: \answerNA{} % Replace by \answerYes{}, \answerNo{}, or \answerNA{}.
    \item[] Justification: Based on the NeurIPS code of ethics, we do not see any direct ethical or societal impact considerations for our work.
    \item[] Guidelines:
    \begin{itemize}
        \item The answer NA means that there is no societal impact of the work performed.
        \item If the authors answer NA or No, they should explain why their work has no societal impact or why the paper does not address societal impact.
        \item Examples of negative societal impacts include potential malicious or unintended uses (e.g., disinformation, generating fake profiles, surveillance), fairness considerations (e.g., deployment of technologies that could make decisions that unfairly impact specific groups), privacy considerations, and security considerations.
        \item The conference expects that many papers will be foundational research and not tied to particular applications, let alone deployments. However, if there is a direct path to any negative applications, the authors should point it out. For example, it is legitimate to point out that an improvement in the quality of generative models could be used to generate deepfakes for disinformation. On the other hand, it is not needed to point out that a generic algorithm for optimizing neural networks could enable people to train models that generate Deepfakes faster.
        \item The authors should consider possible harms that could arise when the technology is being used as intended and functioning correctly, harms that could arise when the technology is being used as intended but gives incorrect results, and harms following from (intentional or unintentional) misuse of the technology.
        \item If there are negative societal impacts, the authors could also discuss possible mitigation strategies (e.g., gated release of models, providing defenses in addition to attacks, mechanisms for monitoring misuse, mechanisms to monitor how a system learns from feedback over time, improving the efficiency and accessibility of ML).
    \end{itemize}
    
\item {\bf Safeguards}
    \item[] Question: Does the paper describe safeguards that have been put in place for responsible release of data or models that have a high risk for misuse (e.g., pretrained language models, image generators, or scraped datasets)?
    \item[] Answer: \answerNA{} % Replace by \answerYes{}, \answerNo{}, or \answerNA{}.
    \item[] Justification: These safeguard concerns do not apply to our experimental domains.
    \item[] Guidelines:
    \begin{itemize}
        \item The answer NA means that the paper poses no such risks.
        \item Released models that have a high risk for misuse or dual-use should be released with necessary safeguards to allow for controlled use of the model, for example by requiring that users adhere to usage guidelines or restrictions to access the model or implementing safety filters. 
        \item Datasets that have been scraped from the Internet could pose safety risks. The authors should describe how they avoided releasing unsafe images.
        \item We recognize that providing effective safeguards is challenging, and many papers do not require this, but we encourage authors to take this into account and make a best faith effort.
    \end{itemize}

\item {\bf Licenses for existing assets}
    \item[] Question: Are the creators or original owners of assets (e.g., code, data, models), used in the paper, properly credited and are the license and terms of use explicitly mentioned and properly respected?
    \item[] Answer: \answerYes{} % Replace by \answerYes{}, \answerNo{}, or \answerNA{}.
    \item[] Justification: We include the information in the supplementary files.
    \item[] Guidelines:
    \begin{itemize}
        \item The answer NA means that the paper does not use existing assets.
        \item The authors should cite the original paper that produced the code package or dataset.
        \item The authors should state which version of the asset is used and, if possible, include a URL.
        \item The name of the license (e.g., CC-BY 4.0) should be included for each asset.
        \item For scraped data from a particular source (e.g., website), the copyright and terms of service of that source should be provided.
        \item If assets are released, the license, copyright information, and terms of use in the package should be provided. For popular datasets, \url{paperswithcode.com/datasets} has curated licenses for some datasets. Their licensing guide can help determine the license of a dataset.
        \item For existing datasets that are re-packaged, both the original license and the license of the derived asset (if it has changed) should be provided.
        \item If this information is not available online, the authors are encouraged to reach out to the asset's creators.
    \end{itemize}

\item {\bf New Assets}
    \item[] Question: Are new assets introduced in the paper well documented and is the documentation provided alongside the assets?
    \item[] Answer: \answerNA{} % Replace by \answerYes{}, \answerNo{}, or \answerNA{}.
    \item[] Justification: Not applicable to this paper.
    \item[] Guidelines:
    \begin{itemize}
        \item The answer NA means that the paper does not release new assets.
        \item Researchers should communicate the details of the dataset/code/model as part of their submissions via structured templates. This includes details about training, license, limitations, etc. 
        \item The paper should discuss whether and how consent was obtained from people whose asset is used.
        \item At submission time, remember to anonymize your assets (if applicable). You can either create an anonymized URL or include an anonymized zip file.
    \end{itemize}

\item {\bf Crowdsourcing and Research with Human Subjects}
    \item[] Question: For crowdsourcing experiments and research with human subjects, does the paper include the full text of instructions given to participants and screenshots, if applicable, as well as details about compensation (if any)? 
    \item[] Answer: \answerNA{} % Replace by \answerYes{}, \answerNo{}, or \answerNA{}.
    \item[] Justification: No crowdsourcing or human subjects involved.
    \item[] Guidelines:
    \begin{itemize}
        \item The answer NA means that the paper does not involve crowdsourcing nor research with human subjects.
        \item Including this information in the supplemental material is fine, but if the main contribution of the paper involves human subjects, then as much detail as possible should be included in the main paper. 
        \item According to the NeurIPS Code of Ethics, workers involved in data collection, curation, or other labor should be paid at least the minimum wage in the country of the data collector. 
    \end{itemize}

\item {\bf Institutional Review Board (IRB) Approvals or Equivalent for Research with Human Subjects}
    \item[] Question: Does the paper describe potential risks incurred by study participants, whether such risks were disclosed to the subjects, and whether Institutional Review Board (IRB) approvals (or an equivalent approval/review based on the requirements of your country or institution) were obtained?
    \item[] Answer: \answerNA{} % Replace by \answerYes{}, \answerNo{}, or \answerNA{}.
    \item[] Justification: No human subjects or crowd-sourcing involved in experiments for this paper.
    \item[] Guidelines:
    \begin{itemize}
        \item The answer NA means that the paper does not involve crowdsourcing nor research with human subjects.
        \item Depending on the country in which research is conducted, IRB approval (or equivalent) may be required for any human subjects research. If you obtained IRB approval, you should clearly state this in the paper. 
        \item We recognize that the procedures for this may vary significantly between institutions and locations, and we expect authors to adhere to the NeurIPS Code of Ethics and the guidelines for their institution. 
        \item For initial submissions, do not include any information that would break anonymity (if applicable), such as the institution conducting the review.
    \end{itemize}

\end{enumerate}

\end{document}